\documentclass[10pt]{article}
\usepackage[numbers,sort&compress]{natbib}
\usepackage{ss}

\usepackage[utf8]{inputenc}
\usepackage[T1]{fontenc}

\usepackage[verbose=true,
letterpaper,
textheight=8.7in,
textwidth=6.4in,
headheight=12pt,
headsep=25pt,
footskip=30pt]{geometry}

\usepackage{mathtools}
\usepackage{booktabs}
\usepackage{graphicx}
\usepackage{amsmath,amssymb,amsfonts,amsthm,amsxtra,bm}
\usepackage{amsthm}
\usepackage{xspace}
\usepackage{enumerate}
\usepackage{hyperref}
\usepackage{url}
\usepackage{colortbl}
\usepackage{makecell,multirow}       
\usepackage{nicefrac}       
\usepackage{thmtools}  
\usepackage{xspace}
\usepackage{footnote, tablefootnote}
\usepackage{float}
\floatstyle{plaintop}
\restylefloat{table}
\usepackage{epstopdf}
\usepackage{latexsym}
\usepackage{algorithm, algorithmicx, algpseudocode}
\usepackage{todonotes}
\usepackage{thmtools, thm-restate}
\usepackage{bbm}

\numberwithin{equation}{section}
\newtheorem{theorem}{Theorem}
\newtheorem{lemma}[theorem]{Lemma}

\newtheorem{corollary}[theorem]{Corollary}
\newtheorem{definition}{Definition}

\newcommand{\E}{\mathbb{E}}
\newcommand{\Var}{\mathrm{Var}}
\newcommand{\SV}{\mathbb{S}}  

\DeclareMathOperator*{\argmax}{argmax}

\newcommand\innerp[2]{\left\langle #1, #2 \right\rangle}
\newcommand{\innerpc}[2]{\bigl\langle #1, #2 \bigr\rangle}
\newcommand{\innerpd}[2]{\biggl\langle #1, #2 \biggr\rangle}
\newcommand{\floor}[1]{\lfloor #1 \rfloor}

\newcommand{\normlone}{L_1}
\newcommand{\normltwo}{L_2}

\newcommand{\normmax}{L_\infty}

\newcommand{\nlsum}{\sum\nolimits}
\newcommand{\nlprod}{\prod\nolimits}

\def\bigo{{\mathcal{O}}}
\def\bigotilde{\tilde{\mathcal{O}}}

\def\mdp{{\mathcal{M}}}
\def\Regret{{\textnormal{Regret}}}
\def\Trans{{P}}
\def\eTrans{{\hat{P}}}  
\def\eR{{\hat{R}}}  
\def\er{{\hat{r}}}
\def\idxTrans#1{{I_{#1}}}
\def\idxR#1{{J_{#1}}}
\def\sumR#1{M_{#1}'}
\def\cntR#1{M_{#1}}
\def\cntTrans#1{N_{#1}}
\def\sumTrans#1{N_{#1}'}
\def\bR{{\beta}}
\def\bTrans{{b}}

\def\unif{{\textnormal{unif}}}
\def\diag{{\textnormal{diag}}}

\def\uV{{\overline{V}}}
\def\lV{{\underline{V}}}
\newcommand\oV[1]{{V_{#1}^{*}}}
\def\uQ{{\overline{Q}}}

\def\textforall{{\text{~for all~}}}
\def\textand{{\text{~and~}}}

\newcommand{\gA}{\mathcal{A}}
\newcommand{\gB}{\mathcal{B}}
\newcommand{\gF}{\mathcal{F}}
\newcommand{\gG}{\mathcal{G}}
\newcommand{\gH}{\mathcal{H}}
\newcommand{\gI}{\mathcal{I}}
\newcommand{\gR}{\mathcal{R}}
\newcommand{\gS}{\mathcal{S}}
\newcommand{\gQ}{\mathcal{Q}}
\newcommand{\gX}{\mathcal{X}}
\newcommand{\gY}{\mathcal{Y}}

\newcommand{\sC}{\mathbb{C}}
\newcommand{\sI}{\mathbb{I}}
\newcommand{\sP}{\mathbb{P}}
\newcommand{\sR}{\mathbb{R}}

\DeclareMathAlphabet{\mathsfit}{T1}{\sfdefault}{\mddefault}{\sldefault}
\SetMathAlphabet{\mathsfit}{bold}{T1}{\sfdefault}{\bfdefault}{\sldefault}

\newcommand{\rx}{\mathsfit{x}}
\newcommand{\ry}{\mathsfit{y}}
\newcommand{\rvy}{\boldsymbol{\mathsfit{y}}}

\makeatletter
\newcommand{\vast}{\bBigg@{3}}
\newcommand{\Vast}{\bBigg@{3.5}}
\makeatother

\begin{document}

\title{Towards Minimax Optimal Reinforcement Learning in Factored Markov Decision Processes}

\author{%
  \name Yi Tian\footnotemark[1] \email{yitian@mit.edu}  \\
  \name Jian Qian\footnotemark[1] \email{jianqian@mit.edu} \\
  \name Suvrit Sra \email{suvrit@mit.edu} \\
  \addr{Massachusetts Institute of Technology, Cambridge, MA, 02139, USA}
}

\renewcommand*{\thefootnote}{\fnsymbol{footnote}}
\footnotetext[1]{These authors contributed equally to this work.}
\renewcommand*{\thefootnote}{\arabic{footnote}}

\maketitle

\begin{abstract}
    We study minimax optimal reinforcement learning in episodic factored Markov decision processes (FMDPs), which are MDPs with conditionally independent transition components. Assuming the factorization is known, we propose two model-based algorithms. The first one achieves minimax optimal regret guarantees for a rich class of factored structures, while the second one enjoys better computational complexity with a slightly worse regret. A key new ingredient of our algorithms is the design of a bonus term to guide exploration. We complement our algorithms by presenting several structure-dependent lower bounds on regret for FMDPs that reveal the difficulty hiding in the intricacy of the structures.
\end{abstract}

\section{Introduction}
In reinforcement learning (RL) an agent interacts with an unknown environment seeking to maximize its cumulative reward. The dynamics of the environment and the agent's interaction with it are typically modeled as a Markov decision process (MDP). We consider the specific setting of \emph{episodic MDPs} with a fixed interaction horizon. Here, at each step the agent observes the current state of the environment, takes an action, and receives a reward. Given the  agent's action, the environment then transits to the next state. The interaction horizon is the number of steps in an episode, while both the transitions and rewards may be unknown to the agent.

The agent's performance is quantified using regret: the gap between the cumulative rewards the agent receives and those obtainable by following an optimal policy. An optimal RL algorithm then becomes one that incurs the minimum regret. 
For episodic MDPs, the minimax regret bound is known to be  $\bigotilde(\sqrt{HSAT})$~\citep{azar2017minimax} (the notation $\bigotilde(\cdot)$ hides log factors), where $H$, $S$, $A$, $T$ denote the horizon, the size of the state space, the size of the action space, and the total number of steps, respectively.

\textbf{Background and our focus.} In problems with large state and action spaces even the  $\bigotilde(\sqrt{SA})$ dependence is impractical. 
We focus on problems where one can circumvent this dependence. Specifically, we focus on problems with conditional independence structure that can be modeled via \emph{factored MDPs} (FMDPs)~\citep{boutilier1995exploiting, boutilier1999decision}.
The state for an FMDP is represented as a tuple, where each component is determined by a small portion of the state-action space, termed its ``scope''. For example, for a home robot, whether the table is clean has nothing to do with whether the robot vacuums the floor. FMDPs also arise naturally from cooperative RL~\citep{guestrin2002multiagent, radoszycki2015solving}, and furthermore, they are useful to model subtask dependencies in hierarchical RL~\citep{sohn2020meta}.

A key benefit of FMDPs is their more compact representation that requires exponentially smaller space to store state transitions than corresponding nonfactored MDPs. Early works~\citep{kearns1999efficient, guestrin2002algorithm, strehl2007model,szita2009optimistic} show that FMDPs can also reduce the sample complexity exponentially, 
albeit with some higher order terms. 
\citeauthor{osband2014near}~\citep{osband2014near} and \citeauthor{xu2020near}~\citep{xu2020near} develop near-optimal regret bound of FMDPs in the episodic and nonepisodic settings, respectively. For an FMDP with $m$ conditionally independent components, both works bound the regret by $\bigotilde(\sum_{i=1}^{m} \sqrt{D^2S_iX[\idxTrans{i}] T})$,~\footnote{The dependence on reward factorization is omitted here and below, unless otherwise stated. The omission also corresponds to the case of known rewards.} where $D$ is the diameter of the MDP~\citep{jaksch2010near} and $X[\idxTrans{i}]$ is the size of the scope of the $i$th state component. 

\textbf{Main difficulty.}
For (nonfactored) MDPs, the above regret bounds from~\citep{osband2014near,xu2020near} reduce to $\bigotilde(\sqrt{D^2S^2 A T})$, or essentially $\bigotilde(\sqrt{H^2S^2 A T})$ in the episodic setting. Either bound has a gap to the corresponding lower bound, leading one to wonder whether we can adapt the techniques from MDPs to obtain tight bounds for FMDPs. It turns out to be nontrivial to answer this question. 

Indeed, the key technique that yields a minimax regret bound for MDPs is a careful design of the bonus term that ensures optimism, by applying scalar concentration instead of $\normlone$-norm concentration (Lemma~\ref{lem:concen-l1}) on a vector~\citep{azar2017minimax}.
To adapt the same idea to FMDPs, perhaps the most natural choice is to estimate the transition of each component separately, and then to sum up the component-wise bonuses to guide overall exploration. With the bonuses derived from $\normlone$-norm concentration, this method ensures optimism, leading to the regret bound in~\citep{osband2014near, xu2020near}. But with the bonuses derived from scalar concentration (which was key to~\citep{azar2017minimax}), it is hard, if possible, to show that this naive sum of component-wise bonuses ensures optimism (see~\eqref{eqn:nat-extension}), and thus is the major obstacle towards removing the suboptimal $\sqrt{S_i}$ factor from the regret bound.

On the other hand, there is no uniform nondegenerate lower bound that holds for any structure.
As a result, the dependence on structure has to be specified, making the lower bound statements subtle due to the intricacy of the structure.

\textbf{Our contributions.}
Our main contribution is to present two provably efficient algorithms for episodic FMDPs, including one with tight regret guarantees for a rich class of structures. We overcome the above noted difficulty by introducing an additional bonus term via an inverse telescoping technique.
Our first algorithm F-UCBVI, uses Hoeffding-style bonuses and achieves the $\bigotilde(\sum_{i=1}^{m} \sqrt{H^2 X[\idxTrans{i}] T})$ regret (Theorem~\ref{thm:f-ucbvi}).
Our second algorithm F-EULER, uses Bernstein-style bonuses and achieves a problem-dependent regret (Theorem~\ref{thm:f-euler}) that reduces to $\bigotilde(\sum_{i=1}^{m} \sqrt{H X[\idxTrans{i}] T})$ in the worst case (Corollary~\ref{cor:f-euler}), and thus closes the gap to the $\Omega(\sqrt{HSAT})$ lower bound (see, e.g.,~\citep{jin2020lecture}) when reduced to the nonfactored case. Despite an extra $\sqrt{H}$ dependence, F-UCBVI has a much simpler form of bonus and enjoys a lower computational complexity.

Furthermore, we identify some reasonably general factored structures for which we prove lower bounds (Theorems~\ref{thm:fmdp-lb-degen-1},~\ref{thm:fmdp-lb-degen-2},~\ref{thm:fmdp-lb-nondegen} and~\ref{thm:fmdp-lb-loop}). Our construction builds upon the lower bounds for multi-armed bandits (MABs)~\citep{lattimore2018bandit}, in a similar manner to~\citep{dann2015sample, jin2020lecture}. As a side remark, the JAO MDP, used to establish the $\Omega(\sqrt{DSAT})$ lower bound for nonepisodic MDPs~\citep{jaksch2010near}, only establishes a ${\Omega}(\sqrt{HSAT/\log T})$ lower bound through a direct extension (Appendix~\ref{sec:jao}). This result misses a log factor resulting from the episodic resetting, and it is not clear whether it can be tightened with the same construction.

\subsection{Related work}

\textbf{Regret bound for episodic MDPs.}
For episodic MDPs, built upon the principle of optimism in the face of uncertainty, recent works establish the $\bigotilde(\sqrt{HSAT})$ worst-case regret bound~\citep{azar2017minimax, kakade2018variance, zanette2019tighter, simchowitz2019non, efroni2019tight}, 
matching the $\Omega(\sqrt{HSAT})$ lower bound (see, e.g.,~\citep{jin2020lecture}) up to log factors.
For model-free algorithms, the state-of-the-art worst-case regret is $\bigotilde(\sqrt{H^2SAT})$~\citep{zhang2020almost}.
Sharper problem-dependent regret bounds~\citep{zanette2019tighter,simchowitz2019non} also exist in this setting.
Posterior sampling for RL (PSRL)~\citep{osband2013more} offers an alternative to optimism-based methods and can perform well empirically~\citep{osband2017posterior}.
The best existing Bayesian regret bound of PSRL for episodic MDPs is $\bigotilde(\sqrt{H^2SAT})$~\citep{osband2017posterior}.
Among these works, our work is mostly motivated by~\citep{azar2017minimax, zanette2019tighter}.
Specifically, for nonfactored MDPs, F-UCBVI and F-EULER reduce to the UCBVI-CH~\citep{azar2017minimax} and EULER~\citep{zanette2019tighter} algorithms, respectively, after which we name our algorithms.

\textbf{Planning in FMDPs.}
For nonepisodic FMDPs, various works exploit the factored structure to develop efficient (approximate) dynamic programming methods~\citep{boutilier2000stochastic, koller2000policy, guestrin2001max, schuurmans2002direct, guestrin2003efficient}.
See~\citep{degris2013factored} for an overview.
For episodic FMDPs, however, efficient approximate planning is yet to be understood.
Planning is a subroutine of model-based RL algorithms.
The focus of this work is the sample efficiency, and we adopt a value iteration procedure that has the same computational complexity as that for a nonfactored MDP.

\textbf{Reinforcement learning in FMDPs.}
To efficiently find the optimal policy assuming known factorization yet unknown state transitions or rewards, early works aim to develop algorithms with PAC guarantees~\citep{kearns1999efficient, guestrin2002algorithm, strehl2007model, szita2009optimistic}.
More recent works place a focus on the regret analysis and establish $\bigotilde(\sum_{i=1}^{m} \sqrt{D^2 S_i X[\idxTrans{i}]T})$ regret in either the episodic or nonepisodic setting~\citep{osband2014near, xu2020near}, also translating to $\bigotilde(\sum_{i=1}^{m} \sqrt{H^2 S_i X[\idxTrans{i}]T})$ in terms of horizon $H$ for episodic problems.
Our work considers minimax optimal regret guarantees in the episodic setting.
In this setting, we provide the first $\bigotilde(\sum_{i=1}^{m} \sqrt{HX[\idxTrans{i}]T})$ regret bound and the first formal treatment of the lower bounds, by which we show that our algorithm is minimax optimal for a rich class of factored structures.

\textbf{Structure learning in FMDPs.}
Assuming unknown structure except for known state factorization, various structure learning algorithms exist in the literature, e.g., SPITI~\citep{degris2006learning}, SLF-Rmax~\citep{strehl2007efficient}, Met-Rmax~\citep{diuk2009adaptive} and LSE-Rmax~\citep{chakraborty2011structure}.
See~\citep[Chapter 7]{chakraborty2014sample} for an overview of these methods.
More recent works include~\citep{hallak2015off, guo2018sample}.
To the best of our knowledge, there is no regret analysis in setting, which is an interesting future direction.

\section{Preliminaries and problem formulation}

\subsection{Learning in episodic MDPs}

An episodic MDP is described by the quintuple $(\gS, \gA, \Trans, r, H)$, where $\gS$ is the state space, $\gA$ is the action space, 
$\Trans: \gS \times \gA \times \gS \mapsto [0, 1]$ is the transition function such that $\Trans(\cdot \vert s, a)$ is in the probability simplex on $\gS$, denoted by $\Delta(\gS)$, 
$r: \gS \times \gA \mapsto [0, 1]$ is the reward function, and $H$ is the horizon or the length of an episode. The focus in this paper is the stationary case, where both the transition function $\Trans$ and the reward function $r$ remain unchanged across steps or episodes.

An RL agent interacts episodically with the environment starting from arbitrary initial states. The role of the agent is abstracted into a sequence of policies. For any natural number $n$, we use $[n]$ to denote the set $\{1, 2, \cdots, n\}$. Then for finite-horizon MDPs, the policy $\pi: \gS \times [H] \mapsto \gA$ maps a state and a time step to an action. 
Let $\pi_k$ denote the policy in the $k$th episode. An RL algorithm, therefore, specifies the updates of the sequence $\{\pi_k\}$.
Our algorithms use deterministic policies, while the lower bounds hold for the more general class of stochastic policies. 

The Q-value function $Q_{h}^{\pi}$ and state-value function $V_{h}^{\pi}$ measure the goodness of a state-action pair and of a state following the policy $\pi$ starting from the $h$th step in an episode, respectively. For state $s\in \gS$ and action $a\in \gA$, these functions are defined by
\begin{align*}
    Q_h^{\pi}(s, a) = \E\Bigl[\nlsum_{t=h}^{H} r(s_t, a_t)\, \bigl\lvert\, s_h = s, a_h = a\Bigr],
    \quad V_h^{\pi}(s) = Q_h^{\pi}(s, \pi(s, h)).
\end{align*}
There exists an optimal policy $\pi^{*}$ such that $V_h^{\pi^{*}}(s) \ge V_h^{\pi}(s)$ for all $s\in \gS, h\in [H]$~\citep{puterman2014markov}. Let $V_h^{*}$ denote $V_h^{\pi^{*}}$, and let $s_{k, 1}$ be the initial state at the $k$th episode; we then define the regret over $K$ episodes ($T = KH$ steps) as
\begin{align*}
    \Regret(K) := \nlsum_{k=1}^{K} \left( V_1^{*}(s_{k, 1}) - V_1^{\pi_k}(s_{k, 1}) \right).
\end{align*}

\subsection{Factored MDPs}

Episodic FMDPs inherit the above definitions, but we also need additional notation that we adopt from~\citep{osband2014near} to describe the factored structure. Let the state-action space $\gX := \gS \times \gA$, so that $x = (s, a) \in \gX$ denotes a state-action pair. 
Let $m, n, l$ be natural numbers. The state space is factorized as 
$\gS = \bigotimes_{i=1}^{m} \gS_i \equiv \gS_1 \times \cdots \times \gS_m$, 
using which we write a state $s \in \gS$ as the tuple $(s[1], \cdots, s[m])$. Similarly, the state-action space is factorized as 
$\gX = \bigotimes_{i=1}^{n} \gX_i \equiv \gX_1 \times \cdots \times \gX_n$. 
\begin{definition}[Scope operation]  \label{def:scope}
    For any natural number $n$, any index set $I \subset [n]$ and any factored set $\gX = \bigotimes_{i=1}^{n} \gX_{i}$, define the scope operation $\gX[I] := \bigotimes_{i\in I} \gX_{i}$, where the indices in the Cartesian product are in ascending order by default.
    Correspondingly, define the scope variable $x[I] \in \gX[\idxTrans{i}]$ as the tuple of $x_i$ where $i\in I$. 
\end{definition}
The transition factored structure refers to the conditional independence of the transitions of the $m$ state components.
Mathematically, $\Trans(s'\vert s, a) = \Trans(s'\vert x) = \prod_{i=1}^{m} \Trans_i(s'[i] \vert x)$, 
where $\Trans_i(\cdot \vert x) \in \Delta(\gS_i)$.
Moreover, the transition of the $i$th state component is determined only by its scope, a subset of the state-action components.
Let $\idxTrans{i} \subset [n]$ be the scope index set for $i \in [m]$.
Then with a slight abuse of notation, $\Trans_i(s'[i] \vert x) \equiv \Trans_i(s'[i] \vert x[\idxTrans{i}])$ where ``$\equiv$'' denotes identity, and the factored transition is given by $P(s'\vert x) = \prod_{i=1}^{m} \Trans_i(s'[i] \vert x[\idxTrans{i}])$.

The reward function $r$ is stochastic and unknown in general, which can also be factored. 
Let $\idxR{i} \subset [n]$ be the reward scope index set for $i\in [l]$.
Then the factored reward is the sum of $l$ independent components $r_i(x) \equiv r_i(x[\idxR{i}])$, given by $r(x) = \sum_{i=1}^{l} r_i(x[\idxR{i}])$. 
Let $R := \E[r]$ and $R_i := \E[r_i]$ be their expectations.
Then $R = \sum_{i=1}^{l} R_i: \gX \to [0, 1]$ is deterministic.  
With the above notation, an FMDP model is described in general by
\begin{align}  \label{eqn:fmdp}
    \mdp = \bigl(
        \left\{\gX_{i}\right\}_{i=1}^{n}, 
        \left\{\gS_{i}, \idxTrans{i}, \Trans_{i}\right\}_{i=1}^{m}, 
        \left\{\idxR{i}, r_{i}\right\}_{i=1}^{l}, 
        H
    \bigr). 
\end{align}

In this paper, we assume unknown transition with known factorization and consider two cases about the knowledge of reward: (1) the reward $r$ is unknown with known factorization, and the agent receives the factored rewards $\{r_{i}\}_{i=1}^{l}$~\citep{osband2014near, xu2020near}; (2) the reward $r$ is known and not necessarily factored, in which case the expected reward $R$ is also known~\citep{azar2017minimax}.

\textbf{Notation.}
Throughout the paper, we use $S, A, X, S_i, X[I]$ to denote the finite cardinalities of the corresponding $\gS, \gA, \gX, \gS_i, \gX[I]$, respectively.
Then $S = \prod_{i=1}^{m} S_i$ and $X[I] \le X = SA$ for any $I\subset [n]$.
Since the abstraction of a state component $s[i]$ is meaningful only if $\gS_i \ge 2$, the number of state components $m$ is at most $\log_{2} S$
, which is treated as a negligible log factor whenever necessary.
For $x = (s, a) \in \gX$, we use $Q(s, a)$ and $Q(x)$ interchangeably to denote a function $Q$ defined on $\gX:= \gS \times \gA$.
A function $V: \gS \to \sR$ is also seen as a real vector indexed by $\gS$, denoted by $V \in \sR^{\gS}$. 
For vectors in $\sR^{\gS}$, let $|\cdot|, (\cdot)^2$ denote entrywise absolute values and squares, respectively. 

Define $\Trans(x) := \Trans(\cdot \vert x) \in \Delta(\gS)$. 
Moreover, for $\Trans_{i}(x) \equiv \Trans_{i}(x[\idxTrans{i}]) := \Trans_{i}(\cdot \vert x[\idxTrans{i}]) \in \Delta(\gS_i)$ and natural numbers $a \le b$, define $\Trans_{a:b}(x) := \prod_{i=a}^{b} \Trans_{i}(x) \in \Delta(\bigotimes_{i=a}^{b} \gS_i)$. In particular, $\prod_{i=1}^{m} \Trans_{i}(x) = \Trans(x) \in \Delta(\gS)$.
For $\Trans \in \Delta(\gS)$, $V \in \sR^{\gS}$, define their inner product $\innerp{\Trans}{V} := \sum_{s\in \gS} \Trans(s) V(s) = \E_{\Trans}[V(s)] \equiv \E_{\Trans}[V]$. The inner product of two elements in $\Delta(\gS_i)$ and $\sR^{\gS_i}$ respectively takes sum over $\gS_i$. For $\Trans_i \in \Delta(\gS_i)$ and $V \in \sR^{\gS}$, their inner product takes sum over $\gS$. In this case, note that $\innerp{\Trans_i}{V} \neq \E_{\Trans_i}[V]$, which is given by $\E_{\Trans_i}[V] \equiv \E_{\Trans_i}[V(s)] = \sum_{s[i]\in \gS} \Trans_i(s[i]) V(s) \in \sR^{\gS[-i]}$, where $\gS[-i] := \bigotimes_{j=1, j\neq i}^{m} \gS_{j}$.

\section{Factored UCBVI and regret bounds}

We are now ready to present the two algorithms for solving FMDPs that we propose. In this section, We introduce the general procedures (Algorithms~\ref{alg:f-ovi} and~\ref{alg:vi-optimism}) and discuss an instantiation called factored UCBVI (F-UCBVI) using Hoeffding-style optimism. In the next section, we use a Bernstein-style instantiation called factored EULER (F-EULER) that obtains better regret guarantees.

\begin{algorithm}[htbp]
    \caption{\texttt{F-OVI} (Factored Optimistic Value Iteration)}  \label{alg:f-ovi}
    \begin{algorithmic}[1]
        \Require Factored structure $\left\{\gX_{i}\right\}_{i=1}^{n}$, $\left\{\gS_{i}, \idxTrans{i}\right\}_{i=1}^{m}$, $\left\{\idxR{i}\right\}_{i=1}^{l}$, and horizon $H$
        \State Initialize counters $\cntTrans{i}(x[\idxTrans{i}]) = \sumTrans{i}(x[\idxTrans{i}], s[i]) = 0$, $\forall (i\in [m], x[\idxTrans{i}] \in \gX[\idxTrans{i}], s[i] \in \gS_i)$
        \vspace{2pt}
        \State Initialize counters $\cntR{i}(x[\idxR{i}]) = \sumR{i}(x[\idxR{i}]) = 0$, $\forall (i\in [l], x[\idxR{i}] \in \gX[\idxR{i}])$ and history $\gH_{r} = \emptyset$
        \vspace{2pt}
        \For{episode $k=1, 2, \cdots, K$}
            \vspace{2pt}
            \State Estimate transitions $\eTrans_{i}(s[i]\vert x[\idxTrans{i}]) = \tfrac{\sumTrans{i}(x[\idxTrans{i}], s[i])}{\max\{1, \cntTrans{i}(x[\idxTrans{i}])\}}$, $\forall (i\in [m], x[\idxTrans{i}] \in \gX[\idxTrans{i}], s[i] \in \gS_i)$
            \vspace{2pt}
            \State Estimate expected rewards $\eR_i(x[\idxR{i}]) = \frac{\sumR{i}(x[\idxR{i}])}{\max \{1,\cntR{i}(x[\idxR{i}]) \}}$, $\forall (i\in [l], x[\idxR{i}] \in \gX[\idxR{i}])$
            \vspace{2pt}
            \State Estimate expected overall reward $\eR(x) = \sum_{i=1}^l \eR_i(x[\idxR{i}])$, $\forall x\in \gX$ 
            \vspace{2pt}
            \State Obtain policy $\pi = \mathtt{VI\_Optimism}(\{\eTrans_{i}\}_{i=1}^{m}, \eR, \{\cntTrans{i}\}_{i=1}^{m},\{\cntR{i}\}_{i=1}^l, \gH_{r})$ 
            \vspace{2pt}
            \State Observe initial state $s_{1}$ from environment
            \vspace{2pt}
            \For{step $h=1, 2, \cdots, H$} 
                \vspace{2pt}
                \State Take action $a_{h} = \pi(s_{h}, h)$ and let $x_{h} = (s_{h}, a_{h})$
                \vspace{1pt}
                \State Receive reward $\{r_{i, h}\}_{i=1}^l$ and observe next state $s_{h+1}$
                \vspace{2pt}
                \State Update transition counters $\cntTrans{i}(x_{h}[\idxTrans{i}]) \mathrel{+}= 1$ and $\sumTrans{i}(x_{h}[\idxTrans{i}], s_{h+1}[i]) \mathrel{+}= 1$, $\forall i\in [m]$
                \vspace{2pt}
                \State Update reward counters $\cntR{i}(x_h[\idxR{i}]) \mathrel{+}= 1$ and $\sumR{i}(x_h[\idxR{i}]) \mathrel{+}= r_{i, h}$, $\forall i\in [l]$
                \vspace{1pt}
                \State Update reward history $\gH_{r} = \gH_{r} \cup \{(x_{h}, \{r_{i, h}\}_{i=1}^{l})\}$
            \EndFor
        \EndFor
      \end{algorithmic}
\end{algorithm}

The general optimistic model-based RL framework for FMDPs is described in \texttt{F-OVI} (Algorithm~\ref{alg:f-ovi}), which maintains an empirical estimate of the transition function, typically the maximum likelihood (ML) estimate.
For FMDPs, although a direct ML estimate of the overall transition is still possible, the product of the ML estimates of component-wise transition functions is preferable to exploit the factored structure.
The same argument applies to the empirical expected reward function.

As a subroutine of \texttt{F-OVI} (Algorithm~\ref{alg:f-ovi}), \texttt{VI\_Optimism} (Algorithm~\ref{alg:vi-optimism}) ensures an upper confidence bound (UCB) on the optimal value function, achieved by introducing bonuses derived from concentration inequalities.
The $\normlone$-norm concentration (Lemma~\ref{lem:concen-l1}) on the transition estimation leads to the UCRL-Factored algorithm, achieving the $\bigotilde(\sum_{i=1}\sqrt{D^2S_i X[\idxTrans{i}] T})$ regret~\citep{osband2014near}.
The scalar concentration on the transition estimation leads to the two new algorithms proposed in this work.

In the case of known rewards, there is no need to estimate the reward $\eR$ or to collect the reward history $\gH_{r}$ in \texttt{F-OVI} (Algorithm~\ref{alg:f-ovi}). Neither is there the need to introduce the reward bonus $\bR$ in \texttt{VI\_Optimism} (Algorithm~\ref{alg:vi-optimism}).
In both procedures, the empirical $\eR$ is replaced by the true $R$.

\textbf{Hoeffding-style bonus.}
Let $L = \log (16mlSXT / \delta)$ be a log factor.
With a slight abuse of notation, let $\cntTrans{i}(x) \equiv \cntTrans{i}(x[\idxTrans{i}])$.
The transition bonus of F-UCBVI is given by
\begin{align} \label{eqn:trans-bonus-f-ucbvi}
    \bTrans(x) := \nlsum_{i=1}^{m} H\sqrt{\tfrac{L}{2\cntTrans{i}(x)}} + \nlsum_{i=1}^{m} \nlsum_{j=i+1}^{m} 2HL \sqrt{\tfrac{S_i S_j}{\cntTrans{i}(x) \cntTrans{j}(x)}},
\end{align}
which consists of both a component-wise bonus term and a cross-component bonus term.
In the case of unknown rewards, the reward bonus of F-UCBVI is given by 
\begin{align}  \label{eqn:r-bonus-f-ucbvi}
    \bR(x) := \nlsum_{i=1}^{l}\sqrt{\tfrac{L}{2\cntR{i}(x)}}.
\end{align}
For the transition and reward bonuses here and below, zero-valued $\cntTrans{i}(x)$ and $\cntR{i}(x)$ in the denominators are replaced by $1$ to prevent zero division, as in \texttt{F-OVI} (Algorithm~\ref{alg:f-ovi}).
For F-UCBVI, the reward history $\gH_{r}$ and the lower confidence bound (LCB) $\lV_{h}$ are redundant and the related updates are removable.
F-UCBVI has the following regret guarantees.

\begin{algorithm}[tbp]
    \caption{\texttt{VI\_Optimism} (Value Iteration with Optimism)}  \label{alg:vi-optimism}
    \begin{algorithmic}[1]
        \Require Empirical transitions $\{\eTrans_{i}\}_{i=1}^{m}$, empirical expected reward $\eR$, transition counters $\{\cntTrans{i}\}_{i=1}^{m}$, reward counters $\{\cntR{i}\}_{i=1}^l$, reward history $\gH_{r}$
        \vspace{1pt}
        \State Initialize UCB and LCB $\uV_{H+1}(s) = \lV_{H+1}(s) = 0, \forall s\in \gS$; let $\eTrans(x) = \prod_{i=1}^{m}\eTrans_{i}(x), \forall x\in \gX$
        \vspace{1pt}
        \For{step $h = H, H-1, \cdots, 1$}
            \For{state $s \in \gS$}
                \For{action $a \in \gA$, with $x = (s, a)$}
                    \State $\uQ_{h}(x) 
                    = \min\left\{ H - h + 1, \eR(x) + \innerp{\eTrans(x)}{\uV_{h+1}} + \bTrans(x) + \bR(x)\right\}$
                \EndFor
                \vspace{1pt}
                \State Let $\pi(s, h) = \argmax_{a\in \gA} \uQ_{h}(s, a)$ and $x = (s, \pi(s, h))$
                \vspace{1pt}
                \State Update UCB $\uV_{h}(s) =  \uQ_{h}(x)$
                \State Update LCB $\lV_{h}(s) = \max\left\{0, \eR(x) + \innerp{\eTrans(x)}{\lV_{h+1}} - \bTrans(x) -\bR(x) \right\}$ 
            \EndFor
        \EndFor
        \State \Return policy $\pi$
    \end{algorithmic}
\end{algorithm}

\begin{theorem}[Worst-case regret bounds of F-UCBVI]  \label{thm:f-ucbvi}
    For any FMDP specified by~\eqref{eqn:fmdp}, in the case of unknown rewards, with probability $1 - \delta$, the regret of F-UCBVI in $K$ episodes is upper bounded by
    \begin{align}  \label{eqn:regret-f-ucbvi}
        \bigotilde\Bigl( \nlsum_{i=1}^{m} \sqrt{H^2X[\idxTrans{i}] T} + \nlsum_{i=1}^{l} \sqrt{X[\idxR{i}] T}\Bigr).
    \end{align}
    In the case of known rewards, the term $\nlsum_{i=1}^{l} \sqrt{X[\idxR{i}]T}$ is dropped.
\end{theorem}

In Theorem~\ref{thm:f-ucbvi} (and Theorem~\ref{thm:f-euler}, Corollary~\ref{cor:f-euler} below), 
the $\sqrt{T}$ terms dominate the upper bounds under the assumption that $T \ge \mathrm{poly}(m, l, \max_{i}S_{i}, \max_{i}X[\idxTrans{i}], H)$, which can be exponentially smaller than $\mathrm{poly}(S, A, H)$, required in general by nonfactored algorithms~\citep{azar2017minimax, kakade2018variance, zanette2019tighter, simchowitz2019non, efroni2019tight, jin2018q,zhang2020almost}.
This weaker requirement on $T$ is also true for~\citep{osband2014near} and can be viewed as another benefit from the factored structure.
The lower-order terms in Section~\ref{sec:proof-sketch} are also under this assumption.

In~\eqref{eqn:regret-f-ucbvi}, the regret bound consists of a transition-induced term and a reward-induced term, which correspond to the sum over time of the component-wise transition bonus term in~\eqref{eqn:trans-bonus-f-ucbvi} and the reward bonus~\eqref{eqn:r-bonus-f-ucbvi}, respectively. 
The cross-component transition bonus term in~\eqref{eqn:trans-bonus-f-ucbvi} diminishes fast and its sum over time turns out to be negligible. 
Nevertheless, this bonus term plays a significant role in guiding exploration at the early stage of the algorithm, when it is large due to the dependence on $S_i$.
Note that the cross-component form is not absolutely necessary, e.g., we can use the AM–GM inequality to decouple it as 
\begin{align*}
    \nlsum_{i=1}^{m} \nlsum_{j=i+1}^{m} 2HL \sqrt{\tfrac{S_i S_j}{\cntTrans{i, k}(x) \cntTrans{j, k}(x)}} 
    \le (m-1) HL \nlsum_{i=1}^{m} \tfrac{S_i}{\cntTrans{i, k}(x)}.
\end{align*}
In this work, we use the cross-component form, as it is tighter and implies an intertwined behavior among the state components.

The component-wise transition bonus term in~\eqref{eqn:trans-bonus-f-ucbvi} is exactly the bonus in the UCBVI-CH algorithm~\citep{azar2017minimax}.
It may also be feasible to derive a factored version of UCBVI-BF~\citep{azar2017minimax}, 
but we switch to an upper-lower bound exploration strategy like EULER~\citep{zanette2019tighter} for the Bernstein-style optimism, due in part to its problem-dependent regret.

\section{Factored EULER and regret bounds}

To describe the Bernstein-style bonus, we first add some convenient notation. We omit the dependence of $\Trans_{i}(x)$ and $\Trans(x)$ on $x$; for $\Trans_{i}\in \Delta(\gS_{i}), \Trans = \prod_{i=1}^{m} \Trans_{i} \in \Delta(\gS)$ and $V \in \sR^{\gS}$, define 
\begin{align}  \label{eqn:gi}
  g_{i}(\Trans, V) := 2\sqrt{L} \sqrt{\Var_{\Trans_{i}} \E_{\Trans_{-i}} [V]},
\end{align} 
where $\E_{\Trans_{-i}} [V] := \E_{\Trans_{1:i-1} \Trans_{i+1:m}} [V] \in \sR^{\gS_i}$ takes expectation over the $m - 1$ components.
With the notation from~\citep{zanette2019tighter}, let $\|X\|_{2, P} = \sqrt{\E_{P}[X^2]}$ be the $\normltwo$-norm on 
the equivalence classes of almost surely the same 
random variables.

\textbf{Bernstein-style bonus.}
The transition bonus of F-EULER is defined by
\begin{equation}  \label{eqn:trans-bonus-f-euler}
    \begin{aligned}
        \bTrans(x)
        &:= \nlsum_{i=1}^{m} \tfrac{g_{i}(\eTrans(x), \uV_{h+1})}{\sqrt{\cntTrans{i}(x)}}
        + \nlsum_{i=1}^{m} \tfrac{\sqrt{2L} \left\| \uV_{h+1} - \lV_{h+1} \right\|_{2, \eTrans(x)}}{\sqrt{\cntTrans{i}(x)}} \\
        &\quad + \nlsum_{i=1}^{m} \nlsum_{j=i+1}^{m} 11HL \sqrt{\tfrac{S_i S_j}{\cntTrans{i}(x) \cntTrans{j}(x)}} + \nlsum_{i=1}^{m} \tfrac{5HL}{\cntTrans{i}(x)}.
    \end{aligned}
\end{equation}
In the case of unknown rewards, the reward bonus of F-EULER is defined by 
\begin{align}  \label{eqn:r-bonus-f-euler}
    \bR(x) := \nlsum_{i=1}^{l} \sqrt{\tfrac{4 \SV[\er_i(x)] L}{\cntR{i}(x)}} + \nlsum_{i=1}^{l} \tfrac{14L}{3\cntR{i}(x)},
\end{align}
where $\SV[\er_i(x)]$ is sample variance of the reward component, calculated using the history $\gH_r$.
Online algorithms like Welford's method~\citep{welford1962note} can efficiently compute $\eR_{i}$ and $\SV[\er_i(x)]$ jointly. F-EULER has the following problem-dependent regret guarantees.

\begin{theorem}[Problem-dependent regret bounds of F-EULER]  \label{thm:f-euler}
    For any FMDP specified by~\eqref{eqn:fmdp}, let 
    \begin{align*}
        \gQ_{i} := \max_{x\in \gX, h\in H} \left\{ \Var_{\Trans_{i}(x)} \E_{\Trans_{-i}(x)} [\oV{h+1}] \right\}, \quad 
        \gR_{i} := \max_{x\in \gX} \left\{ \Var [r_{i}(x)] \right\}.
    \end{align*}
    Let $\gG$ be the upper bound of $\sum_{h=1}^{H} R(x_h)$ for any initial state and any $policy$.
    In the case of known rewards, with probability $1 - \delta$, the regret of F-EULER in $K$ episodes is upper bounded by
    \begin{align}
        & \min\bigl\{
            \bigotilde\bigl(\nlsum_{i=1}^{m} \sqrt{\gQ_i X[\idxTrans{i}]T}\bigr), 
            \bigotilde\bigl(\nlsum_{i=1}^{m} \sqrt{\gG^2X[\idxTrans{i}]K}\bigr) 
        \bigr\} \\
        &\quad + \min\bigl\{
            \bigotilde\bigl(\nlsum_{i=1}^{l} \sqrt{\gR_{i} X[\idxR{i}]T} \bigr), 
            \bigotilde\bigl(\nlsum_{i=1}^{l} \sqrt{\gG^2X[\idxR{i}]K} \bigr) 
        \bigr\}.  \label{eqn:f-euler-pd-rwd}
    \end{align}
    In the case of known rewards, the term~\eqref{eqn:f-euler-pd-rwd} is dropped.
\end{theorem}

F-EULER does not assume knowledge of the actual $\gQ_i, \gR_i$ and $\gG$ and its regret bounds (Theorem~\ref{thm:f-euler}) enjoy the same problem-dependent advantages as discussed in the nonfactored case~\citep{zanette2019tighter}.
$r(\cdot) \in [0, 1]$ yields $\gG \le H$ and $\gR_{i} \le 1$ for all $i\in [l]$. 
Hence, Theorem~\ref{thm:f-euler} implies the following worst-case regret.

\begin{corollary}[Worst-case regret bounds of F-EULER]  \label{cor:f-euler}
    For any FMDP specified by~\eqref{eqn:fmdp}, in the case of unknown rewards, with probability $1 - \delta$, the regret of F-EULER in $K$ episodes is upper bounded by 
    \begin{align}  \label{eqn:f-euler-wc}
        \bigotilde\left(\nlsum_{i=1}^{m} \sqrt{HX[\idxTrans{i}]T} + \nlsum_{i=1}^{l} \sqrt{X[\idxR{i}]T}\right).
    \end{align}
    In the case of known rewards, the term $\nlsum_{i=1}^{l} \sqrt{X[\idxR{i}]T}$ is dropped.
\end{corollary}

In~\eqref{eqn:f-euler-wc}, the regret bound again consists of a transition-induced term and a reward-induced term, resulting from the sum over time of the first term in~\eqref{eqn:trans-bonus-f-euler} and the first term in~\eqref{eqn:r-bonus-f-euler}, respectively. 
For nonfactored MDPs, the reward-induced term $\bigotilde(\sqrt{SAT})$ is absorbed into the transition-induced term $\bigotilde(\sqrt{HSAT})$, while for FMDPs, these two terms depend on their factorization structures, and do not absorb each other in general.
Corollary~\ref{cor:f-euler} indicates that F-EULER improves the transition-induced term of F-UCBVI by a $\sqrt{H}$ factor. On the other hand, F-EULER has a worse computational complexity than F-UCBVI. Even if neglecting the computation on the reward bonus (since efficient methods exist), in each run of the \texttt{VI\_Optimism} subroutine, F-EULER has $\bigo(mS^2AH)$ computational complexity, $m$ times more than F-UCBVI.

\section{Proof sketch}
\label{sec:proof-sketch}

We present here the proof ideas of F-UCBVI in the case of known rewards,
which sheds light on our main new techniques. The full proof is deferred to Appendix~\ref{sec:appx-f-ucbvi}. We add a subcript $k$ to $\eTrans_{i}, \eTrans, \cntTrans{i}, \cntR{i}, \uV_{h}, \bTrans$ in Algorithms~\ref{alg:f-ovi} and~\ref{alg:vi-optimism} to denote the corresponding quantities in the $k$th episode.
For inductive purpose, let $V_{H+1}^{\pi}(s) = 0$ for all $\pi\in \Pi, s \in \gS$, so that $V_{h}^{\pi}(s) = R(x) + \innerp{\Trans(x)}{V_{h+1}^{\pi}}$ with $x = (s, \pi(s, h))$ for all $\pi\in \Pi, s\in \gS, h \in [H]$, where $\Pi$ is the set of all deterministic policies.

\textbf{The mechanism of optimism.} The purpose of the bonus is to ensure optimism, which means obtaining an entrywise UCB $\uV_{k, h}$ on $\oV{h}$ for all $k\in [K], h\in [H]$.

For $h = H + 1$, this is certainly the case, since they are both entrywise $0$.
For $h\in [H]$, $\forall s\in \gS$, with $x_{k, h} = (s, \pi_k(s, h))$ and $x_{h}^{*} = (s, \pi^{*}(s, h))$, 
\begin{align*}
    \uV_{k, h}(s) -\oV{h}(s) 
    &\ge \underbrace{\innerpc{\eTrans_{k}(x_{h}^{*})}{V_{k, h+1} - \oV{h+1}}}_{\text{$\ge 0$ by inductive assumption}} 
    + \underbrace{\innerpc{\eTrans_{k}(x_{h}^{*}) - \Trans(x_{h}^{*})}{\oV{h+1}}}_{\text{transition estimation error}}
    + \bTrans_{k}(x_{h}^{*}).
\end{align*}
To form a valid backward induction argument, we expect the bonus to be a UCB of the absolute value of the transition estimation error.
For nonfactored MDPs, with $\cntTrans{k}(x)$ being the number of visits to $x$ before the $k$th episode, a direct application of Hoeffding's inequality (Lemma~\ref{lem:concen-h}) yields that for any $x\in \gX$, 
\begin{align*}
    \bigl| \innerpc{\eTrans_{k}(x) - \Trans(x)}{\oV{h+1}} \bigr| \le H\sqrt{\tfrac{L}{2\cntTrans{k}(x)}},
\end{align*}
which determines the transition bonus.
For FMDPs, a natural extension is to expect 
\begin{align}  \label{eqn:nat-extension}
    \bigl| \bigl\langle \eTrans_{k}(x) - \Trans(x), \oV{h+1}\bigr\rangle \bigr| 
    \stackrel{?}{\le} \nlsum_{i=1}^{m} \bigl| \bigl\langle{\eTrans_{i, k}(x) - \Trans_{i}(x)},{\oV{h+1}}\bigr\rangle \bigr|
    \le \nlsum_{i=1}^{m} H\sqrt{\tfrac{L}{2\cntTrans{i, k}(x)}},
\end{align}
which can then be used as the transition bonus.
Although it holds that 
\begin{align*}
  \bigl| \bigl\langle{\eTrans_{k}(x) - \Trans(x)},{\oV{h+1}}\bigr\rangle \bigr| 
    \le \nlsum_{i=1}^{m} \bigl\langle{\bigl|\eTrans_{i, k}(x) - \Trans_{i}(x)\bigr|},{\oV{h+1}}\bigr\rangle, 
\end{align*}
the inequality in question in~\eqref{eqn:nat-extension} fails to hold (e.g., using $S_i = m = 2$ one obtains a quick contradiction).
To address this difficulty, we adopt an inverse telescoping technique.
Omitting the dependence of the transitions $\eTrans_{k}(x), \Trans(x), \eTrans_i(x), \Trans_i(x)$ on $x$, we rewrite $\langle{\eTrans_{k} - \Trans},{\oV{h+1}}\rangle$ as
\begin{align}
    &\innerpc{\nlprod_{i=1}^{m} \eTrans_{i, k} - \nlprod_{i=1}^{m} \Trans_{i}}{\oV{h+1}} 
    = \innerpc{\nlsum_{i=1}^{m} (\eTrans_{i, k} - \Trans_{i}) \Trans_{1:i-1} \eTrans_{i+1:m, k}}{\oV{h+1}} \nonumber \\
    = &\nlsum_{i=1}^{m} \innerpc{\eTrans_{i, k} - \Trans_{i}}{\E_{\Trans_{1:i-1}} \E_{\Trans_{i+1:m}} [\oV{h+1}]} \label{eqn:est-err-decomp-1} \\
    &\quad + \nlsum_{i=1}^{m} \innerpc{\eTrans_{i, k} - \Trans_{i}}{\E_{\Trans_{1:i-1}} (\E_{\eTrans_{i+1:m, k}} - \E_{\Trans_{i+1:m}}) [\oV{h+1}]},  \label{eqn:est-err-decomp-2}
\end{align}
where the first equality is referred to as the ``inverse telescoping'' technique, and the second equality adds then subtracts the same term \eqref{eqn:est-err-decomp-1}, which can be bounded by scalar concentration, leading to the component-wise bonus term in~\eqref{eqn:trans-bonus-f-ucbvi}.
\eqref{eqn:est-err-decomp-2} is bounded by $\normlone$-norm concentration (Lemma~\ref{lem:concen-l1}) with another use of the inverse telescoping technique, leading to the cross-component bonus term in~\eqref{eqn:trans-bonus-f-ucbvi}.

\textbf{Regret bound.} In the sequel, we use ``$\lesssim, \approx$'' to represent ``$\le, =$'' hiding constants and lower-order terms.
Let $w_{k, h}(x)$ denote the visit probability of $x \in \gX$ at step $h\in [H]$ following policy $\pi_k$ during the interaction.
Then the optimism ensures that the regret in $K$ episodes is upper bounded by 
\begin{align*}
    &\nlsum_{k = 1}^{K} \uV_{k, 1}(s_{k, 1}) - V_1^{\pi_k}(s_{k, 1}) \\
    \approx & \sum_{k = 1}^{K} \sum_{h = 1}^{H} \sum_{x\in L_{k}} w_{k, h}(x) \bigl\{
        \bTrans_{k}(x) + \underbrace{\innerpc{\eTrans_{k}(x) - \Trans(x)}{\oV{h+1}}}_{\text{transition estimation error}}
        + \underbrace{\innerpc{\eTrans_{k}(x) - \Trans(x)}{\uV_{k, h+1} - \oV{h+1}}}_{\text{correction, sum is lower-order (Lemma~\ref{lem:cul-corr-h})}}
    \bigr\},
\end{align*}
where the ``good'' set $L_{k}$ is a notion of sufficient visits before the $k$th episode and the sum out of $L_{k}$ is a lower-order term (Lemma~\ref{lem:sum-out-of-good-set}).
The cumulative transition estimation error is upper bounded by
\begin{align*}
    \sum_{k = 1}^{K} \sum_{h = 1}^{H} \sum_{x\in L_{k}} w_{k, h}(x)
    \Bigl\{
        \sum_{i=1}^{m} H\sqrt{\tfrac{L}{2\cntTrans{i, k}(x)}} + \underbrace{\sum_{i=1}^{m} \sum_{j=i+1}^{m} 2HL \sqrt{\tfrac{S_i S_j}{\cntTrans{i, k}(x) \cntTrans{j, k}(x)}}}_{\text{sum is lower-order (Lemma~\ref{lem:sum-mixed-vr})}}
    \Bigr\} 
    \lesssim \sum_{i=1}^{m} H\sqrt{X[\idxTrans{i}] T L}.
\end{align*}
The upper bound of the cumulative transition bonus takes exactly the same form.
Thus, we obtain 
\begin{align*}
    \Regret(K) 
    \lesssim \nlsum_{i=1}^{m} H\sqrt{X[\idxTrans{i}] T L} 
    = \bigotilde(\nlsum_{i=1}^{m} \sqrt{H^2X[\idxTrans{i}] T}).
\end{align*}
\vskip1pt
\textbf{Extension to other conditions.}
The treatment of unknown rewards is standard~\citep{zanette2019tighter}.
The same mechanism of optimism applies to F-EULER, except that the scalar concentration uses Bernstein's inequality (Lemma~\ref{lem:concen-b}).
Please see Appendix~\ref{sec:appx-f-euler} for a complete proof.

\section{Lower bounds}

In this section, we present some information theoretic lower bounds for different factored structures of FMDP, which illustrate the structure-dependent nature of the lower bounds; a full characterization is still an open problem.

Recall the formal statement of the lower bound for MDPs that for any given natural numbers $S, A, H$, there is an MDP that has $S$ states, $A$ actions and horizon $H$ with unknown transition $\Trans$ and possibly known rewards $R$, such that the expected regret after $K$ episodes is $\Omega(\sqrt{HSAT})$.
Ideally, the lower bound of FMDPs should be stated in the same way,
specified for any natural numbers $m, n, l, H$ and any factored structure $\{\gS_i\}_{i=1}^{m}, \{\gX_{i}\}_{i=1}^{n}, \{\idxTrans{i}\}_{i=1}^{m}, \{\idxR{i}\}_{i=1}^{l}$.
A simple $\Omega(\sum_{i=1}^{m} \sqrt{HX[\idxTrans{i}]T} + \sum_{i=1}^{l} \sqrt{X[\idxR{i}]T})$ lower bound is tempting, 
yet far from the truth.
To avoid excessive subtlety, our lower bound discussion is restricted to the following normal factored structure, where the state-action factorization is an extension to the state factorization.

\begin{definition}[Normal factored structure]
    For natural numbers $m < n$, let the state space $\gS = \bigotimes_{i=1}^{m} \gS_i$ and the action space $\gA = \bigotimes_{i=1}^{n-m} \gA_i$. 
    The factored structure of an FMDP is normal if and only if the state-action space $\gX = \bigotimes_{i=1}^{n} \gX_i$ where $\gX_i = \gS_i$ for $i\in [m]$ and $\gX_{m+i} = \gA_{i}$ for $i\in [n-m]$.
\end{definition}

We now state the lower bounds for two degenerate structures.

\begin{theorem}[Degenerate case 1]  \label{thm:fmdp-lb-degen-1}
    For any algorithm, under the assumption of unknown rewards, for the normal factored structure that satisfies $\idxTrans{i} \subset [m]$ for all $i\in [m]$ and any $\idxR{i}$ for $i\in [l]$, 
    there is an FMDP with the specified structure such that for some initial states, the expected regret in $K$ episodes is at least $\Omega(\max_{i} \sqrt{X[\idxR{i}'] T})$ where $\idxR{i}' = \idxR{i} \cap \{m+1, \cdots, n\}$.
\end{theorem}

\begin{theorem}[Degenerate case 2]  \label{thm:fmdp-lb-degen-2}
    For any algorithm, under the assumption of unknown rewards, for the normal factored structure that satisfies $\idxR{i} \subset \{m+1, \cdots, n\}$ for all $i\in [l]$, there is an FMDP with the specified structure such that for some initial states, the expected regret in $K$ episodes is at least $\Omega(\max_{i} \sqrt{X[\idxR{i}] T})$.
\end{theorem}

Theorem~\ref{thm:fmdp-lb-degen-1} considers the degenerate case where the transitions do not depend on the action space at all.
Furthermore, the regret is always $0$ if the rewards do not depend on the action space either ($\idxR{i}' = \emptyset$ for all $i\in [l]$). 
Theorem~\ref{thm:fmdp-lb-degen-2} considers the degenerate case where the rewards have no dependence on the state space, and $X[\idxTrans{i}]$ vanishes in the lower bound for this MAB problem.
The following theorem identifies some nondegenerate constraints on the factored structures.

\begin{theorem}[Lower bound, a nondegenerate case]  \label{thm:fmdp-lb-nondegen}
    For any algorithm, under the assumption of known rewards, for the normal factored structure that satisfies $S_i\ge 3$, $i\in \idxTrans{i}$, $\exists j\in \{m+1, \cdots, n\}$ such that $j\in \idxTrans{i}$ for all $i\in [m]$ and $H\ge 2$, 
    there is an FMDP with the specified structure such that for some initial states, the expected regret in $K$ episodes is at least $\Omega(\max_{i}\sqrt{HX[\idxTrans{i}]T})$.
\end{theorem}

Theorem~\ref{thm:fmdp-lb-nondegen} indicates that the regret achieved by F-EULER is minimax optimal up to log factors for a rich class of structures, 
since $\bigotilde(\sum_{i=1}^{m} \sqrt{HX[\idxTrans{i}]T}) = \bigotilde(\max_{i}\sqrt{HX[\idxTrans{i}]T})$.
The main assumption here is that the transition of each state component depends on itself and at least one action component. 
Considering that it is satisfied in the nonfactored case, this assumption is not too restrictive.
A major obstacle to show the lower bound for a general factored structure is that some structures lack the self-loop property, i.e., the transition of a state component depends on itself so that we can say it \emph{stays} in the same state with a certain probability, which is key to the lower bound constructions for nonfactored MDPs~\citep{jaksch2010near, dann2015sample, jin2020lecture}.
However, by generalizing the self-loop property to the loop property, we can show a more general version of the nondegenerate lower bound.

For an FMDP with the normal factored structure, for $i, j \in [m]$, if $i \in \idxTrans{j}$, then we say the $j$th state component \emph{depends} on the $i$th state component, or the $i$th state component \emph{influences} the $j$th state component, denoted by $\gS_i \to \gS_j$.
The $i$th state component has the \emph{loop} property, if and only if the $i$th state component influences itself either directly ($\gS_i \to \gS_i$, self-loop) or through some intermediate state components ($\gS_i \to \cdots \to \gS_j \to \cdots \to \gS_i$ for some $j\text{'s} \in [m]$).
The loop is referred to as an \emph{influence loop}.
Let $\gI$ be the set of the indices of the state components that have the loop property and some action dependence, i.e.,
\begin{align*}
    \gI := \left\{ i \in [m]: \text{$i$ has the loop property, and there exists } j\in \{m+1, \cdots, n\} \text{ such that } j \in \idxTrans{i} \right\}.
\end{align*}
Then we have the following theorem that subsumes Theorem~\ref{thm:fmdp-lb-nondegen}.

\begin{theorem}  \label{thm:fmdp-lb-loop}
    For any algorithm, under the assumption of known rewards, for the normal factored structure that satisfies $S_i\ge 3, \forall i\in [m]$ and $H\ge 2$, there is an FMDP with the specified structure such that for some initial states, the expected regret in $K$ episodes is at least $\Omega(\max_{i\in \gI}\sqrt{HX[\idxTrans{i}]T})$.
\end{theorem}

Theorem~\ref{thm:fmdp-lb-loop} shows that the factored structures that satisfy $\max_{i \in \gI} X[\idxTrans{i}] = \max_{i \in [m]} X[\idxTrans{i}]$, F-EULER is minimax optimal up to log factors. We defer to Appendix~\ref{sec:appx-lb} the proofs of Theorems~\ref{thm:fmdp-lb-degen-1},~\ref{thm:fmdp-lb-degen-2},~\ref{thm:fmdp-lb-nondegen} and~\ref{thm:fmdp-lb-loop}, relying on the lower bounds for MABs~\citep{lattimore2018bandit} and the construction of MAB-like FMDPs (FMDPs whose cumulative rewards depend only on the action in the first step), similar to the MAB-like MDPs in~\citep{dann2015sample, jin2020lecture}.

\section{Conclusion}

In this paper, we study reinforcement learning within the tabular episodic setting of FMDPs, that is, MDPs which enjoy a factored structure in their dynamics. Such a factorization is typically derived via conditional independence structures in the transition, and by using it effectively one can greatly reduce the complexity of learning. However, a straightforward adaptation of the usual minimax optimal regret analysis for MDPs turns out not to be possible.  We uncover the difficulties posed by such an adaptation, and subsequently develop two new algorithms (called F-UCBVI and F-EULER, both motivated by their known nonfactored analogues) that carefully exploit the factored structure to obtain minimax optimal regret bounds for a rich class of factored structures. The key algorithmic technique that we develop is a careful design of a ``cross-component'' bonus term to ensure optimism and guide exploration. 

We present the lower bounds for FMDPs under certain structure constraints. The problem of characterizing lower bounds for FMDPs with arbitrary structures turns out to be more subtle, and remains an open problem. In addition, our methods are model-based; developing model-free algorithms for FMDPs is worth exploring. 
Lastly, an important practical direction is to develop structure-agnostic algorithms that can still ensure regret guarantees.

\subsubsection*{Acknowledgments}
    YT acknowledges partial support from an MIT Presidential Fellowship and a graduate research assistantship from the NSF BIGDATA grant (number 1741341). SS acknowledges partial support from NSF-BIGDATA (1741341) and NSF-TRIPODS+X (1839258).

\bibliographystyle{plainnat}
\bibliography{fmdp}

\newpage
\appendix

\section{Regret analysis of F-UCBVI}
\label{sec:appx-f-ucbvi}

\subsection{Failure event}
\label{subsec:fail-evnt}

Before we define the failure event that causes our regret guarantee to fail, we introduce some additional notation.
As a shorthand, for any natural number $n$, any factored set $\gX = \bigotimes_{i=1}^{n} \gX_i$ and any given index set $I \subset [n]$, let $\gX[-I] := \bigotimes_{i=1, i\notin I}^{n} \gX_{i}$ and $x[-I] \in \gX[-I]$ be a tuple of $x[j]$ for $j\in [n]$ and $j\notin I$.
For a singleton $\{i\}$, let $\gS_{-i} \equiv \gS[-\{i\}] := (\bigotimes_{j=1}^{i-1} \gS_{j}) \times (\bigotimes_{j=i+1}^{m} \gS_{j})$. Let $s[-i] \in \gS_{-i}$ be a tuple of $s[j]$ for $j \in [m]$ and $j\neq i$.
For vector $V \in \sR^{\gS}$, 
\begin{align*}
    V(s[-i]) &:= V((s[1], \cdots, s[i-1], \cdot, s[i+1], \cdots, s[m])) \in \sR^{\gS_i}, \\
    V(s[i]) &:= V((\cdot, \cdots, \cdot, s[i], \cdot, \cdots, \cdot)) \in \sR^{\gS_{-i}}.
\end{align*}
Recall that $w_{k, h}(x)$ is the visit probability to $x$ at step $h$ of episode $k$.
Overloading the notation, we make the following definitions.
\begin{definition}[Visit probabilities]  \label{def:visit-prob}
    Define 
    \begin{align*}
        w_{i, k, h}(x) &:= w_{i, k, h}(x[\idxTrans{i}]) = \nlsum_{x[-\idxTrans{i}] \in \gX[-\idxTrans{i}]} w_{k, h}(x), \\
        v_{i, k, h}(x) &:= v_{i, k, h}(x[\idxR{i}]) = \nlsum_{x[-\idxR{i}] \in \gX[-\idxR{i}]} w_{k, h}(x).
    \end{align*}
    Then let $w_{k}(x) := \sum_{h=1}^{H} w_{k, h}(x)$, $w_{i, k}(x) := \sum_{h=1}^{H} w_{i, k, h}(x)$ and $v_{i, k}(x) := \sum_{h=1}^{H} v_{i, k, h}(x)$.
\end{definition}
Recall that $L = \log (16mlSXT / \delta)$. Then we define the failure event below. 
\begin{definition}[Failure event]  \label{def:fail-evnt-h}
    Define the events 
    \begin{align*}
        \gF_1 &:= \Biggl\{\exists (i\in [m], k \in [K], h\in [H], x\in \gX), \quad \left| \innerp{\eTrans_{i, k}(x) - \Trans_{i}(x)}{\E_{\Trans_{-i}(x)} [\oV{h+1}]} \right| > H\sqrt{\frac{L}{2\cntTrans{i, k}(x)}}  \Biggr\}, \\
        \gF_2 &:= \left\{ \exists (i\in [m], k \in [K], x\in \gX), \quad \left\| \eTrans_{i, k}(x) - \Trans_{i}(x) \right\|_{1} > \sqrt{\frac{2S_i L}{\cntTrans{i, k}(x)}} \right\}, \\
        \gF_3 &:= \left\{ \exists (i\in [l], k \in [K], x\in \gX), \quad \left|  \eR_{i, k}(x) - R_i(x) \right| > \sqrt{\frac{ L}{2\cntR{i, k}(x)}} \right\}, \\
        \gF_4 &:= \left\{ \exists (i\in [m], k \in [K], x\in \gX), \quad \cntTrans{i, k}(x) < \frac{1}{2} \sum_{\kappa < k} w_{i, \kappa}(x) - H L  \right\}, \\
        \gF_5 &:= \left\{ \exists (i\in [l], k \in [K], x\in \gX), \quad \cntR{i, k}(x) < \frac{1}{2} \sum_{\kappa < k} v_{i, \kappa}(x) - H L  \right\}, \\
        \gF_6 &:= \Biggl\{ \exists (i\in [m], k\in [K], x\in \gX, s'\in \gS), \quad \left|\eTrans_{i, k}(s'[i]\vert x) - \Trans_{i}(s'[i]\vert x)\right| > \frac{2L}{3\cntTrans{i, k}(x)} + \sqrt{\frac{2\Trans_{i}(s'[i]\vert x)L}{\cntTrans{i, k}(x)}} \Biggr\}.
    \end{align*}
    Then the failure event for F-UCBVI is defined by $\gF := \bigcup_{i=1}^{6} \gF_{i}$.
\end{definition}

The following lemma shows that the failure event $\gF$ happens with low probability.

\begin{lemma}[Failure probability]  \label{lem:fail-prob-h}
    For any FMDP specified by~\eqref{eqn:fmdp}, during the running of F-UCBVI for $K$ episodes, the failure event $\gF$ happens with probability at most $\delta$.
\end{lemma}

\begin{proof}
    By Hoeffding's inequality (Lemma~\ref{lem:concen-h}) and the union bound, $\gF_1$ happens with probability at most $\delta/8$.
    The same argument applies to $\gF_3$.
    By $\normlone$-norm concentration (Lemma~\ref{lem:concen-l1}) and the union bound, $\gF_{2}$ happens with probability at most $\delta / 8$.
    By the same argument regarding failure event $\gF_N$ in~\citep[Lemma 6, Section B.1]{dann2018policy}, $\gF_4$ and $\gF_5$ happen with probability at most $\delta/16$ respectively.
    By the same Bernstein's inequality argument in~\citep[Lemma 1, Section B.4]{azar2017minimax}, $\gF_6$ happens with probability at most $\delta / 8$.
    Finally, applying the union bound on $\gF_i$ for $i \in [6]$ yields that the failure event $\gF$ happens with probability at most $5\delta/8 \le \delta$.
\end{proof}

The deduction in the rest of this section and hence the regret bound hold outside the failure event $\gF$, with probability at least $1-\delta$.
From the above derivation, note that we can actually use a smaller 
\begin{align*}
    L_0 = \log (10ml\max\{ \max_{i} (S_i X[\idxTrans{i}]), \max_{i} X[\idxR{i}] \} T / \delta)
\end{align*}
to replace $L$ for F-UCBVI. We use $L$ for simplicity.

\subsection{Upper confidence bound}

The transition estimation error refers to a term incurred by the difference between the estimated transition and the true one. 
To apply scalar concentration, we use the standard technique that bounds the inner product of their difference and the optimal value function. 
Specifically, we have the following lemma. 

\begin{lemma}[Transition estimation error, Hoeffding-style] \label{lem:est-err-h}
    Outside the failure event $\gF$, for any episode $k\in [K]$, step $h\in [H]$ and state-action pair $x\in \gX$, the transition estimation error satisfies that
    \begin{align}  \label{eqn:est-err-h}
        \left| \innerp{\eTrans_{k}(x) - \Trans(x)}{\oV{h+1}} \right|
        \le \sum_{i=1}^{m} H\sqrt{\frac{L}{2\cntTrans{i, k}(x)}} + \sum_{i=1}^{m} \sum_{j=i+1}^{m} 2HL \sqrt{\frac{S_i S_j}{\cntTrans{i, k}(x) \cntTrans{j, k}(x)}}.
    \end{align}
\end{lemma}

\begin{proof}
    Omitting the dependence of $\eTrans_{k}(x), \Trans(x), \eTrans_{i, k}(x), \Trans_{i}(x)$ on $x$, 
    \begin{align}
        \innerp{\eTrans_{k} - \Trans}{\oV{h+1}}
        &= \innerp{\prod_{i=1}^{m} \eTrans_{i, k} - \prod_{i=1}^{m} \Trans_{i}}{\oV{h+1}} \nonumber \\
        &= \innerp{\sum_{i=1}^{m} (\eTrans_{i, k} - \Trans_{i}) \Trans_{1:i-1} \eTrans_{i+1:m, k}}{\oV{h+1}} \nonumber \\
        &= \sum_{i=1}^{m} \innerp{\eTrans_{i, k} - \Trans_{i}}{\E_{\Trans_{1:i-1} \eTrans_{i+1:m, k}} [\oV{h+1}]} \nonumber \\
        &= \sum_{i=1}^{m} \innerp{\eTrans_{i, k} - \Trans_{i}}{\E_{\Trans_{1:i-1}} \E_{\eTrans_{i+1:m, k}} [\oV{h+1}]} \nonumber \\
        &= \sum_{i=1}^{m} \innerp{\eTrans_{i, k} - \Trans_{i}}{\E_{\Trans_{1:i-1}} \E_{\Trans_{i+1:m}} [\oV{h+1}]}  \label{eqn:est-err-h-decomp-1} \\
        &\quad + \sum_{i=1}^{m} \innerp{\eTrans_{i, k} - \Trans_{i}}{\E_{\Trans_{1:i-1}} (\E_{\eTrans_{i+1:m, k}} - \E_{\Trans_{i+1:m}}) [\oV{h+1}]},  \label{eqn:est-err-h-decomp-2}
    \end{align}
    where the second equality adopts an inverse telescoping technique (add and subtract a sequence of terms), essential to our analysis.
    Outside the failure event $\gF$ (specifically, $\gF_1$),~\eqref{eqn:est-err-h-decomp-1} is upper bounded by 
    \begin{align}  \label{eqn:est-err-h-c}  
        \left| \innerp{\eTrans_{i, k}(x) - \Trans_{i}(x)}{\E_{\Trans_{1:i-1}(x)} \E_{\Trans_{i+1:m}(x)} [\oV{h+1}]} \right| 
        \le H\sqrt{\frac{L}{2\cntTrans{i, k}(x)}}.
    \end{align}
    By Lemma~\ref{lem:holder}, outside the failure event $\gF$,~\eqref{eqn:est-err-h-decomp-2} is upper bounded by
    \begin{equation}  \label{eqn:est-err-h-xc}  
        \begin{aligned}  
            \left| \innerp{\eTrans_{i, k}(x) - \Trans_{i}(x)}{\E_{\Trans_{1:i-1}(x)} (\E_{\eTrans_{i+1:m, k}(x)} - \E_{\Trans_{i+1:m}(x)}) [\oV{h+1}]} \right|
            \le \sum_{j=i+1}^{m} 2HL \sqrt{\frac{S_i S_j}{\cntTrans{i, k}(x) \cntTrans{j, k}(x)}}.
        \end{aligned}    
    \end{equation}
    Combining~\eqref{eqn:est-err-h-c} and~\eqref{eqn:est-err-h-xc} yields the transition estimation error bound~\eqref{eqn:est-err-h}.
\end{proof}

The following Lemma~\ref{lem:holder} brings in the cross-component term, also as part of the transition bonus later, which results from applying inverse telescoping once and $\normlone$-norm concentration (lemma~\ref{lem:concen-l1}) twice.

\begin{lemma}[Holder's argument]  \label{lem:holder}
    Outside the failure event $\gF$, for any index $i\in [m]$, episode $k\in [K]$, step $h\in [H]$ and state-action pair $x\in \gX$, 
    \begin{align*}
        \left| \innerp{\eTrans_{i, k}(x) - \Trans_{i}(x)}{\E_{\Trans_{1:i-1}(x)} (\E_{\eTrans_{i+1:m, k}(x)} - \E_{\Trans_{i+1:m}(x)}) [V_{h+1}^{*}]} \right|
        \le \sum_{j=i+1}^{m} 2HL \sqrt{\frac{S_i S_j}{\cntTrans{i, k}(x) \cntTrans{j, k}(x)}}.
    \end{align*}
\end{lemma}

\begin{proof}
    By Holder's inequality, 
    \begin{align}
        &\left| \innerp{\eTrans_{i, k}(x) - \Trans_{i}(x)}{\E_{\Trans_{1:i-1}(x)} (\E_{\eTrans_{i+1:m, k}(x)} - \E_{\Trans_{i+1:m}(x)}) [V_{h+1}^{*}]} \right| \nonumber \\
        \le & \left\| \eTrans_{i, k}(x) - \Trans_{i}(x) \right\|_{1} \cdot \left\| \E_{\Trans_{1:i-1}(x)} (\E_{\eTrans_{i+1:m, k}(x)} - \E_{\Trans_{i+1:m}(x)}) [V_{h+1}^{*}] \right\|_{\infty} \nonumber \\
        \le & \sqrt{\frac{2S_i L}{\cntTrans{i, k}(x)}} \cdot \left\| \E_{\Trans_{1:i-1}(x)} (\E_{\eTrans_{i+1:m, k}(x)} - \E_{\Trans_{i+1:m, h}(x)}) [V_{h+1}^{*}] \right\|_{\infty},  \label{eqn:holder-infnorm}
    \end{align}
    where the second inequality holds outside the failure event $\gF$ (specifically, $\gF_2$). 
    We proceed to bound the $\normmax$-norm term by applying the inverse telescoping technique. 
    Omitting the dependence of $\eTrans_{k}(x), \Trans(x), \eTrans_{i, k}(x), \Trans_{i}(x)$ on $x$, for any $i\in [m]$,  
    \begin{align*}
        \E_{\Trans_{1:i-1}} (\E_{\eTrans_{i+1:m, k}} - \E_{\Trans_{i+1:m}}) [V_{h+1}^{*}]
        & = \innerpd{\prod_{j=i+1}^{m} \eTrans_{j, k} - \prod_{j=i+1}^{m} \Trans_{j} }{\E_{\Trans_{1:i-1}}[V_{h+1}^{*}]} \\
        & = \innerpd{\sum_{j=i+1}^{m} (\eTrans_{j, k} - \Trans_{j}) \Trans_{i+1:j-1} \eTrans_{j+1:m, k}}{\E_{\Trans_{1:i-1}}[V_{h+1}^{*}]} \\
        & = \sum_{j=i+1}^{m} \innerp{\eTrans_{j, k} - \Trans_{j}}{\E_{\Trans_{1:i-1}} \E_{\Trans_{i+1:j-1}} \E_{\eTrans_{j+1:m, k}}[V_{h+1}^{*}]} \in \sR^{\gS_i}.
    \end{align*}
    Therefore, for any $i\in [m]$ and $s'[i] \in \gS_i$, 
    \begin{align*}
        \left| \E_{\Trans_{1:i-1}} (\E_{\eTrans_{i+1:m, k}} - \E_{\Trans_{i+1:m}}) [V_{h+1}^{*}] (s'[i]) \right|
        & = \left| \E_{\Trans_{1:i-1}} (\E_{\eTrans_{i+1:m, k}} - \E_{\Trans_{i+1:m}}) [V_{h+1}^{*}(s'[i])] \right| \\
        & \le \sum_{j=i+1}^{m} \left| \innerp{\eTrans_{j, k} - \Trans_{j}}{\E_{\Trans_{1:i-1}} \E_{\Trans_{i+1:j-1}} \E_{\eTrans_{j+1:m, k}}[V_{h+1}^{*}(s'[i])]} \right| \\
        & \le \sum_{j=i+1}^{m} \left\| \eTrans_{j, k} - \Trans_{j} \right\|_{1} \cdot \left\| \E_{\Trans_{i+1:j-1}} \E_{\eTrans_{j+1:m, k}}[V_{h+1}^{*}(s'[i])] \right\|_{\infty} \\
        & \le \sum_{j=i+1}^{m} \sqrt{\frac{2S_j L}{\cntTrans{j, k}}} \cdot H,
    \end{align*}
    where the last inequality holds outside the failure event $\gF$ (specifically, $\gF_2$).
    Substituting the above into~\eqref{eqn:holder-infnorm} yields 
    \begin{align*}
        \left| \innerp{\eTrans_{i, k}(x) - \Trans_{i}(x)}{\E_{\Trans_{1:i-1}(x)} (\E_{\eTrans_{i+1:m, k}(x)} - \E_{\Trans_{i+1:m}(x)}) [V_{h+1}^{*}]} \right|
        & \le \sum_{j=i+1}^{m} \sqrt{\frac{2S_i L}{\cntTrans{i, k}(x)}} \cdot H \sqrt{\frac{2 S_j L}{\cntTrans{j, k}(x)}} \\
        & = \sum_{j=i+1}^{m} 2HL \sqrt{\frac{S_i S_j}{\cntTrans{i, k}(x) \cntTrans{j, k}(x)}}.
    \end{align*}
\end{proof}

Recall that our Hoeffding-style transition bonus~\eqref{eqn:trans-bonus-f-ucbvi} is exactly the transition estimation error bound in~\eqref{eqn:est-err-h}.
Add a subscript $k$ to $\eR, \eR_{i}$ to denote the corresponding quantities in the $k$th episode.
Recall that our choice of the reward bonus upper bounds the reward estimation error outside the failure event $\gF$ (specifically, $\gF_{3}$), i.e., 
\begin{align*}
    \left| \eR_{k}(x) - R(x) \right| \le \sum_{i=1}^{l} \left| \eR_{i, k}(x) - R_i(x) \right| \le \sum_{i=1}^{l} \sqrt{\frac{ L}{2\cntR{i, k}(x)}} := \bR_{k}(x).
\end{align*}
The following lemma shows that these choices ensure the optimism. 
Specifically, $\uV_{k, h}$ is an entrywise UCB of $\oV{h}$ for all $k\in [K], h\in [H]$.

\begin{lemma}[Upper confidence bound]  \label{lem:ucb-h}
    Outside the failure event $\gF$, for the choices of bonuses in~\eqref{eqn:trans-bonus-f-ucbvi} and~\eqref{eqn:r-bonus-f-ucbvi}, $\oV{h}(s) \le \uV_{k, h}(s)$ for any episode $k\in [K]$, step $h\in [H]$ and state $s\in \gS$.
\end{lemma}

\begin{proof}
    For $h = H + 1$, $\oV{H+1}(s) = V_{k, H+1}(s) = 0$ for all $k\in [K]$ and $s\in \gS$.
    We proceed by backward induction.
    For all $k \in [K]$, for a given $h\in [H]$, for all $s\in \gS$, with $x_{k, h} = (s, \pi_k(s, h))$ and $x_{h}^{*} = (s, \pi^{*}(s, h))$, 
    \begin{align*}
        &\uV_{k, h}(s) -\oV{h}(s) \\
        = & \eR(x_{k, h}) + \bR_{k}(x_{k,h})  + \innerp{\eTrans_{k}(x_{k, h})}{\uV_{k, h+1}} + \bTrans_{k}(x_{k, h}) - R(x_{h}^{*}) - \innerp{\Trans(x_{h}^{*})}{\oV{h+1}} \\
        \ge & \eR(x_{h}^{*})+ \bR_{k}(x^*_{h}) + \innerp{\eTrans_{k}(x_{h}^{*})}{\uV_{k, h+1}} + \bTrans_{k}(x_{h}^{*}) - R(x_{h}^{*}) - \innerp{\Trans(x_{h}^{*})}{\oV{h+1}} \\
        \ge & \innerp{\eTrans_{k}(x_{h}^{*})}{\uV_{k, h+1} - \oV{h+1}} + \eR(x_{h}^{*}) - R(x_{h}^{*}) + \bR_{k}(x^*_{h}) 
        + \innerp{\eTrans_{k}(x_{h}^{*}) - \Trans(x_{h}^{*})}{\oV{h+1}} 
        + \bTrans_{k}(x_{h}^{*}), 
    \end{align*}
    where the first equality corresponds to the nontrivial case where $\uV_{k, h}(s) < H - h + 1$.
    Since 
    \begin{align*}
        \innerp{\eTrans_{k}(x_{h}^{*})}{\uV_{k, h+1} - \oV{h+1}} \ge 0   
    \end{align*}
    by the inductive assumption, we have $\uV_{k, h}(s) -\oV{h}(s) \ge 0$ outside the failure event $\gF$.
    Therefore, $\oV{h}(s) \le \uV_{k, h}(s)$ for all $k\in [K], h\in [H], s\in \gS$.
\end{proof}

Refer to the difference between the optimistic value function than the optimal value function as the \emph{confidence radius}.
We now bound the confidence radius in the following lemma.
After the introduction of ``good'' sets, we then bound the sum over time of the squared confidence radius, which is useful to prove that the cumulative correction term is lower-order (polylog in $T$, Lemma~\ref{lem:cul-corr-h}).

\begin{lemma}[Confidence radius, Hoeffding-style]  \label{lem:conf-r-h}
    Let $F_0 := 5 mH\max_{i} S_i L$ be a lower-order term.
    Let $s_{k, t} \in \gS$ denote the state at step $t$ of episode $k$ and $x_{k, t} = (s_{k, t}, \pi_{k}(s_{k, t}, t))$.
    Outside the failure event $\gF$, for any episode $k\in [K]$, step $h\in [H]$ and state $s\in \gS$, the confidence radius of F-UCBVI satisfies that 
    \begin{align*}
        \uV_{k, h}(s) - \oV{h}(s)  \le \min\left\{  \sum_{t=h}^{H} \E_{\pi_k}\left[ \sum_{i=1}^{m} \frac{F_0}{\sqrt{\cntTrans{i, k}(x_{k, t})}} + \sum_{i=1}^{l} \sqrt{\frac{2L}{\cntR{i, k}(x_{k, t})}}  \middle| s_{k, h} = s \right], H \right\}.
    \end{align*}
\end{lemma}

\begin{proof}
    By definition, for any $k\in [K], h\in [H], s\in \gS$, 
    \begin{align}
        \uV_{k, h}(s) - \oV{h}(s)
        & \le \eR(x_{k, h}) + \bR_{k}(x_{k,h})+ \innerp{\eTrans_{k}(x_{k, h})}{\uV_{k, h+1}} + \bTrans_{k}(x_{k, h}) - R(x_{h}^{*}) - \innerp{\Trans(x_{h}^{*})}{\oV{h+1}} \nonumber \\
        & \le \eR(x_{k, h}) + \bR_{k}(x_{k,h})+ \innerp{\eTrans_{k}(x_{k, h})}{\uV_{k, h+1}} + \bTrans_{k}(x_{k, h}) - R(x_{k, h}) - \innerp{\Trans(x_{k, h})}{\oV{h+1}} \nonumber \\
        & \le 2\bR_{k}(x_{k,h}) +\innerp{\eTrans_{k}(x_{k, h}) - \Trans(x_{k, h})}{\uV_{k, h+1}} + \innerp{\Trans(x_{k, h})}{\uV_{k, h+1} - \oV{h+1}} + \bTrans_{k}(x_{k, h}) \nonumber \\
        & \le \innerp{\Trans(x_{k, h})}{\uV_{k, h+1} - \oV{h+1}} 
        + \sum_{i=1}^{m} H \sqrt{\frac{2S_i L}{\cntTrans{i, k}(x)}} 
        + \bTrans_{k}(x_{k, h}) + 2 \bR_{k}(x_{k,h}).  \label{eqn:conf-r-decomp}
    \end{align}
    For $b_{k}(x_{k, h})$, we have 
    \begin{align*}
        b_{k}(x_{k, h}) 
        &= \sum_{i=1}^{m} H\sqrt{\frac{L}{2\cntTrans{i, k}(x_{k, h})}} + \sum_{i=1}^{m} \sum_{j=i+1}^{m} 2HL \sqrt{\frac{S_i S_j}{\cntTrans{i, k}(x_{k, h}) \cntTrans{j, k}(x_{k, h})}} \\
        &\le \sum_{i=1}^{m} H\sqrt{\frac{L}{2\cntTrans{i, k}(x_{k, h})}} + \sum_{i=1}^{m} 2m\max_{i}S_i HL \frac{1}{\sqrt{\cntTrans{i, k}(x_{k, h})}} \\
        &\le 3m\max_{i}S_i HL \sum_{i=1}^{m} \frac{1}{\sqrt{\cntTrans{i, k}(x_{k, h})}}.
    \end{align*}
    Substituting the above into~\eqref{eqn:conf-r-decomp} yields
    \begin{align*}
        \uV_{k, h}(s) - \oV{h}(s)
        & \le \innerp{\Trans(x_{k, h})}{\uV_{k, h+1} - \oV{h+1}}  
        + 5 m H \max_{i} S_i L \sum_{i=1}^{m} \frac{1}{\sqrt{\cntTrans{i, k}(x)}} +\sum_{i=1}^{l} \sqrt{\frac{2L}{\cntR{i, k}(x_{k, h})}} \\
        & = \innerp{\Trans(x_{k, h})}{\uV_{k, h+1} - \oV{h+1}}  
        + \sum_{i=1}^{m} \frac{F_0}{\sqrt{\cntTrans{i, k}(x)}} +\sum_{i=1}^{l} \sqrt{\frac{2L}{\cntR{i, k}(x_{k, h})}}.
    \end{align*}
    By backward induction over the subscript $h$, we have 
    \begin{align*}
        \uV_{k, h}(s) - \oV{h}(s) \le \min\left\{  \sum_{t=h}^{H} \E_{\pi_k}\left[ \sum_{i=1}^{m} \frac{F_0}{\sqrt{\cntTrans{i, k}(x_{k, t})}} + \sum_{i=1}^{l}\sqrt{\frac{2L}{\cntR{i, k}(x_{k, t})}}   \middle| s_{k, h} = s \right], H \right\}.
    \end{align*}
\end{proof}

\subsection{Good sets}
\label{subsec:good-sets}

The following ``good sets''~\citep{zanette2019tighter} are a notion of sufficient visits so that the estimations are meaningful, the introduction of which is seminal to the sum-over-time analysis.

\begin{definition}[Good sets]  \label{def:good-sets}
    Define the good sets of state-action components for transition and reward estimations as 
    \begin{align*}
        L_{i, k} &:= \left\{ x[\idxTrans{i}]\in \gX[\idxTrans{i}]: \frac{1}{4} \sum_{\kappa < k} w_{i, \kappa}(x[\idxTrans{i}]) \ge HL + H \right\}, \\
        \Lambda_{i, k} &:= \left\{ x[\idxR{i}]\in \gX[\idxR{i}]: \frac{1}{4} \sum_{\kappa < k} v_{i, \kappa}(x[\idxR{i}]) \ge HL + H \right\}.
    \end{align*}
    Then the corresponding good sets of state-action pairs are defined by 
    \begin{align*}
        L_k &:= \left\{ x\in \gX: x[\idxTrans{i}] \in L_{i, k} \textforall i \in [m] \right\}, \\
        \Lambda_k &:= \left\{ x\in \gX: x[\idxR{i}] \in \Lambda_{i, k} \textforall i \in [l] \right\}.
    \end{align*}
\end{definition}

We shall restrict our attention to the state-action pairs in the good sets (with sufficient visits).
To this end, we show the sums of visit probabilities to the state-action pairs outside the good sets are lower-order terms in the following lemma.

\begin{lemma}[Sum out of good sets]  \label{lem:sum-out-of-good-set}
    The sums of the visit probabilities of the state-action pairs out of the good sets and over time satisfy that 
    \begin{align*}
        \sum_{k=1}^{K} \sum_{h=1}^{H} \sum_{x\notin L_k} w_{k, h}(x) 
        \le 8 \sum_{i=1}^{m} X[\idxTrans{i}] H L.\\
        \sum_{k=1}^{K} \sum_{h=1}^{H} \sum_{x\notin \Lambda_k} v_{k, h}(x) 
        \le 8 \sum_{i=1}^{l} X[\idxR{i}] H L.\\
    \end{align*} 
\end{lemma}

\begin{proof}
    If $x[\idxTrans{i}] \notin L_{i, k}$, then by definition, 
    \begin{align*}
        \frac{1}{4} \sum_{\kappa \le k} w_{i, \kappa} (x[\idxTrans{i}]) < HL + H + H = H(L + 2).
    \end{align*}
    Therefore, 
    \begin{align*}
        \sum_{k=1}^{K} \sum_{h=1}^{H} \sum_{x\notin L_k} w_{k, h}(x) 
        &\le \sum_{k=1}^{K} \sum_{h=1}^{H} \sum_{x\in \gX} w_{k, h}(x) \sI(x \notin L_k) \\
        &\le \sum_{k=1}^{K} \sum_{h=1}^{H} \sum_{x\in \gX} w_{k, h}(x) \sum_{i=1}^{m} \sI(x[\idxTrans{i}] \notin L_{i, k}) \\
        &\le \sum_{i=1}^{m} \sum_{x[\idxTrans{i}]\in \gX[\idxTrans{i}]} \sum_{k=1}^{K} \sum_{h=1}^{H} w_{i, k, h}(x[\idxTrans{i}]) \sI(x[\idxTrans{i}] \notin L_{i, k}) \\
        &= \sum_{i=1}^{m} \sum_{x[\idxTrans{i}]\in \gX[\idxTrans{i}]} \sum_{k=1}^{K} w_{i, k}(x[\idxTrans{i}]) \sI(x[\idxTrans{i}] \notin L_{i, k}) \\
        &\le 4\sum_{i=1}^{m} X[\idxTrans{i}] H (L+2),
    \end{align*}
    where in the third inequality we write $\sum_{x\in \gX}$ as $\sum_{x[\idxTrans{i}] \in \gX[\idxTrans{i}]} \sum_{x[-\idxTrans{i}] \in \gX[-\idxTrans{i}]}$ and use the definition of $w_{i, k, h}$ (Definition~\ref{def:visit-prob}).
    Since $L = \log (16mlSXT / \delta) \ge 2$, we have 
    \begin{align*}
        \sum_{k=1}^{K} \sum_{h=1}^{H} \sum_{x\notin L_k} w_{k, h}(x) \le 8\sum_{i=1}^{m} X[\idxTrans{i}] H L.    
    \end{align*}
    The same argument applies to $v_{k, h}(x)$.
\end{proof}

The following lemma bridges the visit probabilities $w_{i, k}$ and $v_{i, k}$ to the actual numbers of visits $\cntTrans{i, k}$ and $\cntR{i, k}$ for the state-action pairs in the good sets.

\begin{lemma}[Visit number and visit probability]  \label{lem:vn-vp}
    Outside the failure event $\gF$, the numbers of visits $\cntTrans{i, k}$ and $\cntR{i, k}$ to the state-action pairs in the good sets satisfy that 
    \begin{align*}
        \cntTrans{i, k}(x) &\ge \frac{1}{4} \sum_{\kappa\le k} w_{i, \kappa}(x)\quad \textforall i \in [m] \textand x \in L_{k}, \\
        \cntR{i, k}(x) &\ge \frac{1}{4} \sum_{\kappa\le k} w^r_{i, \kappa}(x)\quad \textforall i \in [l] \textand x\in \Lambda_k.
    \end{align*}
\end{lemma}

\begin{proof}
    Outside the failure event $\gF$ (specifically, $\gF_4$), for all $i\in [m]$, 
    \begin{align*}
        \cntTrans{i, k}(x) 
        &\ge \frac{1}{2} \sum_{\kappa < k} w_{i, \kappa}(x) - HL \\
        &= \frac{1}{4} \sum_{\kappa < k} w_{i, \kappa}(x) + \frac{1}{4} \sum_{\kappa < k} w_{i, \kappa}(x) - HL \\
        &\ge \frac{1}{4} \sum_{\kappa < k} w_{i, \kappa}(x) + H \\
        &\ge \frac{1}{4} \sum_{\kappa \le k} w_{i, \kappa}(x),
    \end{align*}
    where the second inequality results from the definition of good sets (Definition~\ref{def:good-sets}).
    Outside the failure event $\gF$ (specifically, $\gF_5$), the same argument applies to $\cntR{i, k}(x)$ for all $i\in [l]$. 
\end{proof}

By Lemma~\ref{lem:vn-vp} and the definition of the good sets (Definition~\ref{def:good-sets}), for all $k\in [K]$, $\cntTrans{i, k}(x) \ge HL + H \ge 2$ for all $x\in L_k$, and $\cntR{i, k}(x) \ge HL + H \ge 2$ for all $x \in \Lambda_k$.
Therefore, the regret analysis out of the good sets automatically precludes the cases of zero denominators in Algorithms~\ref{alg:f-ovi} and~\ref{alg:vi-optimism}, where we replace the zeros by ones for algorithmic completeness.

Refer to the ratio of visit probability $w_{k, h}$ to visit number $\cntTrans{i, k}$ or $\cntR{i, k}$ as the \emph{visit ratio}.
Then the accumulation of the visit ratios turns out to be a lower-order term, as shown in the following lemma.

\begin{lemma}[Sum of visit ratio in good sets]  \label{lem:sum-vr}
    Outside the failure event $\gF$, the sums of the visit ratios of the state-action pairs within the good sets and over time satisfy that 
    \begin{align*}
        \sum_{k=1}^{K} \sum_{h=1}^{H} \sum_{x\in L_k} \frac{w_{k, h}(x)}{\cntTrans{i, k}(x)} 
        = \sum_{k=1}^{K} \sum_{x\in L_k} \frac{w_{k}(x)}{\cntTrans{i, k}(x[\idxTrans{i}])} 
        \le 4 X[\idxTrans{i}]L \quad \textforall i \in [m],  \\
        \sum_{k=1}^{K} \sum_{h=1}^{H} \sum_{x\in \Lambda_k} \frac{w_{k, h}(x)}{\cntR{i, k}(x)} 
        = \sum_{k=1}^{K} \sum_{x\in \Lambda_k} \frac{w_{k}(x)}{\cntR{i, k}(x[\idxR{i}])}
        \le 4 X[\idxR{i}]L \quad \textforall i \in [l].
    \end{align*}
\end{lemma}

\begin{proof}
    Outside the failure event $\gF$, for any $i \in [m]$, by Lemma~\ref{lem:vn-vp}, 
    \begin{align*}
        \sum_{k=1}^{K} \sum_{x\in L_k} \frac{w_{k}(x)}{\cntTrans{i, k}(x[\idxTrans{i}])} 
        &\le \sum_{k=1}^{K} \sum_{x\in \gX} \frac{w_{k}(x)}{\cntTrans{i, k}(x[\idxTrans{i}])} \sI(x[\idxTrans{i}] \in L_{i, k}) \\
        &\le \sum_{k=1}^{K} \sum_{x[\idxTrans{i}]\in \gX[\idxTrans{i}]} \frac{w_{i, k}(x[\idxTrans{i}])}{\cntTrans{i, k}(x[\idxTrans{i}])} \sI(x[\idxTrans{i}] \in L_{i, k}) \\
        &\le 4 \sum_{k=1}^{K} \sum_{x[\idxTrans{i}]\in \gX[\idxTrans{i}]} \frac{w_{i, k}(x[\idxTrans{i}])}{\sum_{\kappa \le k} w_{i, \kappa}(x[\idxTrans{i}])} \sI(x[\idxTrans{i}] \in L_{i, k}) \\
        &\le 4 X[\idxTrans{i}] L,
    \end{align*}
    where the last inequality is shown by the proof of Lemma 13 in~\citep{zanette2019tighter}.
    The same argument applies to the visit ratio $w_{k, h}(x) / \cntR{i, k}(x)$ for any $i\in [l]$.
\end{proof}

As to be shown below, the cross-component transition bonus term brings in the mixed visit ratio $w_{k, h}(x) / \sqrt{\cntTrans{i, k}(x) \cntTrans{j, k}(x)}$ in the analysis.
By the Cauchy-Schwarz inequality, we immediately have the following control on the accumulation of the mixed visit ratios. 

\begin{lemma}[Sum of mixed visit ratio in good sets]  \label{lem:sum-mixed-vr}
    Outside the failure event $\gF$, the sum of the mixed visit ratios of the state-action pairs within the good set $L_{k}$ and over time satisfies that 
    \begin{align*}
        \sum_{k=1}^{K} \sum_{h=1}^{H} \sum_{x\in L_k} \frac{w_{k, h}(x)}{\sqrt{\cntTrans{i, k}(x) \cntTrans{j, k}(x)}} \le 4 \sqrt{X[\idxTrans{i}] X[\idxTrans{j}]} L \quad \textforall i, j \in [m].
    \end{align*}
\end{lemma}

\begin{proof}
    Outside the failure event $\gF$, by the Cauchy-Schwarz inequality and Lemma~\ref{lem:sum-vr}, for any $i, j\in [m]$, 
    \begin{align*}
        \sum_{k=1}^{K} \sum_{h=1}^{H} \sum_{x\in L_k} \frac{w_{k}(x)}{\sqrt{\cntTrans{i, k}(x) \cntTrans{j, k}(x)}} 
        &= \sum_{k=1}^{K} \sum_{x\in L_k} \frac{w_{k}(x)}{\sqrt{\cntTrans{i, k}(x[\idxTrans{i}]) \cntTrans{j, k}(x[\idxTrans{j}])}} \\
        &\le \sqrt{\sum_{k=1}^{K} \sum_{x\in L_{k}} \frac{w_{k}(x)}{\cntTrans{i, k}(x[\idxTrans{i}])}} \cdot \sqrt{\sum_{k=1}^{K} \sum_{x\in L_{k}} \frac{w_{k}(x)}{\cntTrans{j, k}(x[\idxTrans{j}])}} \\
        &\le 4 \sqrt{X[\idxTrans{i}] X[\idxTrans{j}]} L.
    \end{align*}
\end{proof}

The following lemma bounds the sum over time of the expected squared confidence radius, through the proof of which we can see an initial application of the above lemmas obtained with the notion of good sets. For clarification, by ``sum over time'', we mean the $w_{k, h}(x)$-weighted sum over $k\in [K], h\in [H]$ and $x \in L_{k}$ or $x\in \Lambda_{k}$ henceforth.

\begin{lemma}[Cumulative confidence radius, Hoeffding-style]  \label{lem:cul-conf-r-h}
    Define the lower-order term
    \begin{align*}
        G_0 := 208 m^4 H^4 (\max_{i} S_i)^2 \max_{i} X[\idxTrans{i}] L^3 + 24 l^2 H^3 \max_{i} X[\idxR{i}] L^2.  
    \end{align*}
    Then outside the failure event $\gF$, for all $i\in [m]$, the sum over time of the following expected squared confidence radius of F-UCBVI satisfies that 
    \begin{align*}
        \sum_{k=1}^{K} \sum_{h=1}^{H} \sum_{x\in \gX} w_{k, h}(x) \left( \E_{\Trans_i} \left(\E_{\Trans_{-i}}[\uV_{k, h+1} - \oV{h+1}]\right)^2 \right)
        \le \sum_{k=1}^{K} \sum_{h=1}^{H} \sum_{x\in \gX} w_{k, h}(x) \left( \E_{\Trans}[\left(\uV_{k, h+1} - \oV{h+1}\right)^2] \right)
        \le G_0.
    \end{align*}
\end{lemma}

\begin{proof}
    Let $s_{k, h} \in \gS$ denote the state at step $h$ of episode $k$. 
    Since $(\E[X])^2 \le \E[X^2]$ for any random variable $X$, we have that for all $i\in [m]$, 
    \begin{align}
        & \sum_{k=1}^{K} \sum_{h=1}^{H} \sum_{x\in \gX} w_{k, h}(x) \left( \E_{\Trans_i} \left(\E_{\Trans_{-i}}[\uV_{k, h+1} - \oV{h+1}]\right)^2 \right) \nonumber \\
        \le & \sum_{k=1}^{K} \sum_{h=1}^{H} \sum_{x\in \gX} w_{k, h}(x) \left( \E_{\Trans_i} \E_{\Trans_{-i}}[\left(\uV_{k, h+1} - \oV{h+1}\right)^2] \right) \nonumber \\
        = & \sum_{k=1}^{K} \sum_{h=1}^{H} \sum_{x\in \gX} w_{k, h}(x) \left( \sum_{s'\in \gS} \Trans(s'\vert x) \left(\uV_{k, h+1}(s') - \oV{h+1}(s')\right)^2 \right) \nonumber \\
        = & \sum_{k=1}^{K} \sum_{h=1}^{H} \E_{\pi_k} \left[\left(\uV_{k, h+1}(s_{k, h+1}) - \oV{h+1}(s_{k, h+1})\right)^2 \middle\vert s_{k, 1} \right] \nonumber \\
        \le & \sum_{k=1}^{K} \sum_{h=1}^{H} \E_{\pi_k} \left[ \left(\uV_{k, h}(s_{k, h}) - \oV{h}(s_{k, h})\right)^2 \middle\vert s_{k, 1} \right]. \label{eqn:cul-conf-r-simpl}
    \end{align}
    By the confidence radius lemma (Lemma~\ref{lem:conf-r-h}), for any $k\in [K], h\in[H]$, 
    \begin{align*}
        &\E_{\pi_k} \left[\left(\uV_{k, h}(s_{k, h}) - \oV{h}(s_{k, h})\right)^2 \middle\vert s_{k, 1} \right] \\
        \le & \E_{\pi_k} \left[ 
            \left( \sum_{t=h}^{H} \E_{\pi_k} \left[ \sum_{i=1}^{m} \frac{F_0}{\sqrt{\cntTrans{i, k}(x_{k, t})}} +  \sum_{i=1}^{l}  \sqrt{\frac{2L}{\cntR{i, k}(x_{k, t})}} \middle| s_{k, h} \right] \right)^2 \middle\vert s_{k, 1}
        \right] \\
        \le & 2 m H F_0^2 \sum_{i=1}^{m} \sum_{t=h}^{H} \E_{\pi_k} \left[ \frac{1}{\cntTrans{i, k}(x_{k, t})} \middle| s_{k, 1} \right] 
        + 4 l H L\sum_{i=1}^{l} \sum_{t=h}^{H} \E_{\pi_k} \left[ \frac{1}{\cntR{i, k}(x_{k, t})} \middle| s_{k, 1} \right] \\
        \le & 2m H^2 F_0^2 \sum_{i=1}^{m} \E_{\pi_k}\left[ \frac{1}{\cntTrans{i, k}(x_{k, h})} \middle\vert s_{k, 1} \right] 
        + 4l H^2 L\sum_{i=1}^{l} \E_{\pi_k} \left[ \frac{1}{\cntR{i, k}(x_{k, t})} \middle\vert s_{k, 1} \right],
    \end{align*}
    where in the second inequality we use the inequality 
    $\left(\sum_{i=1}^{n} a_i\right)^2  \le n \sum_{i=1}^{n} a_i^2$
    for multiple times ($n = 2, m, l, H$).
    The confidence radius lemma (Lemma~\ref{lem:conf-r-h}) also guarantees that
    \begin{align*}
        \E_{\pi_k} \left[\left(\uV_{k, h}(s_{k, h}) - \oV{h}(s_{k, h})\right)^2 \middle\vert s_{k, 1}\right] \le H^2.
    \end{align*}
    Therefore, substituting the above two bounds into~\eqref{eqn:cul-conf-r-simpl} and by Lemmas~\ref{lem:sum-out-of-good-set} and~\ref{lem:sum-vr}, we have 
    \begin{align*}
        &\sum_{k=1}^{K} \sum_{h=1}^{H} \E_{\pi_k} \left[\left(\uV_{k, h}(s_{k, h}) - \oV{h}(s_{k, h})\right)^2 \middle\vert s_{k, 1} \right] \\
        \le & 2m H^2 F_0^2 \sum_{i=1}^{m} \sum_{k=1}^{K} \sum_{h=1}^{H} \sum_{x \in L_k} \frac{w_{k, h}(x)}{\cntTrans{i, k}(x)} + 4l H^2L \sum_{i=1}^{l} \sum_{k=1}^{K} \sum_{h=1}^{H} \sum_{x \in \Lambda_k} \frac{w_{k, h}(x)}{\cntR{i, k}(x)} \\
        &\quad + \sum_{k=1}^{K} \sum_{h=1}^{H} \sum_{x \notin L_k} w_{k, h}(x) H^2 +  \sum_{k=1}^{K} \sum_{h=1}^{H} \sum_{x \notin \Lambda_k} w_{k, h}(x) H^2 \\
        \le & 2m H^2 F_0^2 \sum_{i=1}^{m}4 X[\idxTrans{i}] L + 4l H^2 L \sum_{i=1}^{l} 4X[\idxR{i}]L + 8 H^3 \sum_{i=1}^{m} X[\idxTrans{i}] L  + 8 H^3 \sum_{i=1}^{l} X[\idxR{i}] L \\
        \le &8 m H^2 F_0^2 \sum_{i=1}^{m}  X[\idxTrans{i}] L +16l H^2 \sum_{i=1}^{l} X[\idxR{i}]L^2 + 8 H^3 \sum_{i=1}^{m} X[\idxTrans{i}] L + 8 H^3 \sum_{i=1}^{l} X[\idxR{i}] L \\
        = & 208 m^4 H^4 (\max_{i} S_i)^2 \max_{i} X[\idxTrans{i}] L^3 + 24 l^2 H^3 \max_{i} X[\idxR{i}] L^2,
    \end{align*}
    where in the last equality we use the definition of $F_0$ (Lemma~\ref{lem:conf-r-h}).
\end{proof}

\subsection{Regret decomposition}

We decompose the regret in the following standard way~\citep{zanette2019tighter}, and then bound the sum over time of the individual terms in the next few subsections.
Here we assume the general transition bonus $\bTrans$ to be a function of step $h$, as in F-EULER.

\begin{lemma}[Regret decomposition]  \label{lem:regret-decomp}
    Let $L_{k}, \Lambda_{k}$ be the good sets defined in Definition~\ref{def:good-sets}. Then for any given FMDP specified in~\eqref{eqn:fmdp}, outside the failure event $\gF$, the regret of F-UCBVI in $K$ episodes satisfies that
    \begin{align*}
        \Regret(K) 
        & \le \sum_{k=1}^{K} \sum_{x\in \gX} w_{k, h}(x) \min\Bigl\{ \Bigl(\innerp{\eTrans_{k}(x) - \Trans(x)}{\uV_{k, h+1}} + \bTrans_{k, h}(x) + \eR_{k}(x) - R(x) + \bR_{k}(x) \Bigr), H \Bigr\} \\
        & \le \sum_{k = 1}^{K} \sum_{h = 1}^{H} \sum_{x\in L_{k}} w_{k, h}(x) \Bigl(
            \underbrace{\innerp{\eTrans_{k}(x) - \Trans(x)}{\oV{h+1}}}_{\text{transition estimation error}} + \bTrans_{k, h}(x) + \underbrace{\innerp{\eTrans_{k}(x) - \Trans(x)}{\uV_{k, h+1} - \oV{h+1}}}_{\text{correction term}}
        \Bigr) \\ 
        &\quad + \sum_{k = 1}^{K} \sum_{h = 1}^{H} \sum_{x\in \Lambda_{k}} w_{k, h}(x) 2\bR_{k}(x) + 8 H^2 \sum_{i=1}^{m} X[\idxTrans{i}] L + 8 H^2 \sum_{i=1}^{l} X[\idxR{i}] L,
    \end{align*}
    where the $\bTrans_{k, h}(x)$ term is referred to as ``transition optimism'' and $2\bR_{k}(x)$ term is referred to as ``reward estimation error and optimism''.
\end{lemma}

\begin{proof}
    Add a subscript $k$ to $\uQ_{h}$ in the \texttt{VI\_Optimism} procedure (Algorithm~\ref{alg:vi-optimism}) to denote the corresponding optimistic Q-value function in episode $k$.
    Since $\uV_{k, h}$ is an entrywise UCB of $\oV{h}$, we upper bound the regret by 
    \begin{align}
        \Regret(K) 
        & \le \sum_{k = 1}^{K} \uV_{k, 1}(s_{k, 1}) - V_1^{\pi_k}(s_{k, 1}) \\
        & = \sum_{k=1}^{K}\sum_{x\in \gX} w_{k,1}(x)  (Q_{k, 1} (x) - Q_1^{\pi_k}(x)) \nonumber \\
        & = \sum_{k=1}^{K}\sum_{x\in \gX} w_{k,1}(x) \Bigl( \min \left\{ \eR_k(x) + \bR_{k}(x) + \innerp{\eTrans_{k}(x)}{\uV_{k, 2}} + \bTrans_{k, 1}(x), H \right\} \nonumber \\
        &\quad - R(x) - \innerp{\Trans(x)}{V_2^{\pi_k}} \Bigr) \\
        & \le \sum_{k=1}^{K} \sum_{x\in \gX} w_{k,1}(x) \biggl( \min\Bigl\{ \Bigl( \eR_{k}(x) - R(x) + \bR_{k}(x) + \innerp{\eTrans_{k}(x) - \Trans(x)}{\uV_{k, 2}} \nonumber \\
        &\quad + \bTrans_{k, 1}(x) \Bigr), H \Bigr\} + \innerp{\Trans(x)}{\uV_{k, 2} - V_{2}^{\pi_{k}}} \biggr) \nonumber \\
        &= \sum_{k=1}^{K} \biggl( \sum_{x\in \gX} w_{k, 1}(x) \min\Bigl\{ \Bigl( \eR_{k}(x) - R(x) + \bR_{k}(x) + \innerp{\eTrans_{k}(x) - \Trans(x)}{\uV_{k, 2}} \nonumber \\
        &\quad + \bTrans_{k, 1}(x) \Bigr), H \Bigr\} + \sum_{x\in \gX} w_{k, 1}(x) \sum_{s'} \Trans(s'\vert x) \left(\uV_{k, 2}(s') - V_{2}^{\pi_{k}}(s')\right) \biggr).  \label{eqn:regret-decomp-iter}
    \end{align}
    Let $x' = (s', a')$. By definition, the visit probability $w_{k, h}(x)$ has the property that 
    \begin{align*}
        w_{k, h + 1}(x') = \sum_{x\in \gX} w_{k, h}(x) \Trans(s'\vert x) \sP(\pi_{k}(s', h) = a'),
    \end{align*}
    where $\sP(\cdot)$ denotes an appropriate probability measure. Hence, 
    \begin{align*}
        & \sum_{x\in \gX} w_{k,1}(x) \sum_{s'} \Trans(s'\vert x) \left(\uV_{k, 2}(s') - V_{2}^{\pi_{k}}(s') \right) \\
        = & \sum_{x\in \gX} w_{k, 1}(x) \sum_{x' \in \gX} P(s'\vert x) \sP(\pi_{k}(s', 1) = a') \left(Q_{k, 2}(x') - Q_{2}^{\pi_{k}}(x') \right) \\
        = & \sum_{x' \in \gX} w_{k, 2}(x') \left(Q_{k, 2}(x') - Q_{2}^{\pi_{k}}(x') \right).
    \end{align*}
    Substituting the above into~\eqref{eqn:regret-decomp-iter} yields 
    \begin{align*}
        \Regret(K) 
        & \le \sum_{k=1}^{K}\sum_{x\in \gX} w_{k,1}(x)  (Q_{k, 1} (x) - Q_1^{\pi_k}(x)) \\
        & \le \sum_{k=1}^{K} \biggl( \sum_{x\in \gX} w_{k, 1}(x) \min\Bigl\{ \Bigl( \eR_{k}(x) - R(x) + \bR_{k}(x) + \innerp{\eTrans_{k}(x) - \Trans(x)}{\uV_{k, 2}} \\
        &\quad + \bTrans_{k, 1}(x) \Bigr), H \Bigr\} + \sum_{x \in \gX} w_{k, 2}(x) \left(Q_{k, 2}(x) - Q_{2}^{\pi_{k}}(x) \right) \biggr).
    \end{align*}
    Inductively, we have 
    \begin{align*}
        \Regret(K)
        \le \sum_{k=1}^{K} \sum_{x\in \gX} w_{k, h}(x) \min\Bigl\{ \Bigl( \eR_{k}(x) - R(x) + \bR_{k}(x) + \innerp{\eTrans_{k}(x) - \Trans(x)}{\uV_{k, h+1}} + \bTrans_{k, h}(x) \Bigr), H \Bigr\}.
    \end{align*}
    Outside the failure event $\gF$ (specifically, $\gF_3$ for F-UCBVI), 
    \begin{align*}
        \Regret(K)
        & \le \sum_{k = 1}^{K} \sum_{h = 1}^{H} \sum_{x\in \gX} w_{k, h}(x) \min\left\{ \left(\innerp{\eTrans_{k}(x) - \Trans(x)}{\uV_{k, h+1}} + \bTrans_{k,h}(x) + 2 \bR_{k}(x) \right), H \right\} \\
        & = \sum_{k = 1}^{K} \sum_{h = 1}^{H} \sum_{x\in \gX} w_{k, h}(x) \min\Bigl\{ \Bigl(
            \innerp{\eTrans_{k}(x) - \Trans(x)}{\oV{h+1}} + \bTrans_{k,h}(x) \\
            &\quad + \innerp{\eTrans_{k}(x) - \Trans(x)}{\uV_{k, h+1} - \oV{h+1}} + 2 \bR_{k}(x)
        \Bigr), H \Bigr\}, \\
        & \le \sum_{k = 1}^{K} \sum_{h = 1}^{H} \sum_{x\in L_{k}} w_{k, h}(x) \Bigl(
            \innerp{\eTrans_{k}(x) - \Trans(x)}{\oV{h+1}} + \bTrans_{k, h}(x) + \innerp{\eTrans_{k}(x) - \Trans(x)}{\uV_{k, h+1} - \oV{h+1}}
        \Bigr) \\ 
        &\quad + \sum_{k = 1}^{K} \sum_{h = 1}^{H} \sum_{x\in \Lambda_{k}} w_{k, h}(x) 2\bR_{k}(x) + \sum_{k = 1}^{K} \sum_{h = 1}^{H} \sum_{x\notin L_{k}} w_{k, h}(x) H + \sum_{k = 1}^{K} \sum_{h = 1}^{H} \sum_{x\notin \Lambda_{k}} w_{k, h}(x) H \\
        & \le \sum_{k = 1}^{K} \sum_{h = 1}^{H} \sum_{x\in L_{k}} w_{k, h}(x) \Bigl(
            \innerp{\eTrans_{k}(x) - \Trans(x)}{\oV{h+1}} + \bTrans_{k, h}(x) + \innerp{\eTrans_{k}(x) - \Trans(x)}{\uV_{k, h+1} - \oV{h+1}}
        \Bigr) \\ 
        &+ \sum_{k = 1}^{K} \sum_{h = 1}^{H} \sum_{x\in \Lambda_{k}} w_{k, h}(x) 2\bR_{k}(x) + 8 H^2 \sum_{i=1}^{m} X[\idxTrans{i}] L + 8 H^2 \sum_{i=1}^{l}  X[\idxR{i}] L.
    \end{align*}
\end{proof}

\subsection{Bounds on the individual terms in regret}

To prove the following bounds on the individual terms in regret (Lemma~\ref{lem:regret-decomp}), we heavily use the lemmas derived from the notion of the good sets (Section~\ref{subsec:good-sets}).

\begin{lemma}[Cumulative transition estimation error, Hoeffding-style]  \label{lem:cul-est-err-h}
    For F-UCBVI, outside the failure event $\gF$, the sum over time of the transition estimation error satisfies that 
    \begin{align*}
        \sum_{k = 1}^{K} \sum_{h = 1}^{H} \sum_{x\in L_{k}} w_{k, h}(x) \innerpc{\eTrans_{k}(x) - \Trans(x)}{\oV{h+1}}
        \le \sum_{i=1}^{m} H\sqrt{2 X[\idxTrans{i}] T L} + 4m^2  H\max_{i} S_i \max_{i} X_i L^2.
    \end{align*}
\end{lemma}

\begin{proof}
    Outside the failure event $\gF$, by lemma~\ref{lem:est-err-h}, 
    \begin{align}
        &\sum_{k = 1}^{K} \sum_{h = 1}^{H} \sum_{x\in L_{k}} w_{k, h}(x) \innerp{\eTrans_{k}(x) - \Trans(x)}{\oV{h+1}} \nonumber \\
        \le & \sum_{k = 1}^{K} \sum_{h = 1}^{H} \sum_{x\in L_{k}} w_{k, h}(x)
        \left(  
            \sum_{i=1}^{m} H\sqrt{\frac{L}{2\cntTrans{i, k}(x)}} + \sum_{i=1}^{m} \sum_{j=i+1}^{m} 2HL \sqrt{\frac{S_i S_j}{\cntTrans{i, k}(x) \cntTrans{j, k}(x)}}
        \right),  \label{eqm:cul-est-err-h-decomp}
    \end{align}
    For the first term in~\eqref{eqm:cul-est-err-h-decomp}, by the Cauchy-Schwarz inequality and Lemma~\ref{lem:sum-vr}, 
    \begin{align*}
        \sum_{k = 1}^{K} \sum_{h = 1}^{H} \sum_{x\in L_{k}} w_{k, h}(x) \left(\sum_{i=1}^{m} H \sqrt{\frac{L}{2\cntTrans{i, k}(x)}} \right)
        & \le H \sqrt{\frac{L}{2}} \sum_{i=1}^{m} \sqrt{\sum_{k = 1}^{K} \sum_{h = 1}^{H} \sum_{x\in L_{k}} w_{k, h}(x)} 
        \cdot \sqrt{\sum_{k = 1}^{K} \sum_{h = 1}^{H} \sum_{x\in L_{k}} \frac{w_{k, h}(x)}{\cntTrans{i, k}(x)}} \\
        & \le \sum_{i=1}^{m} H\sqrt{2X[\idxTrans{i}] T L}.
    \end{align*}
    For the second term in~\eqref{eqm:cul-est-err-h-decomp}, by Lemma~\ref{lem:sum-mixed-vr}, 
    \begin{align*}
        \sum_{k = 1}^{K} \sum_{h = 1}^{H} \sum_{x\in L_{k}} w_{k, h}(x)
        \left(  
            \sum_{i=1}^{m} \sum_{j=i+1}^{m} 2HL \sqrt{\frac{S_i S_j}{\cntTrans{i, k}(x) \cntTrans{j, k}(x)}}
        \right)
        & \le m^2 HL \max_{i} S_i \cdot 4 \sqrt{X[\idxTrans{i}] X[\idxTrans{j}]} L \\
        & \le 4m^2 H\max_{i} S_i \max_{i} X_i L^2.
    \end{align*}
    Substituting the above two bounds into~\eqref{eqm:cul-est-err-h-decomp}completes the proof.
\end{proof}

\begin{lemma}[Cumulative transition optimism, Hoeffding-style]  \label{lem:cul-bonus-h}
    For F-UCBVI, outside the failure event $\gF$, the sum over time of the transition optimism satisfies that 
    \begin{align*}
        \sum_{k = 1}^{K} \sum_{h = 1}^{H} \sum_{x\in L_{k}} w_{k, h}(x) \bTrans_{k}(x)
        &\le \sum_{i=1}^{m} H\sqrt{2 X[\idxTrans{i}] T L} + 4m^2  H\max_{i} S_i \max_{i} X_i L^2.
    \end{align*}
\end{lemma}

\begin{proof}
    The proof is exactly the same as that of Lemma~\ref{lem:cul-est-err-h} by noting 
    \begin{align*}
        \bTrans_{k}(x) = \sum_{i=1}^{m} H\sqrt{\frac{L}{2\cntTrans{i, k}(x)}} + \sum_{i=1}^{m} \sum_{j=i+1}^{m} 2HL \sqrt{\frac{S_i S_j}{\cntTrans{i, k}(x) \cntTrans{j, k}(x)}}.
    \end{align*}
\end{proof}

\begin{lemma}[Cumulative reward estimation error and optimism, Hoeffding-style]  \label{lem:cul-bonus-r-h}
    For F-UCBVI, outside the failure event $\gF$, the sum over time of the reward estimation error and optimism satisfies that 
    \begin{align*}
        \sum_{k = 1}^{K} \sum_{h = 1}^{H} \sum_{x\in \Lambda_{k}} w_{k, h}(x) 2\bR_{k}(x) \le \sum_{i=1}^{l} 2\sqrt{2X[\idxR{i}] T} L
    \end{align*}
\end{lemma}

\begin{proof}
    Outside the failure event $\gF$, by the Cauchy-Schwarz inequality, 
    \begin{align*}
        \sum_{k = 1}^{K} \sum_{h = 1}^{H} \sum_{x\in \Lambda_{k}} w_{k, h}(x) \bR_{k}(x)
        &= \sum_{k = 1}^{K} \sum_{h = 1}^{H} \sum_{x\in \Lambda_{k}} w_{k, h}(x) \sum_{i=1}^{l} \sqrt{\frac{L}{2\cntR{i, k}(x)}} \\ 
        &\le \sqrt{\frac{L}{2}} \sum_{i=1}^{l} 
        \sqrt{\sum_{k=1}^{K}\sum_{h=1}^{H} \sum_{x\in \Lambda_k} \frac{w_{k, h}(x)}{\cntR{i, k}(x)}} 
        \sqrt{\sum_{k=1}^{K}\sum_{h=1}^{H} \sum_{x\in \Lambda_k} w_{k, h}(x)}  \\
        &\le \sum_{i=1}^{l}\sqrt{2X[\idxR{i}]TL},
    \end{align*}
    where the last inequality is due to Lemma~\ref{lem:sum-vr}.
\end{proof}

\begin{lemma}[Cumulative correction term, Hoeffding-style]  \label{lem:cul-corr-h}
    For F-UCBVI, outside the failure event $\gF$, the sum over time of the correction term satisfies that 
    \begin{align*}
        &\sum_{k = 1}^{K} \sum_{h = 1}^{H} \sum_{x\in L_{k}} w_{k, h}(x) \innerp{\eTrans_{k}(x) - \Trans(x)}{\uV_{k, h+1} - \oV{h+1}} \\
        \le & 45 m^3 H^2 (\max_{i} S_i)^{1.5} \max_{i} X[\idxTrans{i}] L^{2.5} + 14 m l H^{1.5} (\max_{i} S_i)^{0.5} (\max_{i} X[\idxTrans{i}])^{0.5} (\max_{i} X[\idxR{i}])^{0.5} L^2.
    \end{align*}
\end{lemma}

\begin{proof}
    Since $\uV_{k, h+1}$ is a random vector, we cannot apply scalar concentration as in bounding the transition estimation error (Lemma~\ref{lem:est-err-h}).
    However, some techniques there are useful here, including the inverse telescoping technique and the Holder's argument (Lemma~\ref{lem:holder}).

    Omitting the dependence of $\eTrans_{k}(x), \Trans(x), \eTrans_{i, k}(x), \Trans_{i}(x)$ on $x$, for any fixed $i$ and $s'[i] \in \gS_i$, by the inverse telescoping technique, 
    \begin{align}
        \innerp{\eTrans_{k} - \Trans}{\uV_{k, h+1} - \oV{h+1}}
        & = \innerp{\prod_{i=1}^{m} \eTrans_{i, k} - \prod_{i=1}^{m} \Trans_{i}}{\uV_{k, h+1} - \oV{h+1}} \nonumber \\
        & = \innerp{\sum_{i=1}^{m} (\eTrans_{i, k} - \Trans_{i}) \Trans_{1:i-1} \eTrans_{i+1:m, k}}{\uV_{k, h+1} - \oV{h+1}} \nonumber \\
        & = \sum_{i=1}^{m} \innerp{\eTrans_{i, k} - \Trans_{i}}{\E_{\Trans_{1:i-1}} \E_{\eTrans_{i+1:m, k}}[\uV_{k, h+1} - \oV{h+1}]} \nonumber \\
        & = \sum_{i=1}^{m} \innerp{\eTrans_{i, k} - \Trans_{i}}{\E_{\Trans_{1:i-1}} \E_{\Trans_{i+1:m}}[\uV_{k, h+1} - \oV{h+1}]} \label{eqn:correction-1} \\
        &\quad + \sum_{i=1}^{m} \innerp{\eTrans_{i, k} - \Trans_{i}}{\E_{\Trans_{1:i-1}} (\E_{\eTrans_{i+1:m, k}} - \E_{\Trans_{i+1:m}})[\uV_{k, h+1} - \oV{h+1}]}.  \label{eqn:correction-2}
    \end{align}  
    For~\eqref{eqn:correction-1}, outside the failure event $\gF$ (specifically, $\gF_6$), 
    \begin{align*}
        & \innerp{\eTrans_{i, k} - \Trans_{i}}{\E_{\Trans_{1:i-1}} \E_{\Trans_{i+1:m}}[\uV_{k, h+1} - \oV{h+1}]} \\
        \le & \sum_{s'[i] \in S_i} \left(
            \frac{2L}{3\cntTrans{i, k}(x)} 
            + \sqrt{\frac{2\Trans_{i}(s'[i]\vert x)L}{\cntTrans{i, k}(x)}}
        \right) \E_{\Trans_{1:i-1}} \E_{\Trans_{i+1:m}}[\uV_{k, h+1} - \oV{h+1}]\\
        \le & \frac{2H S_i L}{3 \cntTrans{i, k}(x)} + \sum_{s'[i] \in S_i} \sqrt{\frac{2\Trans_{i}(s'[i]\vert x)L}{\cntTrans{i, k}(x)}} \E_{\Trans_{-i}}[\uV_{k, h+1} - \oV{h+1}], 
    \end{align*}
    the sum over time of which are both bounded, respectively, by
    \begin{align*}
        \sum_{k = 1}^{K} \sum_{h = 1}^{H} \sum_{x\in L_{k}} w_{k, h}(x) \frac{2H S_i L}{3 \cntTrans{i, k}(x)}
        \le \frac{8}{3} HS_i X[\idxTrans{i}] L^2 \le 3HS_i X[\idxTrans{i}] L^2, 
    \end{align*}
    and  
    \begin{align*}
        & \sum_{k = 1}^{K} \sum_{h = 1}^{H} \sum_{x\in L_{k}} w_{k, h}(x) \sum_{s'[i] \in S_i} \sqrt{\frac{2\Trans_{i}(s'[i]\vert x)L}{\cntTrans{i, k}(x)}} \E_{\Trans_{-i}}[\uV_{k, h+1} - \oV{h+1}] \\
        \le & \sqrt{2S_i L} \sum_{k = 1}^{K} \sum_{h = 1}^{H} \sum_{x\in L_{k} } w_{k, h}(x) \sqrt{\frac{\E_{\Trans_{i}} (\E_{\Trans_{-i}}[\uV_{k, h+1} - \oV{h+1}])^2 }{\cntTrans{i, k}(x)}} \\
        \le & \sqrt{2S_i L} \sqrt{\sum_{k = 1}^{K} \sum_{h = 1}^{H} \sum_{x\in L_{k}} \frac{w_{k, h}(x)}{\cntTrans{i, k}(x)}} \cdot \sqrt{\sum_{k = 1}^{K} \sum_{h = 1}^{H} \sum_{x\in L_{k} } w_{k, h}(x) \E_{\Trans_{i}} (\E_{\Trans_{-i}}[\uV_{k, h+1} - \oV{h+1}])^2} \\
        \le & \sqrt{2S_i L} \cdot 2 \sqrt{X[\idxTrans{i}] L} \cdot \sqrt{G_0} \\
        \le & 41 m^2 H^2 (\max_{i} S_i)^{1.5} \max_{i} X[\idxTrans{i}] L^{2.5} + 14 l H^{1.5} (\max_{i} S_i)^{0.5} (\max_{i} X[\idxTrans{i}])^{0.5} (\max_{i} X[\idxR{i}])^{0.5} L^2,
    \end{align*}
    where the first and second inequalities are due to the Cauchy-Schwarz inequality, and the third inequality is due to Lemma~\ref{lem:sum-vr} and Lemma~\ref{lem:cul-conf-r-h}.
    With the same Holder's argument as in Lemma~\ref{lem:holder},~\eqref{eqn:correction-2} is upper bounded by 
    \begin{align*}
        \sum_{i=1}^{m} \innerp{\eTrans_{i, k} - \Trans_{i}}{\E_{\Trans_{1:i-1}} (\E_{\eTrans_{i+1:m, k}} - \E_{\Trans_{i+1:m}})[\uV_{k, h+1} - \oV{h+1}]}
        \le \sum_{i=1}^{m} \sum_{j=i+1}^{m} 2HL \sqrt{\frac{S_i S_j}{\cntTrans{i, k}(x) \cntTrans{j, k}(x)}},
    \end{align*}
    the sum over time of which is upper bounded by $4m^2 H\max_{i} S_i \max_{i} X_i L^2$ due to Lemma~\ref{lem:sum-mixed-vr}.
\end{proof}

\subsection{Regret bounds (proof of Theorem~\ref{thm:f-ucbvi})}

\begin{proof}
    Outside the failure event $\gF$, combining Lemmas~\ref{lem:regret-decomp},~\ref{lem:cul-est-err-h},~\ref{lem:cul-bonus-h},~\ref{lem:cul-bonus-r-h} and~\ref{lem:cul-corr-h}, we obtain 
    \begin{align*}
        \Regret(K) 
        & \le 2 \left(\sum_{i=1}^{m} H\sqrt{2 X[\idxTrans{i}] T L} + 4m^2  H\max_{i} S_i \max_{i} X_i L^2 \right) \\
        &\quad + 45 m^3 H^2 (\max_{i} S_i)^{1.5} \max_{i} X[\idxTrans{i}] L^{2.5} \\ 
        &\quad + 14 m l H^{1.5} (\max_{i} S_i)^{0.5} (\max_{i} X[\idxTrans{i}])^{0.5} (\max_{i} X[\idxR{i}])^{0.5} L^2 + 4 \sum_{i=1}^{l} \sqrt{X[\idxR{i}] T L} \\
        &\quad + 8 \sum_{i=1}^{m} X[\idxTrans{i}] H^2 L + 8 \sum_{i=1}^{l} X[\idxR{i}] H^2 L\\
        & \le 3\sum_{i=1}^{m} H\sqrt{X[\idxTrans{i}] T L} + 4\sum_{i=1}^{l} \sqrt{X[\idxR{i}] T L} + 53 m^3 H^2 (\max_{i} S_i)^{1.5} \max_{i} X[\idxTrans{i}] L^{2.5} \\
        &\quad + 22 m l H^2 (\max_{i} S_i)^{0.5} (\max_{i} X[\idxTrans{i}])^{0.5} \max_{i} X[\idxR{i}] L^2 \\
        &= \bigotilde\left( \sum_{i=1}^{m} \sqrt{H^2 X[\idxTrans{i}] T} +   \sum_{i=1}^{l} \sqrt{X[\idxR{i}] T}\right), 
    \end{align*}
    where in the last equality we assume that $T \ge \mathrm{poly}(m, l, \max_{i}S_{i}, \max_{i}X[\idxTrans{i}], H)$.

    To accommodate the case of known rewards, it suffices to remove the parts related to reward estimation and reward bonuses in both the algorithm and the analysis, which yields the regret bound $\bigotilde( \sum_{i=1}^{m} \sqrt{H^2 X[\idxTrans{i}] T} )$.
\end{proof}

\newpage
\section{Regret analysis of F-EULER}
\label{sec:appx-f-euler}

The basic structure of the regret analysis in this section is the same as that of F-UCBVI (Section~\ref{sec:appx-f-ucbvi}), e.g., the notion of the good sets, including the definitions and lemmas (specifically, Definition~\ref{def:good-sets} and Lemmas~\ref{lem:sum-out-of-good-set},~\ref{lem:vn-vp},~\ref{lem:sum-vr} and~\ref{lem:sum-mixed-vr}), carries through here.
The key difference is a more refined analysis using Bernstein-style concentrations (Lemmas~\ref{lem:concen-b} and~\ref{lem:concen-b-emp}).

\subsection{Failure event}

Recall that $L = \log (16mlSXT / \delta)$. We define the failure event for F-EULER as follows.
\begin{definition}[Failure events]  \label{def:fail-evnt-b}
    Define the events 
    \begin{align*}
        \gB_1 &:= \bigg\{ \exists (i\in [m], k \in [K], h\in [H], x\in \gX), \\
        &\qquad \qquad 
        \left| \innerp{\eTrans_{i, k}(x) - \Trans_{i}(x)}{\E_{\Trans_{-i}(x)} [\oV{h+1}]} \right| 
        > \sqrt{\frac{2 \Var_{\Trans_{i}(x)} \E_{\Trans_{-i}(x)} [\oV{h+1}] L}{\cntTrans{i, k}(x)}} 
        + \frac{2 H L}{3 \cntTrans{i, k}(x)}  \bigg\}, \\
        \gB_2 &:= \gF_2, \\
        \gB_3 &:= \left\{ \exists (i\in [l], k \in [K], x\in \gX), \quad \left| \eR_{i, k}(x) - R_i(x) \right| > \sqrt{\frac{2 \SV[\er_i(x)] L}{\cntR{i, k}(x)}} + \frac{14 L}{3 \cntR{i, k}(x)} \right\}, \\
        \gB_4 &:= \Biggl\{ \exists (i\in [m], k\in[K], h\in [H], x\in \gX, s[-i] \in S_{-i}), \\
        &\qquad \qquad \left| \innerp{\eTrans_{i, k}(x) - \Trans_{i}(x)}{\oV{h+1}(s[-i])} \right| >  H\sqrt{\frac{L}{2 \cntTrans{i, k}(x)}} \Biggr\}, \\
        \gB_5 &:= \biggl\{ \exists (i\in [m], k\in[K], h\in [H], x\in \gX), \\
        &\qquad \qquad \left| \sqrt{\Var_{\eTrans_{i, k}(x)} \E_{\Trans_{-i}(x)} [\oV{h}]} - \sqrt{\Var_{\Trans_{i}(x)} \E_{\Trans_{-i}(x)} [\oV{h}]} \right| 
        > 3H \sqrt{\frac{L}{\cntTrans{i, k}(x)}} \biggr\}, \\
        \gB_6 &:= \left\{ \exists (i\in [m], k\in[K], h\in[H], x\in \gX), \quad \left|\sqrt{\SV[\er_i(x)]} - \sqrt{\Var(r_i(x))}\right| > \sqrt{\frac{4L}{\cntR{i, k}(x)}} \right\}, \\
        \gB_7 &:= \gF_4, \quad \gB_8 := \gF_5, \quad \gB_9 := \gF_6,
    \end{align*}
    where in $\gB_3$ and $\gB_5$ we assume $\cntTrans{i, k} \ge 2$ (true for $x$ in the good set $L_{k}$), and in $\gB_6$ we assume $\cntR{i, k} \ge 2$ (true for $x$ in the good set $\Lambda_{k}$).
    Then the failure event for F-EULER is defined by $\gB := \bigcup_{i=1}^{9} \gB_{i}$.
\end{definition}

The following lemma shows that the failure event $\gB$ happens with low probability.

\begin{lemma}[Failure probability]  \label{lem:fail-prob-b}
    For any FMDP specified by~\eqref{eqn:fmdp}, during the running of F-EULER for $K$ episodes, the failure event $\gB$ happens with probability at most $\delta$.
\end{lemma}

\begin{proof}
    By Bernstein's inequality (Lemma~\ref{lem:concen-b}) and the union bound, $\gB_1$ happens with probability at most $\delta / 8$.
    By empirical Bernstein's inequality (Lemma~\ref{lem:concen-b-emp})~\citep{maurer2009empirical} and the union bound, $\gB_3$ happens with probability at most $\delta / 8$, where $1 / (\cntTrans{i, k}(x) - 1)$ is replaced by $2 / \cntTrans{i, k}(x)$ for $\cntTrans{i, k}(x) \ge 2$.
    By Hoeffding's inequality (Lemma~\ref{lem:concen-h}) and the union bound, $\gB_4$ happens with probability at most $\delta / 8$.
    By Theorem 10 in~\citep{maurer2009empirical} and the union bound, $\gB_6$ happens with probability at most $\delta / 8$, where $1 / (\cntR{i, k}(x) - 1)$ is replaced by $2 / \cntR{i, k}(x)$ for $\cntR{i, k}(x) \ge 2$.
    By the proof of Lemma~\ref{lem:fail-prob-h}, $\gB_2$ and $\gB_9$ happen with probability at most $\delta / 8$, respectively; 
    $\gB_7$ and $\gB_8$ happen with probability at most $\delta / 16$, respectively. The argument on $\gB_5$ is more involved, which is stated as follows.

    For $k\in [K], h\in [H], x\in \gX$, let $\SV_{\eTrans_{i, k}(x)} \E_{\Trans_{-i}(x)} [\oV{h}]$ denote the sample variance of $\E_{\Trans_{-i}(x)} [\oV{h}]$, whose relationship with the empirical variance is given by 
    \begin{align*}
        \SV_{\eTrans_{i, k}(x)} \E_{\Trans_{-i}(x)} [\oV{h}] = \frac{\cntTrans{i, k}(x)}{\cntTrans{i, k}(x) - 1} \Var_{\eTrans_{i, k}(x)} \E_{\Trans_{-i}(x)} [\oV{h}].
    \end{align*}
    Hence, for $\cntTrans{i, k}(x) \ge 2$, 
    \begin{align*}
        \SV_{\eTrans_{i, k}(x)} \E_{\Trans_{-i}(x)} [\oV{h}] - \Var_{\eTrans_{i, k}(x)} \E_{\Trans_{-i}(x)} [\oV{h}]
        = \frac{1}{\cntTrans{i, k}(x) - 1} \Var_{\eTrans_{i, k}(x)} \E_{\Trans_{-i}(x)} [\oV{h}] 
        \le \frac{2H^2}{\cntTrans{i, k}(x)}.
    \end{align*}
    By Theorem 10 in~\citep{maurer2009empirical} and the union bound, with probability at least $1 - \delta / 8$, for all $i\in [m], k\in[K], h\in [H], x\in \gX$ and $\cntTrans{i, k}(x) \ge 2$, 
    \begin{align*}
        \left| \sqrt{\SV_{\eTrans_{i, k}(x)} \E_{\Trans_{-i}(x)} [\oV{h}]}
        - \sqrt{\Var_{\Trans_{i}(x)} \E_{\Trans_{-i}(x)} [\oV{h}]} \right| 
        \le H \sqrt{\frac{2L}{\cntTrans{i, k}(x) - 1}}
        \le H \sqrt{\frac{4L}{\cntTrans{i, k}(x)}}, 
    \end{align*}
    which yields that 
    \begin{align*}
        & \left| \sqrt{\Var_{\eTrans_{i, k}(x)} \E_{\Trans_{-i}(x)} [\oV{h}]} 
        - \sqrt{\Var_{\Trans_{i}(x)} \E_{\Trans_{-i}(x)} [\oV{h}]} \right| \\
        \le & 
        \Bigl| \sqrt{\Var_{\eTrans_{i, k}(x)} \E_{\Trans_{-i}(x)} [\oV{h}]} 
        - \sqrt{\SV_{\eTrans_{i, k}(x)} \E_{\Trans_{-i}(x)} [\oV{h}]} 
        + \sqrt{\SV_{\eTrans_{i, k}(x)} \E_{\Trans_{-i}(x)} [\oV{h}]}
        - \sqrt{\Var_{\Trans_{i}(x)} \E_{\Trans_{-i}(x)} [\oV{h}]} \Bigr| \\
        \le & \sqrt{\frac{2H^2}{\cntTrans{i, k}(x)}} + H \sqrt{\frac{4L}{\cntTrans{i, k}(x)}} 
        \le 3H \sqrt{\frac{L}{\cntTrans{i, k}(x)}}, 
    \end{align*}
    where in the second inequality we use $|\sqrt{a} - \sqrt{b}| \le \sqrt{|a - b|}$ for all $a, b \ge 0$ and in the third inequality we use $L = \log (16mlSXT / \delta) \ge 2$.
    Therefore, $\gB_5$ happens with probability at most $\delta / 8$.
    
    Finally, applying the union bound on $\gB_i$ for $i \in [9]$ yields that the failure event $\gF$ happens with probability at most $\delta$. 
\end{proof}

The deduction in the rest of this section and hence the regret bound hold outside the failure event $\gB$, with probability at least $1-\delta$.
From the above derivation, note that we can actually use a smaller 
\begin{align*}
    L_1 = \log (16ml\max\{ S \max_{i} X[\idxTrans{i}], \max_{i} X[\idxR{i}] \} T / \delta)
\end{align*}
to replace $L$ for F-EULER. We use $L$ for simplicity.

\subsection{Upper confidence bound}

The goal of this section is to show that $\uV_{k, h}$ is an entrywise UCB of $\oV{h}$ for all $k\in [K], h\in [H]$.
Analogously to the analysis of F-UCBVI, this optimism is achieved by setting the transition (reward, respectively) bonus as the UCB of the transition (reward, respectively) estimation error.
While the reward estimation error is direct to bound (failure event $\gB_3$), to bound the transition estimation error, we need some properties of the relevant functions.

Recall that in~\eqref{eqn:gi}, for $i\in [m], \Trans_{i} \in \Delta(\gS_i), \Trans = \prod_{i=1}^{m} \Trans_i \in \Delta(\gS)$ and $V\in \sR^{\gS}$, we define $g_{i}(\Trans, V) := 2\sqrt{L} \sqrt{\Var_{\Trans_{i}} \E_{\Trans_{-i}} [V]}$.
Now for $i\in [m], k\in [K], \Trans \in \Delta(\gS)$, $V\in \sR^{\gS}$ and $x\in \gX$, we define 
\begin{align*}
    \phi_{i, k}(\Trans, V, x) 
    := \frac{g_{i}(\Trans, V)}{\sqrt{\cntTrans{i, k}(x)}} + \frac{2 H L}{3 \cntTrans{i, k}(x)}.
\end{align*}
The following lemma establishes some ``Lipschitzness'' properties of $g_i$, which are then used to prove a similar property of $\phi_{i, k}$.
\begin{lemma}[Properties of $g_i$]  \label{lem:ppt-g}
    Outside the failure event $\gB$, for any index $i\in[m]$, episode $k\in[K]$, step $h\in[H]$, state-action pair $x\in \gX$ and vectors $V_1, V_2\in \mathbb{R}^S$,
    \begin{align}
        \left| g_i(\Trans(x), V_1) - g_i(\Trans(x), V_2) \right|
        &\le \sqrt{2L} \| V_1 - V_2 \|_{2, \Trans(x)},  \label{eqn:gi-prop-1} \\
        \left| g_i(\eTrans_{k}(x), \oV{h}) - g_i(\Trans(x), \oV{h}) \right|
        &\le 3\sqrt{2}HL \sum_{j=1}^{m} \frac{1}{\sqrt{\cntTrans{j, k}(x)}}.  \label{eqn:gi-prop-2}
    \end{align}
\end{lemma}
\begin{proof}
    For any finite set $\gS$, $P \in \Delta(\gS)$ and vectors $V_1, V_2 \in \sR^{\gS}$, we have that $\sqrt{\Var_{P}[V_1]} \leq \sqrt{\Var_{P}[V_2]} + \sqrt{\Var_{P} [V_1 - V_2]}$ (see~\citep[Section D.3]{zanette2019tighter} for a proof).
    Then for the first property~\eqref{eqn:gi-prop-1}, 
    \begin{align*}
        \left| g_i(\Trans(x), V_1) - g_i(\Trans(x), V_2) \right|
        &\le \sqrt{2L} \sqrt{\Var_{\Trans_{i}(x)} \E_{\Trans_{1:i-1}(x)} \E_{\Trans_{i+1:m}(x)} [V_1 - V_2]} \\
        &\le \sqrt{2L} \sqrt{\E_{\Trans_{i}(x)} \left(\E_{\Trans_{1:i-1}(x)} \E_{\Trans_{i+1:m}(x)} [V_1 - V_2] \right)^2} \\
        &\le \sqrt{2L} \sqrt{\E_{P}[(V_1 - V_2)^2]} 
        := \sqrt{2L} \| V_1 - V_2 \|_{2, \Trans(x)}.
    \end{align*}
    For the second property~\eqref{eqn:gi-prop-2}, 
    \begin{align}
        \left| g_i(\eTrans_{k}(x), \oV{h}) - g_i(\Trans(x), \oV{h}) \right| 
        & = \sqrt{2L} \left| 
            \sqrt{ \Var_{\eTrans_{i, k}(x)} \E_{\eTrans_{-i, k}(x)}[\oV{h}]} 
            - \sqrt{\Var_{\Trans_{i}(x)} \E_{\Trans_{-i}(x)}[\oV{h}]}
        \right| \nonumber \\
        & = \sqrt{2L} \biggl| 
            \sqrt{ \Var_{\eTrans_{i, k}(x)} \E_{\eTrans_{-i, k}(x)}[\oV{h}]} 
            - \sqrt{ \Var_{\eTrans_{i, k}(x)} \E_{\Trans_{-i}(x)}[\oV{h}]} \nonumber \\
        &\quad + \sqrt{ \Var_{\eTrans_{i, k}(x)} \E_{\Trans_{-i}(x)}[\oV{h}]}
            - \sqrt{\Var_{\Trans_{i}(x)} \E_{\Trans_{-i}(x)}[\oV{h}]}
        \biggr| \nonumber \\
        & \le \sqrt{2L} \left| 
            \sqrt{ \Var_{\eTrans_{i, k}(x)} \E_{\eTrans_{-i, k}(x)}[\oV{h}]} 
            - \sqrt{ \Var_{\eTrans_{i, k}(x)} \E_{\Trans_{-i}(x)}[\oV{h}]} 
        \right|  \label{eqn:gi-prop-2-decomp-1} \\
        &\quad + \sqrt{2L} \left| 
            \sqrt{ \Var_{\eTrans_{i, k}(x)} \E_{\Trans_{-i}(x)}[\oV{h}]}
            - \sqrt{\Var_{\Trans_{i}(x)} \E_{\Trans_{-i}(x)}[\oV{h}]}
        \right|.  \label{eqn:gi-prop-2-decomp-2}
    \end{align}
    To bound~\eqref{eqn:gi-prop-2-decomp-1}, omitting the dependence of $\eTrans_{i, k}(x), \Trans_{i}(x)$ on $x$, outside the failure event $\gB$ (specifically, $\gB_4$), by an inverse telescoping argument, 
    \begin{align*}
        &\left| 
            \sqrt{ \Var_{\eTrans_{i, k}} \E_{\eTrans_{-i, k}}[\oV{h}]} 
            - \sqrt{ \Var_{\eTrans_{i, k}} \E_{\Trans_{-i}}[\oV{h}]} 
        \right| \\
        \le & \sum_{j=1, j\neq i}^{m} \left| 
        \sqrt{ \Var_{\eTrans_{i, k}} \E_{\Trans_{1:j-1\backslash i, k}}\E_{\eTrans_{j:m\backslash i, k}} [\oV{h}]} 
        - \sqrt{ \Var_{\eTrans_{i, k}} \E_{\Trans_{1:j\backslash i, k}}\E_{\eTrans_{j+1:m\backslash i, k}}[\oV{h}]} 
        \right|\\
        \le & \sum_{j=1, j\neq i}^{m} \sqrt{ \Var_{\eTrans_{i, k}} \E_{\Trans_{1:j-1\backslash i}} \E_{\eTrans_{j+1:m\backslash i, k}} (\E_{\eTrans_{j, k}} - \E_{\Trans_{j}})[\oV{h}] }  \\
        \le & \sum_{j=1, j\neq i}^{m} \sqrt{ \E_{\eTrans_{i, k}} \left(\E_{\Trans_{1:j-1\backslash i}} \E_{\eTrans_{j+1:m\backslash i, k}} (\E_{\eTrans_{j, k}} - \E_{\Trans_{j}})[\oV{h}]\right)^2 } 
        \le \sum_{j=1, j\neq i}^{m} H \sqrt{\frac{L}{2\cntTrans{j, k}(x)}},
    \end{align*}
    where ``$\cdot \backslash i$'' denotes excluding $i$.
    For~\eqref{eqn:gi-prop-2-decomp-2}, outside the failure event $\gB$ (specifically, $\gB_5$),
    \begin{align*}
        \left| \sqrt{\Var_{\eTrans_{i, k}(x)} \E_{\Trans_{-i}(x)} [\oV{h}]} - \sqrt{\Var_{\Trans_{i}(x)} \E_{\Trans_{-i}(x)} [\oV{h}]} \right| 
        \le 3H \sqrt{\frac{L}{\cntTrans{i, k}(x)}},
    \end{align*}
    Combining the above bounds on~\eqref{eqn:gi-prop-2-decomp-1} and~\eqref{eqn:gi-prop-2-decomp-2} yields 
    \begin{align*}
        \left| g_i(\eTrans_{k}(x), \oV{h}) - g_i(\Trans(x), \oV{h}) \right|
        \le 3\sqrt{2}HL \sum_{j=1}^{m} \frac{1}{\sqrt{\cntTrans{j, k}(x)}}.
    \end{align*}
\end{proof}

Then $\phi_{i, k}$ satisfies the following ``Lipschitzness'' property.

\begin{lemma}[Property of $\phi_{i, k}$] \label{lem:ppt-phi}
    Outside the failure event $\gB$, for any index $i\in [m]$, episode $k\in [K]$, step $h\in [H]$, state-action pair $x\in \gX$ and vector $V\in \mathbb{R}^S$,
    \begin{align*}
        \left| \phi_{i, k}(\eTrans_{k}(x), V, x) - \phi_{i, k}(\Trans(x), \oV{h+1}, x) \right|
        \le \frac{\sqrt{2L} \left\| V - \oV{h+1} \right\|_{2, \eTrans_{k}(x)}}{\sqrt{\cntTrans{i, k}(x)}} 
        + \frac{3\sqrt{2}HL}{\sqrt{\cntTrans{i, k}(x)}} \sum_{j=1}^{m} \frac{1}{\sqrt{\cntTrans{j, k}(x)}}.
    \end{align*}
\end{lemma}

\begin{proof}
    By Lemma~\ref{lem:ppt-g}, outside the failure event $\gB$, 
    \begin{align*}
        &\left| \phi_{i, k}(\eTrans_{k}(x), V, x) - \phi_{i, k}(\Trans(x), \oV{h+1}, x) \right| \\
        \le & \frac{1}{\sqrt{\cntTrans{i, k}(x)}} \left(\left| g_{i}(\eTrans_{k}(x), V) 
        - g_{i}(\eTrans_{k}(x), \oV{h+1}) \right|
        + \left| g_{i}(\eTrans_{k}(x), \oV{h+1})
        - g_{i}(\Trans(x), \oV{h+1}) \right| \right) \\
        \le & \frac{\sqrt{2L} \left\| V - \oV{h+1} \right\|_{2, \eTrans_{k}(x)}}{\sqrt{\cntTrans{i, k}(x)}} 
        + \frac{3\sqrt{2}HL}{\sqrt{\cntTrans{i, k}(x)}} \sum_{j=1}^{m} \frac{1}{\sqrt{\cntTrans{j, k}(x)}}.
    \end{align*}
\end{proof}

Now we are ready to present the bounds on transition estimation error in the following lemma.

\begin{lemma}[Transition estimation error, Bernstein-style] \label{lem:est-err-b}
    Outside the failure event $\gB$, for any episode $k\in [K]$, step $h\in [H]$ and state-action pair $x\in \gX$, 
    \begin{align*}
        \left| \innerp{\eTrans_{k}(x) - \Trans_{k}(x)}{\oV{h+1}} \right| 
        \le \sum_{i=1}^{m} \phi_{i, k}(\Trans(x), \oV{h+1},x) + \sum_{i=1}^{m} \sum_{j=i+1}^{m} 2HL \sqrt{\frac{S_i S_j}{\cntTrans{i, k}(x) \cntTrans{j, k}(x)}}.
    \end{align*}
    And for a given $k\in [K]$ and a given $h\in [H]$, if $\lV_{k, h+1} \le \oV{h+1} \le \uV_{k, h+1}$ entrywise, then the above inequality yields 
    \begin{align*}
        \left| \innerp{\eTrans_{k}(x) - \Trans_{k}(x)}{\oV{h+1}} \right|
        & \le \sum_{i=1}^{m} \phi_{i, k}(\eTrans_{k}(x), \uV_{k, h+1},x) 
        + \sum_{i=1}^{m} \frac{\sqrt{2L} \left\| \uV_{k, h+1} - \lV_{k, h+1} \right\|_{2, \eTrans_{k}(x)}}{\sqrt{\cntTrans{i, k}(x)}} \\
        &+ \sum_{i=1}^{m} \frac{3\sqrt{2} HL}{\sqrt{\cntTrans{i, k}(x)}} \sum_{j=1}^{m} \frac{1}{\sqrt{\cntTrans{j, k}(x)}} 
        + \sum_{i=1}^{m} \sum_{j=i+1}^{m} 2HL \sqrt{\frac{S_i S_j}{\cntTrans{i, k}(x) \cntTrans{j, k}(x)}}.
    \end{align*}
\end{lemma}

\begin{proof}
    Recall that in Lemma~\ref{lem:est-err-h}, for all $k\in [K], h\in [H], x\in \gX$, omitting the dependence of $\eTrans_{k}(x), \Trans(x), \allowbreak \eTrans_{i, k}(x), \Trans_{i}(x)$ on $x$, we decompose the transition estimation error as 
    \begin{align}
        \innerp{\eTrans_{k} - \Trans}{\oV{h+1}}
        &= \sum_{i=1}^{m} \innerp{\eTrans_{i, k} - \Trans_{i}}{\E_{\Trans_{1:i-1}} \E_{\Trans_{i+1:m}} [\oV{h+1}]} \label{eqn:est-err-b-decomp-1} \\
        &\quad + \sum_{i=1}^{m} \innerp{\eTrans_{i, k} - \Trans_{i}}{\E_{\Trans_{1:i-1}} (\E_{\eTrans_{i+1:m, k}} - \E_{\Trans_{i+1:m}}) [\oV{h+1}]}. \label{eqn:est-err-b-decomp-2}
    \end{align}
    Outside the failure event $\gB$ (specifically, $\gB_1$),~\eqref{eqn:est-err-b-decomp-1} is bounded by 
    \begin{align*}
        \left| \innerp{\eTrans_{i, k}(x) - \Trans_{i}(x)}{\E_{\Trans_{1:i-1}(x)} \E_{\Trans_{i+1:m}(x)} [\oV{h+1}]} \right|
        & \le \sqrt{\frac{2\Var_{\Trans_{i}(x)} \E_{\Trans_{1:i-1}(x)} \E_{\Trans_{i+1:m}(x)} [\oV{h+1}] L}{\cntTrans{i, k}(x)}} 
        + \frac{2HL}{3\cntTrans{i, k}(x)} \\
        & = \phi_{i, k}(\Trans(x), \oV{h+1}, x).
    \end{align*}
    If $\lV_{k, h+1} \le \oV{h+1} \le \uV_{k, h+1}$ entrywise, then by Lemma~\ref{lem:ppt-phi},~\eqref{eqn:est-err-b-decomp-1} is further bounded by 
    \begin{align*}
        \phi_{i, k}(\Trans(x), \oV{h+1}, x)
        & \le \phi_{i, k}(\eTrans_{k}(x), \uV_{h+1},x) 
        + \frac{\sqrt{2L} \left\| \uV_{k, h+1} - \oV{h+1} \right\|_{2, \eTrans_{k}(x)}}{\sqrt{\cntTrans{i, k}(x)}} 
        + \frac{3\sqrt{2} HL}{\sqrt{\cntTrans{i, k}(x)}} \sum_{j=1}^{m} \frac{1}{\sqrt{\cntTrans{j, k}(x)}} \\
        & \le \phi_{i, k}(\eTrans_{k}(x), \uV_{h+1},x) 
        + \frac{\sqrt{2L} \left\| \uV_{k, h+1} - \lV_{k, h+1} \right\|_{2, \eTrans_{k}(x)}}{\sqrt{\cntTrans{i, k}(x)}} 
        + \frac{3\sqrt{2} HL}{\sqrt{\cntTrans{i, k}(x)}} \sum_{j=1}^{m} \frac{1}{\sqrt{\cntTrans{j, k}(x)}}, 
    \end{align*}
    For~\eqref{eqn:est-err-b-decomp-2}, the Holder's argument (Lemma~\ref{lem:holder}) yields 
    \begin{align*}
        \left| \innerp{\eTrans_{i, k}(x) - \Trans_{i}(x)}{\E_{\Trans_{1:i-1}(x)} (\E_{\eTrans_{i+1:m, k}(x)} - \E_{\Trans_{i+1:m}(x)}) [\oV{h+1}]} \right|
        \le \sum_{j=i+1}^{m} 2HL \sqrt{\frac{S_i S_j}{\cntTrans{i, k}(x) \cntTrans{j, k}(x)}}.
    \end{align*} 
    Combining the above bounds on~\eqref{eqn:est-err-b-decomp-1} and~\eqref{eqn:est-err-b-decomp-2} completes the proof.
\end{proof}

The first Bernstein-style bound on the transition estimation error (Lemma~\ref{lem:est-err-b}) depends on the unknown $\oV{h+1}$, which is why we derive the second bound in terms of the UCB $\uV_{k, h}$ and LCB $\lV_{k, h}$.
Now simplify the second bound to the form of the transition bonus.
Expanding $\phi_{i, k}$ by definition, 
\begin{align*}
    &\sum_{i=1}^{m} \phi_{i, k}(\eTrans_{k}(x), \uV_{k, h+1},x) 
    + \sum_{i=1}^{m} \frac{\sqrt{2L} \left\| \uV_{k, h+1} - \lV_{k, h+1} \right\|_{2, \eTrans_{k}(x)}}{\sqrt{\cntTrans{i, k}(x)}} \\
    &\quad + \sum_{i=1}^{m} \frac{3\sqrt{2} HL}{\sqrt{\cntTrans{i, k}(x)}} \sum_{j=1}^{m} \frac{1}{\sqrt{\cntTrans{j, k}(x)}} 
    + \sum_{i=1}^{m} \sum_{j=i+1}^{m} 2HL \sqrt{\frac{S_i S_j}{\cntTrans{i, k}(x) \cntTrans{j, k}(x)}} \\
    =& \sum_{i=1}^{m}\frac{g_{i}(\eTrans_k(x), \uV_{k, h+1})}{\sqrt{\cntTrans{i, k}(x)}} + \sum_{i=1}^{m}\frac{2HL}{3\cntTrans{i, k}(x)} + \sum_{i=1}^{m} \frac{\sqrt{2L} \left\| \uV_{k, h+1} - \lV_{k, h+1} \right\|_{2, \eTrans_{k}(x)}}{\sqrt{\cntTrans{i, k}(x)}} \\
    &\quad + \sum_{i=1}^{m} \frac{3\sqrt{2}HL}{\sqrt{\cntTrans{i, k}(x)}} \sum_{j=1}^{m} \frac{1}{\sqrt{\cntTrans{j, k}(x)}} 
    + \sum_{i=1}^{m} \sum_{j=i+1}^{m} 2HL \sqrt{\frac{S_i S_j}{\cntTrans{i, k}(x) \cntTrans{j, k}(x)}}\\
    \le& \sum_{i=1}^{m} \frac{g_{i}(\eTrans_k(x), \uV_{k, h+1})}{\sqrt{\cntTrans{i, k}(x)}}
    + \sum_{i=1}^{m} \frac{\sqrt{2L} \left\| \uV_{k, h+1} - \lV_{k, h+1} \right\|_{2, \eTrans_k(x)}}{\sqrt{\cntTrans{i, k}(x)}} \\
    &\quad + \sum_{i=1}^{m} \sum_{j=i+1}^{m} 11 HL \sqrt{\frac{S_i S_j}{\cntTrans{i, k}(x) \cntTrans{j, k}(x)}} + \sum_{i=1}^{m} \frac{5 HL}{\cntTrans{i, k}(x)}, 
\end{align*}
which is precisely the transition bonus $\bTrans_{k, h}(x)$ in~\eqref{eqn:trans-bonus-f-euler}. 
Therefore, for any $k\in [K], h\in [H]$, 
\begin{align*}
    \left| \innerp{\eTrans_{k}(x) - \Trans_{k}(x)}{\oV{h+1}} \right| 
    \le & b_{k, h}(x), 
\end{align*}
if $\lV_{k, h+1} \le \oV{h+1} \le \uV_{k, h+1}$ holds entrywise, which we soon prove in Lemma~\ref{lem:ulcb}.
Note that our choice of the reward bonus upper bounds the reward estimation error outside the failure event $\gB$ (specifically, $\gB_{3}$), i.e., 
\begin{align*}
    \left| \eR_{k}(x) - R(x) \right| \le \sum_{i=1}^{l} \left| \eR_{i, k}(x) - R_i(x) \right| \le \sum_{i=1}^{l} \sqrt{\frac{2 \SV[\er_i(x)] L}{\cntR{i, k}(x)}} + \sum_{i=1}^{l} \frac{14 L}{3 \cntR{i, k}(x)} := \bR_{k}(x).
\end{align*}
With $x_{k, h} = (s, \pi_{k}(s, h))$, recall the optimistic and pessimistic value iterations are defined as
\begin{align*}
    \uV_{k, h}(s) &= \min\left\{ H - h + 1, \eR(x_{k, h})  + \innerp{\eTrans_{k}(x_{k, h})}{\uV_{k, h+1}} + \bTrans_{k, h}(x_{k, h}) + \bR_{k}(x_{k, h}) \right\}, \\
    \lV_{k, h}(s) &= \max\left\{ 0, R(x_{k, h}) + \innerp{\eTrans_{k}(x_{k, h})}{\lV_{k, h+1}} - \bTrans_{k, h}(x_{k, h}) - \bR_{k}(x_{k, h}) \right\}.
\end{align*}
The following lemma indicates that the Bernstein-style bonuses and the above value iterations ensure optimism and pessimism.
Specifically, $\uV_{k, h}$ and $\lV_{k, h}$ are entrywise upper and lower confidence bounds of $\oV{h}$ for all $k\in [K], h\in [H]$.
\begin{lemma}[Upper-lower confidence bounds]  \label{lem:ulcb}
    Outside the failure event $\gB$, for the choices of bonuses in~\eqref{eqn:trans-bonus-f-euler} and~\eqref{eqn:r-bonus-f-euler}, for any episode $k\in [K]$, step $h\in [H]$ and state $s\in \gS$, 
    \begin{align}  \label{eqn:ulcb}
        \lV_{k, h}(s) \le \oV{h}(s) \le \uV_{k, h}(s).
    \end{align}
\end{lemma}

\begin{proof}
    For $h = H + 1$, $\lV_{k, H+1}(s) = \oV{H+1}(s) = \uV_{k, H+1}(s) = 0$ for all $s\in \gS, k\in [K]$.
    We proceed by backward induction. For $h\in [H]$, assume~\eqref{eqn:ulcb} holds for $h+1$. 
    The transition bonus then satisfies $| \innerpc{\eTrans_{k}(x) - \Trans_{k}(x)}{\oV{h+1}} | \le b_{k, h}(x)$ for all $x\in \gX$.
    For all $s\in \gS$, with $x_{k, h} = (s, \pi_k(s, h))$ and $x_{h}^{*} = (s, \pi^{*}(s, h))$, $\uV_{k, h}$ satisfies that 
    \begin{align*}
        \uV_{k, h}(s) -\oV{h}(s)
        & = \eR(x_{k, h}) + \bR_{k}(x_{k, h})  + \innerp{\eTrans_{k}(x_{k, h})}{\uV_{k, h+1}} + \bTrans_{k}(x_{k, h}) - R(x_{h}^{*}) - \innerp{\Trans(x_{h}^{*})}{\oV{h+1}} \\
        & \ge \eR(x_{h}^{*})+ \bR_{k}(x^*_{h}) + \innerp{\eTrans_{k}(x_{h}^{*})}{\uV_{k, h+1}} + \bTrans_{k}(x_{h}^{*}) - R(x_{h}^{*}) - \innerp{\Trans(x_{h}^{*})}{\oV{h+1}} \\
        & \ge \innerp{\eTrans_{k}(x_{h}^{*})}{\uV_{k, h+1} - \oV{h+1}}
        + \innerp{\eTrans_{k}(x_{h}^{*}) - \Trans(x_{h}^{*})}{\oV{h+1}} 
        + \bTrans_{k}(x_{h}^{*}) \ge 0,
    \end{align*}
    and $\lV_{k, h}$ satisfies that 
    \begin{align*}
        \oV{h}(s) - \lV_{k, h}(s)
        & \ge R(x_{h}^{*}) + \innerp{\Trans(x_{h}^{*})}{\oV{h+1}} - \eR_{k}(x_{k, h}) - \bR_{k}(x_{k, h}) - \innerp{\eTrans_{k}(x_{k, h})}{\lV_{k, h+1}} + \bTrans_{k, h}(x_{k, h}) \\
        & \ge R(x_{k, h}) + \innerp{\Trans(x_{k, h})}{\oV{h+1}} - \eR_{k}(x_{k, h}) - \bR_{k}(x_{k, h}) - \innerp{\eTrans_{k}(x_{k, h})}{\lV_{k, h+1}} + \bTrans_{k, h}(x_{k, h}) \\
        & \ge \innerp{\eTrans_{k}(x_{k, h})}{\oV{h+1} - \lV_{k, h+1}} 
        + \innerp{\Trans(x_{k, h}) - \eTrans_{k}(x_{k, h})}{\oV{h+1}} 
        + \bTrans_{k, h}(x_{k, h}) \ge 0.
    \end{align*}
    Inductively, $\lV_{k, h} \le \oV{h} \le \uV_{k, h}$ holds entrywise for all $k\in [K]$ and $h\in [H]$.
\end{proof}

For F-EULER, we refer to the difference between the optimistic value function and the pessimistic value function as the \emph{confidence radius}, which we bound in the following lemma.
For here and below, we use ``$\lesssim, \approx$'' to denote ``$\le, =$'' neglecting constants. Different from the main text, we make lower-order terms explicit in the appendices.

\begin{lemma}[Confidence radius, Bernstein-style]  \label{lem:conf-r-b}
    Let $F_1 := mH\max_{i} S_i L$ be a lower-order term.
    Let $s_{k, t} \in \gS$ denote the state at step $t$ of episode $k$ and $x_{k, t} = (s_{k, t}, \pi_{k}(s_{k, t}, t))$.
    Outside the failure event $\gB$, for any episode $k\in [K]$, step $h\in [H]$ and state $s\in \gS$, the confidence radius of F-EULER satisfies that 
    \begin{align*}
        \uV_{k, h}(s) - \lV_{k, h}(s) \lesssim \min\left\{ \sum_{t=h}^{H} \E_{\pi_k}\left[ \sum_{i=1}^{m} \frac{F_1}{\sqrt{\cntTrans{i, k}(x_{k, t})}}   + \sum_{i=1}^{l} \frac{L}{ \sqrt{\cntR{i, k}(x_{k, t})} }  \middle| s_{k, h} = s \right], H \right\}.
    \end{align*}
\end{lemma}

\begin{proof}
    By definition, for any $k\in [K], h\in [H], s\in \gS$, 
    \begin{align}
        &\uV_{k, h}(s) - \lV_{k, h}(s) \nonumber \\
        \le & \eR_{k}(x_{k, h}) + \innerp{\eTrans_{k}(x_{k, h})}{\uV_{k, h+1}} + \bTrans_{k, h}(x_{k, h}) + \bR_{k}(x_{k, h}) \nonumber \\ 
        &\quad - \eR_{k}(x_{k, h}) - \innerp{\eTrans_{k}(x_{k, h})}{\lV_{k, h+1}} + \bTrans_{k, h}(x_{k, h}) + \bR_{k}(x_{k, h}) \nonumber \\
        = & \innerp{\eTrans_{k}(x_{k, h})}{\uV_{k, h+1} - \lV_{k, h+1}} + 2 \bTrans_{k, h}(x_{k, h}) + 2 \bR_{k}(x_{k, h}) \nonumber \\
        = & \innerp{\Trans(x_{k, h})}{\uV_{k, h+1} - \lV_{k, h+1}} + \innerp{\Trans(x_{k, h}) - \eTrans_{k}(x_{k, h})}{\uV_{k, h+1} - \lV_{k, h+1}} 
        + 2 \bTrans_{k, h}(x_{k, h}) + 2\bR_{k}(x_{k, h}) \nonumber \\
        \le & \innerp{\Trans(x_{k, h})}{\uV_{k, h+1} - \lV_{k, h+1}} 
        + \sum_{i=1}^{m} H \sqrt{\frac{2S_i L}{\cntTrans{i, k}(x_{k, h})}} 
        + 2 \bTrans_{k, h}(x_{k, h}) + 2\bR_{k}(x_{k, h}), \label{eqn:conf-r-b-decomp}
    \end{align}
    where the last inequality results from Holder's inequality and holds outside the failure event $\gB$ (specifically, $\gB_2$).
    We apply the following loose bounds on the transition bonus and the reward bonus that for any $k\in [K], h\in [H], x\in \gX$, 
    \begin{align*}
        \bTrans_{k, h}(x)
        &= \sum_{i=1}^{m} \frac{g_{i}(\eTrans_k(x), \uV_{k, h+1})}{\sqrt{\cntTrans{i, k}(x)}}
        + \sum_{i=1}^{m} \frac{\sqrt{2L} \left\| \uV_{k, h+1} - \lV_{k, h+1} \right\|_{2, \eTrans_k(x)}}{\sqrt{\cntTrans{i, k}(x)}} \\
        &\quad + \sum_{i=1}^{m} \sum_{j=i+1}^{m} 8 HL \sqrt{\frac{S_i S_j}{\cntTrans{i, k}(x) \cntTrans{j, k}(x)}} + \sum_{i=1}^{m} \frac{4 HL}{\cntTrans{i, k}(x)} \\
        & \le  \frac{2\sqrt{L} \sqrt{\Var_{\eTrans_{i, k}} \E_{\eTrans_{-i, k}} [\uV_{k, h+1}]} }{ \sqrt{\cntTrans{i, k}(x)}} + \sum_{i=1}^{m} \frac{\sqrt{2L} H}{\sqrt{\cntTrans{i, k}(x)}} \\
        &\quad + \sum_{i=1}^{m} \sum_{j=i+1}^{m} 8HL \max_{i} S_i \frac{1}{\sqrt{\cntTrans{i, k}(x)}} + \sum_{i=1}^{m} \frac{4HL}{\cntTrans{i, k}(x)} \\
        &\lesssim m H \max_{i} S_i L \sum_{i=1}^{m} \frac{1}{\sqrt{\cntTrans{i, k}(x)}}, 
    \end{align*}
    and 
    \begin{align*}
        \bR_{k}(x)
        = \sum_{i=1}^{l} \sqrt{\frac{2 \SV[\er_i(x)] L}{\cntR{i, k}(x)}} + \sum_{i=1}^{l} \frac{14 L}{3 \cntR{i, k}(x)}
        \lesssim L \sum_{i=1}^{l} \frac{1}{\sqrt{\cntR{i, k}(x)}}.
    \end{align*}
    Substituting the above bounds on $\bTrans_{k, h}(x)$ and $\bR_{k}(x)$ at $x = x_{k, h}$ into~\eqref{eqn:conf-r-b-decomp} yields 
    \begin{align*}
        \uV_{k, h}(s) - \lV_{k, h}(s)
        \lesssim & \innerp{\Trans(x_{k, h})}{\uV_{k, h+1} - \lV_{k, h+1}} 
        + \sum_{i=1}^{m} \frac{F}{\sqrt{\cntTrans{i, k}(x_{k, h})}} + \sum_{i=1}^{l} \frac{L}{\sqrt{\cntR{i, k}(x_{k, h})}}.
    \end{align*}    
    Inductively, we have 
    \begin{align*}
        \uV_{k, h}(s) - \lV_{k, h}(s) 
        \lesssim \min\left\{  \sum_{t=h}^{H} \E_{\pi_k}\left[ \sum_{i=1}^{m} \frac{ F_1}{\sqrt{\cntTrans{i, k}(x_{k, t})}} + \sum_{i=1}^{l} \frac{L}{\sqrt{\cntR{i, k}(x_{k, t})}}  \middle| s_{k, h} = s \right], H \right\}.
    \end{align*}
\end{proof}

As noted above, the notion of the good sets in the analysis of F-UCBVI carries over here, which again plays an important role in showing sum-over-time bounds.
Analogous to Lemma~\ref{lem:cul-conf-r-h}, the following lemma bounds the sum over time of the squared confidence radius, which is later used to bound the cumulative correction term (Lemma~\ref{lem:cul-corr-b}).

\begin{lemma}[Cumulative confidence radius, Bernstein-style]  \label{lem:cul-conf-r-b}
    Define the lower-order term
    \begin{align*}
        G_1 := m^4 H^4 (\max_{i} S_i)^2 \max_{i} X[\idxTrans{i}] L^3 + l^2 H^3 \max_{i} X[\idxR{i}] L^3.
    \end{align*}
    Then outside the failure event $\gB$, for all $i\in [m]$, the sum over time of the following expected squared confidence radius of F-EULER satisfies that 
    \begin{align*}
        \sum_{k=1}^{K} \sum_{h=1}^{H} \sum_{x\in \gX} w_{k, h}(x) \left( \E_{\Trans_i} \left(\E_{\Trans_{-i}}[\uV_{k, h+1} - \lV_{k, h+1}]\right)^2 \right)
        \le \sum_{k=1}^{K} \sum_{h=1}^{H} \sum_{x\in \gX} w_{k, h}(x) \left( \E_{\Trans} [(\uV_{k, h+1} - \lV_{k, h+1})^2] \right)
        \lesssim G_1.
    \end{align*}
\end{lemma}

\begin{proof}
    Replacing the failure event $\gF$ and confidence radius bound (Lemma~\ref{lem:conf-r-h}) of F-UCBVI by the failure event $\gB$ and confidence radius bound (Lemma~\ref{lem:conf-r-b}) of F-EULER, the proof of Lemma~\ref{lem:cul-conf-r-h} carries over here.
    Note that there is a minor difference in the order of $L$ between the second terms of $G_1$ (F-EULER) and $G_0$ (F-UCBVI), resulting from the $\sqrt{L}$ difference in the corresponding confidence radius bounds.
\end{proof}

Like the analysis of EULER~\citep{zanette2019tighter}, we show two problem-dependent regret bounds of F-EULER.
The following lemma bridges one bound to the other.

\begin{lemma}[Bound bridge]  \label{lem:bound-bridge}
    Outside the failure event $\gB$, for F-EULER, we have 
    \begin{align*}
        \sum_{k=1}^{K} \sum_{h=1}^{H} \sum_{x\in L_k} w_{k, h}(x) \frac{g_i(\Trans(x), \oV{h+1}) - g_i(\Trans(x), V_{h+1}^{\pi_k})}{\sqrt{\cntTrans{i, k}(x)}}
        \le 2 \sqrt{2L} H \sqrt{X[\idxTrans{i}] L} \sqrt{\Regret(K)}.
    \end{align*}
\end{lemma}

\begin{proof}
    Outside the failure event $\gB$, by the properties of $g_i$ (Lemma~\ref{lem:ppt-g}), 
    \begin{align*}
        &\sum_{k=1}^{K} \sum_{h=1}^{H} \sum_{x\in L_k} w_{k, h}(x) \frac{g_i(\Trans(x), \oV{h+1}) - g_i(\Trans(x), V_{h+1}^{\pi_k})}{\sqrt{\cntTrans{i, k}(x)}} \\
        \le & \sqrt{2L} \sum_{k=1}^{K} \sum_{h=1}^{H} \sum_{x\in L_k} w_{k, h}(x) \frac{\|\oV{h+1} - V_{h+1}^{\pi_k}\|_{2, \Trans(x)}}{\sqrt{\cntTrans{i, k}(x)}} \\
        \le & \sqrt{2L} \sqrt{\sum_{k=1}^{K} \sum_{h=1}^{H} \sum_{x\in L_k} \frac{w_{k, h}(x)}{\cntTrans{i, k(x)}}} \cdot \sqrt{\sum_{k=1}^{K} \sum_{h=1}^{H} \sum_{x\in L_k} w_{k, h}(x) \innerp{\Trans(x)}{(\oV{h+1} - V_{h+1}^{\pi_k})^2}} \\
        \le & 2 \sqrt{2L} \sqrt{X[\idxTrans{i}] L} \cdot \sqrt{H^2 \Regret(K)} = 2 \sqrt{2L} H \sqrt{X[\idxTrans{i}] L} \sqrt{\Regret(K)},
    \end{align*}
    where in the second inequality we use the Cauchy-Schwarz inequality, and the third inequality is due to Lemma~\ref{lem:sum-vr} in this work and Lemma 16 in~\citep{zanette2019tighter}.
\end{proof}

\subsection{Bounds on the individual terms in regret}

Replacing the failure event $\gF$ by $\gB$, the regret decomposition in Lemma~\ref{lem:regret-decomp} carries over here. Hence, outside the failure event $\gB$, 
\begin{align*}
    \Regret(K)
    & \le \sum_{k = 1}^{K} \sum_{h = 1}^{H} \sum_{x\in L_{k}} w_{k, h}(x) \Bigl(
        \underbrace{\innerp{\eTrans_{k}(x) - \Trans(x)}{\oV{h+1}}}_{\text{transition estimation error}} + \bTrans_{k, h}(x) + \underbrace{\innerp{\eTrans_{k}(x) - \Trans(x)}{\uV_{k, h+1} - \oV{h+1}}}_{\text{correction term}}
    \Bigr) \\ 
    &\quad + \sum_{k = 1}^{K} \sum_{h = 1}^{H} \sum_{x\in \Lambda_{k}} w_{k, h}(x) 2\bR_{k}(x) + 8 H^2 \sum_{i=1}^{m} X[\idxTrans{i}] L + 8 H^2 \sum_{i=1}^{l} X[\idxR{i}] L,
\end{align*}
where the $\bTrans_{k, h}(x)$ term is referred to as ``transition optimism'' and $2\bR_{k}(x)$ term is referred to as ``reward estimation error and optimism''.
In this subsection, we present the bounds on the above individual terms for F-EULER.

\begin{lemma}[Cumulative transition estimation error, Bernstein-style]  \label{lem:cul-est-err-b}
    Define 
    \begin{align*}
        \sC_{i}^{*} 
        &:= \frac{1}{T} \sum_{k = 1}^{K} \sum_{h = 1}^{H} \sum_{x\in \gX} w_{k, h}(x) g_i^2(\Trans(x), \oV{h+1})
        = \frac{1}{T} \sum_{k = 1}^{K} \sum_{h = 1}^{H} \E_{\pi_k} [g_i^2(P, \oV{h+1}) \vert s_{k, 1}], \\
        \sC_{i}^{\pi}
        &:= \frac{1}{T} \sum_{k = 1}^{K} \sum_{h = 1}^{H} \sum_{x\in \gX} w_{k, h}(x) g_i^2(\Trans(x), V_{h+1}^{\pi_k})
        = \frac{1}{T} \sum_{k = 1}^{K} \sum_{h = 1}^{H} \E_{\pi_k} [g_i^2(P, V_{h+1}^{\pi_k})\vert s_{k, 1}],
    \end{align*}
    where $s_{k, 1}$ denotes the initial state in the $k$th episode.
    Then for F-EULER, outside the failure event $\gB$, the sum over time of the transition estimation error satisfies that 
    \begin{align*}
        & \sum_{k = 1}^{K} \sum_{h = 1}^{H} \sum_{x\in L_{k}} w_{k, h}(x) \innerp{\eTrans_{k}(x) - \Trans(x)}{\oV{h+1}}
        \lesssim \sum_{i=1}^{m} \sqrt{\sC_{i}^{*} X[\idxTrans{i}] T} 
        + m^2 H \max_{i} S_i \max_{i} X[\idxTrans{i}] L^2, 
    \end{align*}
    and that 
    \begin{align*}
        &\sum_{k = 1}^{K} \sum_{h = 1}^{H} \sum_{x\in L_{k}} w_{k, h}(x) \innerp{\eTrans_{k}(x) - \Trans(x)}{\oV{h+1}} \\
        \lesssim & \sum_{i=1}^{m} \sqrt{\sC_{i}^{\pi} X[\idxTrans{i}] T L}
        + \sum_{i=1}^{m} H \sqrt{X[\idxTrans{i}] L} \sqrt{\Regret(K)}
        + m^2 H \max_{i} S_i \max_{i} X[\idxTrans{i}] L^2.
    \end{align*}
\end{lemma}

\begin{proof}
    Outside the failure event $\gB$, by lemma~\ref{lem:est-err-b}, 
    \begin{align}
        &\sum_{k = 1}^{K} \sum_{h = 1}^{H} \sum_{x\in L_{k}} w_{k, h}(x) \innerp{\eTrans_{k}(x) - \Trans(x)}{\oV{h+1}} \nonumber \\
        \le & \sum_{k = 1}^{K} \sum_{h = 1}^{H} \sum_{x\in L_{k}} w_{k, h}(x)
        \left(  
            \sum_{i=1}^{m} \phi_{i, k}(\Trans(x), \oV{h+1}, x) + \sum_{i=1}^{m} \sum_{j=i+1}^{m} 2HL \sqrt{\frac{S_i S_j}{\cntTrans{i, k}(x) \cntTrans{j, k}(x)}}
        \right) \nonumber \\
        = & \sum_{k = 1}^{K} \sum_{h = 1}^{H} \sum_{x\in L_{k}} w_{k, h}(x)
        \vast(   
            \sum_{i=1}^{m} \frac{g_i(\Trans(x), \oV{h+1})}{\sqrt{\cntTrans{i, k}(x)}} \label{eqn:cul-est-err-b-decomp-1} \\
            &\qquad \qquad \qquad \qquad \qquad \qquad 
            + \sum_{i=1}^{m} \frac{2HL}{3\cntTrans{i, k}(x)}
            + \sum_{i=1}^{m} \sum_{j=i+1}^{m} 2HL \sqrt{\frac{S_i S_j}{\cntTrans{i, k}(x) \cntTrans{j, k}(x)}}
        \vast). \label{eqn:cul-est-err-b-decomp-2}
    \end{align}
    By the Cauchy-Schwarz inequality and Lemma~\ref{lem:sum-vr}, the term~\eqref{eqn:cul-est-err-b-decomp-1} is bounded by 
    \begin{align*}
        &\sum_{k = 1}^{K} \sum_{h = 1}^{H} \sum_{x\in L_{k}} w_{k, h}(x) \left(\sum_{i=1}^{m} \frac{g_i(\Trans(x), \oV{h+1})}{\sqrt{\cntTrans{i, k}(x)}} \right) \\
        \le & \sum_{i=1}^{m} \sqrt{\sum_{k = 1}^{K} \sum_{h = 1}^{H} \sum_{x\in L_{k}} w_{k, h}(x) g_i^2(\Trans(x), \oV{h+1})} 
        \cdot \sqrt{\sum_{k = 1}^{K} \sum_{h = 1}^{H} \sum_{x\in L_{k}} \frac{w_{k, h}(x)}{\cntTrans{i, k}(x)}} \\
        \le & \sum_{i=1}^{m} 2 \sqrt{\sC_{i}^{*} X[\idxTrans{i}] T L}.
    \end{align*}
    By Lemma~\ref{lem:sum-vr} and Lemma~\ref{lem:sum-mixed-vr}, the terms in~\eqref{eqn:cul-est-err-b-decomp-2} are bounded by 
    \begin{align*}
        &\sum_{k = 1}^{K} \sum_{h = 1}^{H} \sum_{x\in L_{k}} w_{k, h}(x) \left(
            \sum_{i=1}^{m} \frac{2HL}{3\cntTrans{i, k}(x)}
            + \sum_{i=1}^{m} \sum_{j=i+1}^{m} 2HL \sqrt{\frac{S_i S_j}{\cntTrans{i, k}(x) \cntTrans{j, k}(x)}}
        \right) \\
        \le & \frac{2}{3} HL \sum_{i=1}^{m} 4X[\idxTrans{i}] L + 2 HL \max_{i} S_i \sum_{i=1}^{m} \sum_{j=i+1}^{m} 4 \sqrt{X[\idxTrans{i}] X[\idxTrans{j}]} L \\
        \lesssim & m^2 H \max_{i} S_i \max_{i} X[\idxTrans{i}] L^2.
    \end{align*}
    Combining the above bounds on~\eqref{eqn:cul-est-err-b-decomp-1} and~~\eqref{eqn:cul-est-err-b-decomp-2} yields the first bound in this lemma.
    To show the second bound, we bound~\eqref{eqn:cul-est-err-b-decomp-1} otherwise, by 
    \begin{align*}
        &\sum_{k = 1}^{K} \sum_{h = 1}^{H} \sum_{x\in L_{k}} w_{k, h}(x) \left(\sum_{i=1}^{m} \frac{g_i(\Trans(x), \oV{h+1})}{\sqrt{\cntTrans{i, k}(x)}} \right) \\
        = & \sum_{i=1}^{m} \sum_{k = 1}^{K} \sum_{h = 1}^{H} \sum_{x\in L_{k}} w_{k, h}(x) \left( 
            \frac{g_i(\Trans(x), V_{h+1}^{\pi_k})}{\sqrt{\cntTrans{i, k}(x)}} 
            + \frac{g_i(\Trans(x), \oV{h+1}) - g_i(\Trans(x), V_{h+1}^{\pi_k})}{\sqrt{\cntTrans{i, k}(x)}} 
        \right) \\
        \le & \sum_{i=1}^{m} 2\sqrt{\sC_{i}^{\pi} X[\idxTrans{i}] T L}
        + \sum_{i=1}^{m} 2 \sqrt{2L} H \sqrt{X[\idxTrans{i}] L} \sqrt{\Regret(K)},
    \end{align*}
    where the inequality is due to the bound bridge (Lemma~\ref{lem:bound-bridge}).
\end{proof}

\begin{lemma}[Cumulative transition optimism, Bernstein-style]  \label{lem:cul-bonus-b}
    For F-EULER, outside the failure event $\gB$, the sum over time of the transition optimism satisfies that 
    \begin{align*}
        \sum_{k = 1}^{K} \sum_{h = 1}^{H} \sum_{x\in L_{k}} w_{k, h}(x) \bTrans_{k, h}(x)
        &\lesssim \sum_{i=1}^{m} \sqrt{\sC_{i}^{*} X[\idxTrans{i}] T} + m^3 H^2 \max_{i} S_i \max_{i} X[\idxTrans{i}] L^2 \\
        &\quad + m^{1.5} l H^{1.5} (\max_{i} S_i)^{0.25} (\max_{i} X[\idxTrans{i}])^{0.75} (\max_{i} X[\idxR{i}])^{0.5} L^2,
    \end{align*}
    and that 
    \begin{align*}
        \sum_{k = 1}^{K} \sum_{h = 1}^{H} \sum_{x\in L_{k}} w_{k, h}(x) \bTrans_{k, h}(x)
        &\lesssim \sum_{i=1}^{m} \sqrt{\sC_{i}^{\pi} X[\idxTrans{i}] T L}
        + \sum_{i=1}^{m} H \sqrt{X[\idxTrans{i}] L} \sqrt{\Regret(K)} \\
        &\quad + m^3 H^2 \max_{i} S_i \max_{i} X[\idxTrans{i}] L^2 \\
        &\quad + m^{1.5} l H^{1.5} (\max_{i} S_i)^{0.25} (\max_{i} X[\idxTrans{i}])^{0.75} (\max_{i} X[\idxR{i}])^{0.5} L^2.
    \end{align*}
\end{lemma}

\begin{proof}
    Outside the failure event $\gB$, for all $k\in [K], h\in [H]$ and $x\in \gX$, by definition, 
    \begin{align*}
        \bTrans_{k, h}(x)
        &= \sum_{i=1}^{m} \frac{g_{i}(\eTrans_k(x), \uV_{k, h+1})}{\sqrt{\cntTrans{i, k}(x)}}
        + \sum_{i=1}^{m} \frac{\sqrt{2L} \left\| \uV_{k, h+1} - \lV_{k, h+1} \right\|_{2, \eTrans_k(x)}}{\sqrt{\cntTrans{i, k}(x)}} \\
        &\quad + \sum_{i=1}^{m} \sum_{j=i+1}^{m} 11 HL \sqrt{\frac{S_i S_j}{\cntTrans{i, k}(x) \cntTrans{j, k}(x)}} + \sum_{i=1}^{m} \frac{5 HL}{\cntTrans{i, k}(x)} \\
        &= \sum_{i=1}^{m} \phi_{i, k}(\eTrans_{k}(x), \uV_{h+1}, x) + \sum_{i=1}^{m} \frac{\sqrt{2L} \left\| \uV_{k, h+1} - \lV_{k, h+1} \right\|_{2, \eTrans_k(x)}}{\sqrt{\cntTrans{i, k}(x)}} \\
        &\quad + \sum_{i=1}^{m} \sum_{j=i+1}^{m} 11 HL \sqrt{\frac{S_i S_j}{\cntTrans{i, k}(x) \cntTrans{j, k}(x)}} + \sum_{i=1}^{m} \frac{13 HL}{3\cntTrans{i, k}(x)} \\
        &\le \sum_{i=1}^{m} \phi_{i, k}(\Trans(x), \oV{h+1}, x) 
        + \sum_{i=1}^{m} \frac{2\sqrt{2L} \left\| \uV_{k, h+1} - \lV_{k, h+1} \right\|_{2, \eTrans_{k}(x)}}{\sqrt{\cntTrans{i, k}(x)}} \\
        &\quad + \sum_{i=1}^{m} \frac{3\sqrt{2} HL}{\sqrt{\cntTrans{i, k}(x)}} \sum_{j=1}^{m} \frac{1}{\sqrt{\cntTrans{j, k}(x)}} 
        + \sum_{i=1}^{m} \sum_{j=i+1}^{m} 11 HL \sqrt{\frac{S_i S_j}{\cntTrans{i, k}(x) \cntTrans{j, k}(x)}} + \sum_{i=1}^{m} \frac{13 HL}{3\cntTrans{i, k}(x)} \\
        &\le \sum_{i=1}^{m} \phi_{i, k}(\Trans(x), \oV{h+1}, x) 
        + \sum_{i=1}^{m} \frac{2\sqrt{2L} \left\| \uV_{k, h+1} - \lV_{k, h+1} \right\|_{2, \eTrans_{k}(x)}}{\sqrt{\cntTrans{i, k}(x)}} \\
        &\quad + \sum_{i=1}^{m} \sum_{j=i+1}^{m} 20 HL \sqrt{\frac{S_i S_j}{\cntTrans{i, k}(x) \cntTrans{j, k}(x)}}
        + \sum_{i=1}^{m} \frac{9 HL}{\cntTrans{i, k}(x)}
    \end{align*}
    where the first inequality is due to Lemma~\ref{lem:ppt-phi}.
    Substituting the definition of $\phi_{i, k}$ into the above bound on $\bTrans_{k, h}(x)$, we have 
    \begin{align}
        & \sum_{k = 1}^{K} \sum_{h = 1}^{H} \sum_{x\in L_{k}} w_{k, h}(x) \bTrans_{k, h}(x) \nonumber \\
        \lesssim & \sum_{k = 1}^{K} \sum_{h = 1}^{H} \sum_{x\in L_{k}} w_{k, h}(x) \Biggl(
            \sum_{i=1}^{m} \frac{\sqrt{L} \left\| \uV_{k, h+1} - \lV_{k, h+1} \right\|_{2, \eTrans_{k}(x)}}{\sqrt{\cntTrans{i, k}(x)}} \label{eqn:cul-trans-o-b-decomp-1} \\
            &\quad + \sum_{i=1}^{m} \frac{g_i(\Trans(x), \oV{h+1})}{\sqrt{\cntTrans{i, k}(x)}} 
            + \sum_{i=1}^{m} \frac{HL}{\cntTrans{i, k}(x)}
            + \sum_{i=1}^{m} \sum_{j=i+1}^{m} HL \sqrt{\frac{S_i S_j}{\cntTrans{i, k}(x) \cntTrans{j, k}(x)}} \label{eqn:cul-trans-o-b-decomp-2}
        \Biggr).
    \end{align}
    By the Cauchy-Schwarz inequality and the inequality that $\sqrt{a + b} \le \sqrt{a} + \sqrt{b}$ for all $a, b \ge 0$, we bound the term in~\eqref{eqn:cul-trans-o-b-decomp-1} by
    \begin{align}
        & \sum_{k = 1}^{K} \sum_{h = 1}^{H} \sum_{x\in L_{k}} w_{k, h}(x) \frac{\left\| \uV_{k, h+1} - \lV_{k, h+1} \right\|_{2, \eTrans_{k}(x)}}{\sqrt{\cntTrans{i, k}(x)}} \nonumber \\
        \le & \sqrt{\sum_{k = 1}^{K} \sum_{h = 1}^{H} \sum_{x\in L_{k}} \frac{w_{k, h}(x)}{\cntTrans{i, k}(x)}} 
        \cdot \sqrt{\sum_{k = 1}^{K} \sum_{h = 1}^{H} \sum_{x\in L_{k}} w_{k, h}(x) \left\| \uV_{k, h+1} - \lV_{k, h+1} \right\|^2_{2, \eTrans_{k}(x)}} \nonumber \\
        \le & 2\sqrt{X[\idxTrans{i}] L} \cdot \Vast(
            \sqrt{\sum_{k = 1}^{K} \sum_{h = 1}^{H} \sum_{x\in L_{k}} w_{k, h}(x) \left\| \uV_{k, h+1} - \lV_{k, h+1} \right\|^2 _{2, \Trans(x)}} \label{eqn:cul-trans-o-b-v1} \\
            & \quad + \sqrt{\sum_{k = 1}^{K} \sum_{h = 1}^{H} \sum_{x\in L_{k}} w_{k, h}(x) \innerp{\eTrans_{k}(x) - \Trans(x)}{(\uV_{k, h+1} - \lV_{k, h+1})^2}}
        \Vast), \label{eqn:cul-trans-o-b-v2}
    \end{align}
    where the term in~\eqref{eqn:cul-trans-o-b-v1} is bounded, due to Lemma~\ref{lem:cul-conf-r-b}, by 
    \begin{align*}
        \sum_{k = 1}^{K} \sum_{h = 1}^{H} \sum_{x\in L_{k}} w_{k, h}(x) \left\| \uV_{k, h+1} - \lV_{k, h+1} \right\|^2 _{2, \Trans(x)} \lesssim G_1, 
    \end{align*}
    and the term in~\eqref{eqn:cul-trans-o-b-v2} is bounded, due to Lemma~\ref{lem:cul-corr-b}, by 
    \begin{align*}
        & \sum_{k = 1}^{K} \sum_{h = 1}^{H} \sum_{x\in L_{k}} w_{k, h}(x) \innerp{\eTrans_{k}(x) - \Trans(x)}{(\uV_{k, h+1} - \lV_{k, h+1})^2} \\
        \le & H \sum_{k = 1}^{K} \sum_{h = 1}^{H} \sum_{x\in L_{k}} w_{k, h}(x) \left|\innerp{\eTrans_{k}(x) - \Trans(x)}{\uV_{k, h+1} - \lV_{k, h+1}} \right| \\
        \lesssim & m^3 H^3 (\max_{i} S_i)^{1.5} \max_{i} X[\idxTrans{i}] L^{2.5} + m l H^{2.5} (\max_{i} S_i)^{0.5} (\max_{i} X[\idxTrans{i}])^{0.5} (\max_{i} X[\idxR{i}])^{0.5} L^{2.5}.
    \end{align*}
    By the proof of Lemma~\ref{lem:cul-est-err-b}, the same bounds as those on the cumulative transition estimation error applies to~\eqref{eqn:cul-trans-o-b-decomp-2}.
    Combining the above bounds on~\eqref{eqn:cul-trans-o-b-decomp-1} and~\eqref{eqn:cul-trans-o-b-decomp-2}, we obtain the first bound that 
    \begin{align*}
        & \sum_{k = 1}^{K} \sum_{h = 1}^{H} \sum_{x\in L_{k}} w_{k, h}(x) \bTrans_{k, h}(x) \\
        \lesssim & \sum_{i=1}^{m} \sqrt{\sC_{i}^{*} X[\idxTrans{i}] T} 
        + m^2 H \max_{i} S_i \max_{i} X[\idxTrans{i}] L^2 + m\sqrt{\max_{i}X[\idxTrans{i}] L} \sqrt{G_1} \\
        &\quad + m\sqrt{\max_{i} X[\idxTrans{i}] L} \sqrt{ m^3 H^3 (\max_{i} S_i)^{1.5} \max_{i} X[\idxTrans{i}] L^{2.5} } \\
        &\quad + m\sqrt{\max_{i} X[\idxTrans{i}] L} \sqrt{ m l H^{2.5} (\max_{i} S_i)^{0.5} (\max_{i} X[\idxTrans{i}])^{0.5} (\max_{i} X[\idxR{i}])^{0.5} L^{2.5}} \\
        \lesssim & \sum_{i=1}^{m} \sqrt{\sC_{i}^{*} X[\idxTrans{i}] T} 
        + m^2 H \max_{i} S_i \max_{i} X[\idxTrans{i}] L^2 \\
        &\quad + m (\max_{i}X[\idxTrans{i}])^{0.5} L^{0.5} m^2 H^2 \max_{i} S_i (\max_{i} X[\idxTrans{i}])^{0.5} L^{1.5} \\ 
        &\quad + m (\max_{i}X[\idxTrans{i}])^{0.5} L^{0.5} l H^{1.5} (\max_{i} X[\idxR{i}])^{0.5} L^{1.5} \\
        &\quad + m (\max_{i}X[\idxTrans{i}])^{0.5} L^{0.5} m^{1.5} H^{1.5} (\max_{i} S_i)^{0.75} (\max_{i} X[\idxTrans{i}])^{0.5} L^{1.25} \\
        &\quad + m (\max_{i}X[\idxTrans{i}])^{0.5} L^{0.5} m^{0.5} l^{0.5} H^{1.25} (\max_{i} S_i)^{0.25} (\max_{i} X[\idxTrans{i}])^{0.25} (\max_{i} X[\idxR{i}])^{0.25} L^{1.25} \\
        \lesssim & \sum_{i=1}^{m} \sqrt{\sC_{i}^{*} X[\idxTrans{i}] T} + m^3 H^2 \max_{i} S_i \max_{i} X[\idxTrans{i}] L^2 \\
        &\quad + m^{1.5} l H^{1.5} (\max_{i} S_i)^{0.25} (\max_{i} X[\idxTrans{i}])^{0.75} (\max_{i} X[\idxR{i}])^{0.5} L^2,
    \end{align*}
    and the second bound that 
    \begin{align*}
        & \sum_{k = 1}^{K} \sum_{h = 1}^{H} \sum_{x\in L_{k}} w_{k, h}(x) \bTrans_{k, h}(x) \\
        \lesssim & \sum_{i=1}^{m} \sqrt{\sC_{i}^{\pi} X[\idxTrans{i}] T} 
        + \sum_{i=1}^{m} H \sqrt{X[\idxTrans{i}] L} \sqrt{\Regret(K)} \\
        &\quad + m^3 H^2 \max_{i} S_i \max_{i} X[\idxTrans{i}] L^2  + m^{1.5} l H^{1.5} (\max_{i} S_i)^{0.25} (\max_{i} X[\idxTrans{i}])^{0.75} (\max_{i} X[\idxR{i}])^{0.5} L^2.
    \end{align*}
\end{proof}

\begin{lemma}[Cumulative reward estimation error and optimism, Bernstein-style] \label{lem:cul-bonus-r-b}
    For F-EULER, outside the failure event $\gB$, the sum over time of the reward estimation error and optimism satisfies that 
    \begin{align*}
        \sum_{k = 1}^{K} \sum_{h = 1}^{H} \sum_{x\in \Lambda_{k}} w_{k, h}(x) 2\bR_{k}(x_{k, h})  
        & \lesssim \sum_{i=1}^{l} \sqrt{ \gR_{i} X[\idxR{i}] T} L + \sum_{i=1}^{l} X[\idxR{i}] L^2,
    \end{align*}
    and that 
    \begin{align*}
        \sum_{k = 1}^{K} \sum_{h = 1}^{H} \sum_{x\in \Lambda_{k}} w_{k, h}(x) 2\bR_{k}(x_{k, h})  
        &\lesssim \sum_{i=1}^{l} \sqrt{ \gG^2 X[\idxR{i}] K} L + \sum_{i=1}^{l} X[\idxR{i}] L^2.
    \end{align*}
\end{lemma}

\begin{proof}
    The treatment of the cumulative $\bR_{k, h}(x)$ is essentially the same as the proof of Lemma 8 in~\citep{zanette2019tighter}.
    Recall that $\gR_{i} := \max_{x\in \gX} \left\{ \Var [r_{i}(x)] \right\}$.
    Outside the failure event $\gB$ (specifically, $\gB_6$), 
    \begin{align}
        & \sum_{k=1}^{K} \sum_{h=1}^{H} \sum_{x\in \Lambda_{k}} w_{k, h}(x) \bR_{k}(x_{k, h}) \nonumber \\
        =& \sum_{k=1}^{K} \sum_{h=1}^{H} \sum_{x\in \Lambda_{k}} w_{k, h}(x) \sum_{i=1}^{l}  \left(\sqrt{\frac{2 \SV[\er_i(x)] L}{\cntR{i, k}(x)}}+ \frac{14 L}{3 \cntR{i, k}(x)} \right) \nonumber \\ 
        \le & \sum_{k=1}^{K} \sum_{h=1}^{H} \sum_{x\in \Lambda_{k}} w_{k, h}(x) \sum_{i=1}^{l} \left(\sqrt{\frac{2 \Var(r_i(x)) L}{\cntR{i, k}(x)}} + \frac{\sqrt{4 L}}{\cntR{i, k}(x)} + \frac{14 L}{3 \cntR{i, k}(x)} \right) \nonumber \\
        \lesssim & \sqrt{L} \sum_{i=1}^{l} \sqrt{\sum_{k=1}^{K} \sum_{h=1}^{H} \sum_{x\in \Lambda_k} \frac{w_{k, h}(x)}{\cntR{i, k}(x)}} \sqrt{\sum_{k=1}^{K} \sum_{h=1}^{H} \sum_{x\in \Lambda_k} w_{k, h}(x)  \Var(r_i(x))}  + X[\idxR{i}] L^2 \label{eqn:cul-bonus-r-b-decomp} \\
        \lesssim & \sum_{i=1}^{l} \sqrt{ \gR_{i} X[\idxR{i}] T} L + X[\idxR{i}]L^2, \nonumber
     \end{align}
    where the second inequality is due to the Cauchy-Schwarz inequality, and the last inequality is due to Lemma \ref{lem:sum-vr}.

    Let $s_{k, h} \in \gS$ denote the state at step $h$ of episode $k$ and $x_{k, h} = (s_{k, h}, \pi_{k}(s_{k, h}, h))$.
    Recall that $\sum_{h=1}^{H} r_i(x_{k, h}) \le \gG$ for any sequence $\{x_{k, h}\}_{h=1}^{H}$.
    To obtain the second bound, consider 
    \begin{align*}
        \Var \left[ \sum_{h=1}^{H} r_i(x_{k, h}) \middle| \{s_{k, h}\}_{h=1}^{H} \right]  
        \le \E \left[ \left(\sum_{h=1}^{H} r_i(x_{k, h})\right)^2 \middle| \{s_{k, h}\}_{h=1}^{H} \right]
        \le \gG^2.
    \end{align*}    
    Take the expectation over the trajectory $\{s_{k, h}\}_{h=1}^{H}$ yields 
    \begin{align*}
        \E_{\{s_{k, h}\}_{h=1}^{H}} \left[ \Var \left[ \sum_{h=1}^{H} r_i(x_{k, h}) \middle| \{s_{k, h}\}_{h=1}^{H} \right] \right] \le \gG^2.
    \end{align*}
    Since $r_{i}(\cdot)$ has independent randomness in different steps, in an alternative notation and taking sum over $k\in [K]$, we have 
    \begin{align*}
        \sum_{k=1}^{K} \sum_{h=1}^{H} \sum_{x\in \Lambda_k} w_{k, h}(x) \Var(r_i(x)) \le K \gG^2.
    \end{align*}
    Substituting the above into~\eqref{eqn:cul-bonus-r-b-decomp} yields 
    \begin{align*}
        \sum_{k=1}^{K} \sum_{h=1}^{H} \sum_{x\in \Lambda_{k}} w_{k, h}(x) 2\bR_{k}(x_{k, h}) 
        \lesssim \sum_{i=1}^{l}\sqrt{ \gG^2 X[\idxR{i}] K} L + \sum_{i=1}^{l} X[\idxR{i}] L^2.
    \end{align*}
\end{proof}

\begin{lemma}[Cumulative correction term, Bernstein-style]  \label{lem:cul-corr-b}
    For F-EULER, outside the failure event $\gB$, the sum over time of the correction term satisfies that 
    \begin{align*}
        & \sum_{k = 1}^{K} \sum_{h = 1}^{H} \sum_{x\in L_{k}} w_{k, h}(x) \innerp{\eTrans_{k}(x) - \Trans(x)}{\uV_{k, h+1} - \oV{h+1}} \\ 
        \le & \sum_{k = 1}^{K} \sum_{h = 1}^{H} \sum_{x\in L_{k}} w_{k, h}(x) \left|\innerp{\eTrans_{k}(x) - \Trans(x)}{\uV_{k, h+1} - \oV{h+1}} \right| \\ 
        \lesssim & m^3 H^2 (\max_{i} S_i)^{1.5} \max_{i} X[\idxTrans{i}] L^{2.5} + m l H^{1.5} (\max_{i} S_i)^{0.5} (\max_{i} X[\idxTrans{i}])^{0.5} (\max_{i} X[\idxR{i}])^{0.5} L^{2.5}.
    \end{align*}
\end{lemma}

\begin{proof}
    Replacing the failure event $\gF$ by $\gB$ and the lower-order term $G_0$ by $G_1$, the proof of the bound on the cumulative correction term for F-UCBVI (Lemma~\ref{lem:cul-corr-h}) carries over. 
    Note that the second term in the bound here has an extra $\sqrt{L}$ factor due to the difference between $G_0$ and $G_1$.
\end{proof}

\subsection{Regret bounds}

The following lemma is useful in deriving the problem-dependent regret bounds of F-EULER.

\begin{lemma}[Component variance bound]  \label{lem:comp-var}
    For any finite set $\gS = \bigotimes_{i=1}^{m} \gS_i$, any probability mass function $\Trans = \prod_{i=1}^{m} \Trans_{i} \in \Delta(\gS)$ where $\Trans_{i} \in \Delta(\gS_i)$ for $i\in [m]$, and any vector $V \in \sR^{\gS}$, 
    \begin{align*}
        \Var_{\Trans_{i}} \E_{\Trans_{-i}} [V] 
        \le \Var_{\Trans} [V].
    \end{align*}
\end{lemma}

\begin{proof}
    By the definition of variance, 
    \begin{align*}
        \Var_{\Trans} [V] - \Var_{\Trans_{i}} \E_{\Trans_{-i}} [V]
        & = \E_{\Trans} [V^2] - \left(\E_{\Trans} [V] \right)^2
        - \E_{\Trans_{i}} \left[ \left(\E_{\Trans_{-i}} [V] \right)^2 \right] + \left(\E_{\Trans} [V]\right)^2 \\
        & = \E_{\Trans_{i}} \left[\Var_{\Trans_{-i}} [V] \right] \\
        & \ge 0.
    \end{align*}
\end{proof}

\subsubsection{Proof of Theorem~\ref{thm:f-euler}}

\begin{proof}
    To prove Theorem~\ref{thm:f-euler}, we actually need to show four combinations of the bounds therein.

    Outside the failure event $\gB$, combining Lemmas~\ref{lem:regret-decomp},~\ref{lem:cul-est-err-b},~\ref{lem:cul-bonus-b},~\ref{lem:cul-bonus-r-b} and~\ref{lem:cul-corr-b}, for F-EULER, we obtain 
    \begin{align*}
        \Regret(K)
        & \lesssim \sum_{i=1}^{m} \sqrt{\sC_{i}^{*} X[\idxTrans{i}] T} + m^3 H^2 \max_{i} S_i \max_{i} X[\idxTrans{i}] L^2 \\
        &\quad + m^{1.5} l H^{1.5} (\max_{i} S_i)^{0.25} (\max_{i} X[\idxTrans{i}])^{0.75} (\max_{i} X[\idxR{i}])^{0.5} L^2, \\
        &\quad + \sum_{i=1}^{l} \sqrt{ \gR_{i} X[\idxR{i}] T} L + \sum_{i=1}^{l} X[\idxR{i}] L^2 \\
        &\quad + m^3 H^2 (\max_{i} S_i)^{1.5} \max_{i} X[\idxTrans{i}] L^{2.5} \\ 
        &\quad + m l H^{1.5} (\max_{i} S_i)^{0.5} (\max_{i} X[\idxTrans{i}])^{0.5} (\max_{i} X[\idxR{i}])^{0.5} L^{2.5} \\
        &\quad + H^2 \sum_{i=1}^{m} X[\idxTrans{i}] L + H^2 \sum_{i=1}^{l} X[\idxR{i}] L \\
        & \lesssim \sum_{i=1}^{m} \sqrt{\sC_{i}^{*} X[\idxTrans{i}] T} 
        + \sum_{i=1}^{l} \sqrt{ \gR_{i} X[\idxR{i}] T} L 
        + m^3 H^2 (\max_{i} S_i)^{1.5} \max_{i} X[\idxTrans{i}] L^{2.5} \\
        &\quad + m^{1.5} l H^{2} (\max_{i} S_i)^{0.5} (\max_{i} X[\idxTrans{i}])^{0.75} \max_{i} X[\idxR{i}] L^{2.5}, 
    \end{align*}
    and similarly, 
    \begin{align*}
        \Regret(K)
        & \lesssim \sum_{i=1}^{m} \sqrt{\sC_{i}^{*} X[\idxTrans{i}] T} 
        + \sum_{i=1}^{l} \sqrt{ \gG^2 X[\idxR{i}] K} L 
        + m^3 H^2 (\max_{i} S_i)^{1.5} \max_{i} X[\idxTrans{i}] L^{2.5} \\
        &\quad + m^{1.5} l H^{2} (\max_{i} S_i)^{0.5} (\max_{i} X[\idxTrans{i}])^{0.75} \max_{i} X[\idxR{i}] L^{2.5}.
    \end{align*}
    By definition, 
    \begin{align*}
        \sC_i^{*} = \frac{1}{T} \sum_{k = 1}^{K} \sum_{h = 1}^{H} \E_{\pi_k} [g_i^2(P, \oV{h+1})\vert s_{k, 1}] \le (2\sqrt{L})^2 \gQ_i = 4 \gQ_i L.
    \end{align*}
    Substituting the above bound on $\sC_i^{*}$ into the above two regret bounds, we obtain two combinations of the bounds in Theorem~\ref{thm:f-euler}. Specifically, assuming $T \ge \mathrm{poly}(m, l, \max_{i}S_{i}, \max_{i}X[\idxTrans{i}], H)$, we have 
    \begin{align*}
        \Regret(K) &= \bigotilde(\sum_{i=1}^{m} \sqrt{\gQ_i X[\idxTrans{i}]T} + \sum_{i=1}^{l} \sqrt{\gR_{i} X[\idxR{i}]T}), \\
        \Regret(K) &= \bigotilde(\sum_{i=1}^{m} \sqrt{\gQ_i X[\idxTrans{i}]T} + \sum_{i=1}^{l} \sqrt{\gG X[\idxR{i}]K}).
    \end{align*}

    For other two combinations of the regret bounds of F-EULER, outside the failure event $\gB$, combining Lemmas~\ref{lem:regret-decomp},~\ref{lem:cul-est-err-b},~\ref{lem:cul-bonus-b},~\ref{lem:cul-bonus-r-b} and~\ref{lem:cul-corr-b} yet with the alternative bounds in Lemmas~\ref{lem:cul-est-err-b} and~\ref{lem:cul-bonus-b}, we obtain 
    \begin{equation}  \label{eqn:regret-f-euler-ineq-1}
        \begin{aligned}
            \Regret(K)
            & \lesssim \sum_{i=1}^{m} \sqrt{\sC_{i}^{\pi} X[\idxTrans{i}] T} 
            + \sum_{i=1}^{m} H \sqrt{X[\idxTrans{i}] L} \sqrt{\Regret(K)} 
            + \sum_{i=1}^{l} \sqrt{ \gR_{i} X[\idxR{i}] T} L \\
            &\quad + m^3 H^2 (\max_{i} S_i)^{1.5} \max_{i} X[\idxTrans{i}] L^{2.5} \\
            &\quad + m^{1.5} l H^{2} (\max_{i} S_i)^{0.5} (\max_{i} X[\idxTrans{i}])^{0.75} \max_{i} X[\idxR{i}] L^{2.5},    
        \end{aligned}
    \end{equation}
    and 
    \begin{equation}  \label{eqn:regret-f-euler-ineq-2}
        \begin{aligned}
            \Regret(K)
            & \lesssim \sum_{i=1}^{m} \sqrt{\sC_{i}^{\pi} X[\idxTrans{i}] T} 
            + \sum_{i=1}^{m} H \sqrt{X[\idxTrans{i}] L} \sqrt{\Regret(K)} 
            + \sum_{i=1}^{l} \sqrt{ \gG^2 X[\idxR{i}] K} L \\
            &\quad + m^3 H^2 (\max_{i} S_i)^{1.5} \max_{i} X[\idxTrans{i}] L^{2.5} \\
            &\quad + m^{1.5} l H^{2} (\max_{i} S_i)^{0.5} (\max_{i} X[\idxTrans{i}])^{0.75} \max_{i} X[\idxR{i}] L^{2.5}.
        \end{aligned}
    \end{equation}
    For inequality $y \lesssim A \sqrt{y} + B$, the solution is 
    \begin{align*}
        \sqrt{y} \lesssim \frac{A + \sqrt{A^2 + 4B}}{2} \lesssim A + \sqrt{B},
    \end{align*}
    which yields $y \lesssim A^2 + B + 2A^2 \sqrt{B} \lesssim A^2 + B$, where the last inequality is due to the AM-GM inequality.
    Applying the above solution to~\eqref{eqn:regret-f-euler-ineq-1} and~\eqref{eqn:regret-f-euler-ineq-2} yields
    \begin{align}  \label{eqn:regret-f-euler-3}
        \Regret(K)
        & \lesssim \sum_{i=1}^{m} \sqrt{\sC_{i}^{\pi} X[\idxTrans{i}] T} 
        + \sum_{i=1}^{l} \sqrt{ \gR_{i} X[\idxR{i}] T} L
        + m^3 H^2 (\max_{i} S_i)^{1.5} \max_{i} X[\idxTrans{i}] L^{2.5} \\
        &\quad + m^{1.5} l H^{2} (\max_{i} S_i)^{0.5} (\max_{i} X[\idxTrans{i}])^{0.75} \max_{i} X[\idxR{i}] L^{2.5}.
    \end{align}
    and 
    \begin{align}  \label{eqn:regret-f-euler-4}
        \Regret(K)
        & \lesssim \sum_{i=1}^{m} \sqrt{\sC_{i}^{\pi} X[\idxTrans{i}] T} 
        + \sum_{i=1}^{l} \sqrt{ \gG^2 X[\idxR{i}] K} L 
        + m^3 H^2 (\max_{i} S_i)^{1.5} \max_{i} X[\idxTrans{i}] L^{2.5} \\
        &\quad + m^{1.5} l H^{2} (\max_{i} S_i)^{0.5} (\max_{i} X[\idxTrans{i}])^{0.75} \max_{i} X[\idxR{i}] L^{2.5}.
    \end{align}
    With $x_{k, h} = (s_{k, h}, \pi_{k}(s_{k, h}, h))$ where $s_{k, h}\in \gS$ denotes the state in step $h$ of episode $k$, by definition, 
    \begin{align*}
        \sC_{i}^{\pi}
        &= \frac{1}{T} \sum_{k = 1}^{K} \sum_{h = 1}^{H} \sum_{x\in \gX} w_{k, h}(x) g_i^2(\Trans(x), V_{h+1}^{\pi_k}) \\
        &= \frac{4L}{T} \sum_{k = 1}^{K} \sum_{h = 1}^{H} \sum_{x\in \gX} w_{k, h}(x) \Var_{\Trans_{i}(x)} \E_{\Trans_{-i}(x)} [V_{h+1}^{\pi_k}] \\
        &\le \frac{4L}{T} \sum_{k = 1}^{K} \sum_{h = 1}^{H} \sum_{x\in \gX} w_{k, h}(x) \Var_{\Trans(x)} [V_{h+1}^{\pi_k}] \\
        &= \frac{4L}{T} \sum_{k = 1}^{K} \sum_{h = 1}^{H} \E_{\pi_k} \left[\Var_{\Trans(x_{k, h})} [V_{h+1}^{\pi_k} \vert s_{k, h}] \middle\vert s_{k, 1} \right] \\
        &\le \frac{4L}{T} \sum_{k = 1}^{K} \E_{\pi_k} \left[ \left(\sum_{h=1}^{H} R(x_{k, h}) - V_{1}^{\pi_k}(s_{k, 1})\right)^2 \middle\vert s_{k, 1}\right] \\
        &\le \frac{4K \gG^2 L}{T} = \frac{4\gG^2 L}{H},
    \end{align*}
    where the first inequality is due to the component variance bound (Lemma~\ref{lem:comp-var}), the second inequality is due to a law of total variance argument (see Lemma 15 in~\citep{zanette2019tighter} for a proof).
    Substituting the above bound on $\sC_{i}^{\pi}$ into~\eqref{eqn:regret-f-euler-3} and~\eqref{eqn:regret-f-euler-4}, we obtain the other two combinations of the bounds in Theorem~\ref{thm:f-euler}. Specifically, assuming $T \ge \mathrm{poly}(m, l, \max_{i}S_{i}, \max_{i}X[\idxTrans{i}], H)$, we have 
    \begin{align*}
        \Regret(K) &= \bigotilde(\sum_{i=1}^{m} \sqrt{\gG X[\idxTrans{i}]K} + \sum_{i=1}^{l} \sqrt{\gR_{i} X[\idxR{i}]T}), \\
        \Regret(K) &= \bigotilde(\sum_{i=1}^{m} \sqrt{\gG X[\idxTrans{i}]K} + \sum_{i=1}^{l} \sqrt{\gG X[\idxR{i}]K}).
    \end{align*}

    To accommodate the case of known rewards, it suffices to remove the parts related to reward estimation and reward bonuses in both the algorithm and the analysis, which yields the regret bound 
    \begin{align*}
        \min\left\{ \bigotilde\left( \sum_{i=1}^{m} \sqrt{\gQ_i X[\idxTrans{i}] T} \right), \bigotilde\left( \sum_{i=1}^{m} \sqrt{\gG X[\idxTrans{i}] K} \right) \right\}. 
    \end{align*}
\end{proof}

\subsubsection{Proof of Corollary~\ref{cor:f-euler}}

\begin{proof}
    Since $r(x) \in [0, 1]$ and $R(x) \in [0, 1]$ for all $x\in \gX$, we have $\gG \le H$ and $\gR_{i} = \max_{x\in \gX}[\Var(r_i(x))] \le 1$ for all $i\in [l]$.
    Then by Theorem~\ref{thm:f-euler}, for F-EULER, 
    \begin{align*}
        \Regret(K)
        & \lesssim \sum_{i=1}^{m} \sqrt{\gG^2 X[\idxTrans{i}] K L} 
        + \sum_{i=1}^{l} \sqrt{ \gR_{i} X[\idxR{i}] T} L
        + m^3 H^2 (\max_{i} S_i)^{1.5} \max_{i} X[\idxTrans{i}] L^{2.5} \\
        &\quad + m^{1.5} l H^{2} (\max_{i} S_i)^{0.5} (\max_{i} X[\idxTrans{i}])^{0.75} \max_{i} X[\idxR{i}] L^{2.5} \\
        & \lesssim \sum_{i=1}^{m} \sqrt{H X[\idxTrans{i}] T L} 
        + \sum_{i=1}^{l} \sqrt{X[\idxR{i}] T} L
        + m^3 H^2 (\max_{i} S_i)^{1.5} \max_{i} X[\idxTrans{i}] L^{2.5} \\
        &\quad + m^{1.5} l H^{2} (\max_{i} S_i)^{0.5} (\max_{i} X[\idxTrans{i}])^{0.75} \max_{i} X[\idxR{i}] L^{2.5}.
    \end{align*}
    Assuming $T \ge \mathrm{poly}(m, l, \max_{i}S_{i}, \max_{i}X[\idxTrans{i}], H)$, we further have 
    \begin{align*}
        \Regret(K) = \bigotilde\left( 
            \sum_{i=1}^{m} \sqrt{H X[\idxTrans{i}] T} 
            + \sum_{i=1}^{l} \sqrt{X[\idxR{i}] T} 
        \right). 
    \end{align*}
    In the case of known rewards, again, by removing the parts related to reward estimation and reward bonuses in both the algorithm and the analysis, we obtain the regret bound $\bigotilde(\sum_{i=1}^{m} \sqrt{H X[\idxTrans{i}] T})$.
\end{proof}

\newpage
\section{Lower bounds for FMDPs}  \label{sec:appx-lb}

\subsection{Proof for degenerate case 1 (Theorem~\ref{thm:fmdp-lb-degen-1})}

\begin{proof}
    Without loss of generality, assume $\argmax_{i} X[\idxR{i}'] = 1$.
    Let the reward function depends only on $\gX[\idxR{1}']$.
    Then for arbitrary transition, learning in this FMDP can be converted to an MAB problem with $X[\idxR{1}']$ arms, whose regret in $T$ steps has the lower bound $\Omega(\sqrt{X[\idxR{1}'] T})$~\citep{lattimore2018bandit}.
\end{proof}

\subsection{Proof for degenerate case 2 (Theorem~\ref{thm:fmdp-lb-degen-2})}

\begin{proof}
    The construction is essentially the same as that in the proof of Theorem~\ref{thm:fmdp-lb-degen-1}.
    Without loss of generality, assume $\argmax_{i} X[\idxR{i}] = 1$.
    Let the reward function depends only on $\gX[\idxR{1}]$.
    Then for arbitrary transition, learning in this FMDP can be converted to an MAB problem with $X[\idxR{1}]$ arms, whose regret in $T$ steps has the lower bound $\Omega(\sqrt{X[\idxR{1}] T})$~\citep{lattimore2018bandit}.
\end{proof}

\subsection{Proof for the nondegenerate case (Theorem~\ref{thm:fmdp-lb-nondegen})}

\begin{proof}
    The proof of this lower bound relies on the $\Omega(\sqrt{HSAT})$ regret for nonfactored MDPs. The basic idea to prove the $\Omega(\sqrt{HSAT})$ regret bound in~~\citep{jin2020lecture} is to construct an MDP with $3$ states (an initial state, two states with reward $0$ and $1$, respectively) and $A$ actions, where the state remains unchanged in the following $H-1$ steps after one step of transition. 
    This MDP is equivalent to an MAB with $A$ arms, which has the lower bound $\Omega((H-1) \sqrt{AK})$. Making $S/3$ copies of this MAB-like MDP and restarting at each copy uniformly at random yield the expected regret lower bound $\Omega(\sqrt{HSAT})$.

    Here, without loss of generality, assume $\argmax_{i} X[\idxTrans{i}] = 1$.
    Then we only consider the transition of $\gS_1$ and neglect the rest.
    Let the reward function depends only on $\gS_1$.
    Let the transition of $\gS_1$ depend only on $\gS_1$ and $\gX[\idxTrans{1}']$, where $\idxTrans{1}' = \idxTrans{1} \cap \{m+1, \cdots, n\}$ contains only the action component indices. 
    Learning in this component can then be converted to learning in a nonfactored MDP with $\gS_1$ states and $X[\idxTrans{1}']$ actions, which has $\Omega(\sqrt{HS_1 X[\idxTrans{1}'] T})$ regret lower bound.

    Now consider the other state components in $\gX[\idxTrans{1}]$, which we denote by $\gS'_1 = \gX[\idxTrans{1} \cap ([m] \backslash \{i\})] $, where ``$\backslash$'' denotes set subtraction.
    Then $X[\idxTrans{1}] = X'[\idxTrans{1}] \cdot S'_1 \cdot S_1$. 
    Make $S'_1$ copies of the above FMDP, and restart at each copy uniformly at random, so that each copy is expected to run $T/S'_1$ steps ($K/S'_1$ episodes).
    The regret lower bound is then given by $\Omega(S'_1 \sqrt{H S_1 X'[\idxTrans{1}] T / S'_1}) = \Omega(\sqrt{HX[\idxTrans{1}]T})$.
\end{proof}

\subsection{Proof for the more general nondegenerate case (Theorem~\ref{thm:fmdp-lb-loop})}

\begin{proof}
    Refer to the number of intermediate state components in an influence loop as its length. For a state component with the loop property, define its minimum influence loop as the one with the minimum length.
    Without loss of generality, assume $\argmax_{i\in \mathcal{I}} X[\idxTrans{i}] = 1$ and the minimum influence loop of the first state component to be 
    \begin{align}  \label{eqn:loop}
        \gS_1\to \gS_2\to \dots \to \gS_{u-1} \to \gS_u \to \gS_1.
    \end{align}
    Since the above influence loop is minimum, we have $i \not\in \idxTrans{1}$ for all $i\in \{2, \cdots, u-1 \}$.
    Let $\idxTrans{1, s} := \idxTrans{1} \cap ([m] \backslash \{u\})$ and $\idxTrans{1, a} := \idxTrans{1} \cap \{m+1, \cdots, n\}$ be the state parts (excluding $u$) and action parts of the scope index sets of $\gS_1$, respectively.
    Then $X[\idxTrans{1}] = S_{u} X[\idxTrans{1, s}] X[\idxTrans{1, a}]$.
    Note that in the proof below, we can actually relax the assumption that $S_i \ge 3$ for all $i$ to the assumption that $S_u\ge 3$ and $S_i \ge 2$ for all $i\in [u-1]$.
    In the case where $\gS_1$ has the self-loop property, the proof below carries over by letting $u = 1$.
    
    For the space $\gS_i$ in the loop~\eqref{eqn:loop}, we define two special values $s_{+}^{i}, s_{-}^{i} \in \gS_i$ as positive and negative state component values, respectively. 
    Let the reward function be the indicator function of whether at least one of the state components in the loop takes its positive value, i.e., for $x = (s, a)$, 
    \begin{align*}
        R(x) := \max_{i \in [u]}\{ \sI(s[i] = s_{+}^{i}) \}.
    \end{align*}
    Construct the transition of $s[1]$ with dependence only on $\gS_u$ and $\gX[\idxTrans{1, a}]$.
    If $s[u] = s_{+}^{u}$ (or $s_{-}^{u}$, respectively), then $s[1]$ follows $s[u]$ in the next step and transits to $s_{+}^{1}$ (or $s_{-}^{1}$, respectively) deterministically.
    Otherwise, $s[1]$ also transits to $s_{+}^{1}$ or $s_{-}^{1}$, but with probabilities specified by the action components in $\gX[\idxTrans{1, a}]$.
    Construct the transition of the intermediate state components $s[i], i\in \{2, \cdots, u\}$ with dependence only on $s[i-1]$.
    If $s[i-1] = s_{+}^{i-1}$, then in the next step $s[i]$ transits to $s_{+}^{i}$; otherwise $s[i]$ transits to $s_{-}^{i}$.
    The transitions of other state components are arbitrary and irrelevant to the regret.
    Therefore, after the first step, the positive and negative state component values shift their places in a cyclic way within the influence loop.
    
    For initialization, we choose $s[u]$ to be arbitrary in $\gS_{u} \backslash \{ s^{u}_{+}, s^{u}_{-} \}$, $s[i]$ to be arbitrary in $\gS_i \backslash \{ s^{i}_{+} \}$ for $i \in [u-1]$, and $s[\idxTrans{1, s}]$ to be arbitrary in $\gX[\idxTrans{1, s}]$, where we apply the scope operation (Definition~\ref{def:scope}) to $s$.
    The initializations of other state components are arbitrary and irrelevant to the regret.
    In this way, after the first step of transition, there can be zero or one positive state component values in the influence loop, depending on the action components in $\gX[\idxTrans{1, a}]$.
    Moreover, the number of positive state components in the loop remains the same for $H-1$ steps until the end of the episode.
    
    Therefore, for any $s[u] \in \gS_{u} \backslash \{ s_{+}^{u}, s_{-}^{u} \}$ and any $s[\idxTrans{i}] \in \gX[\idxTrans{1, s}]$, learning in the above FMDP can be converted to an MAB problem with $X[\idxTrans{1, a}]$ actions where the reward is $(H - 1)$ if $s_{+}^{1}$ is reached at the second time step and $0$ otherwise. 
    The regret of such an MAB has the lower bound $\Omega((H - 1)\sqrt{X[\idxTrans{1, a}] T}) = \Omega(\sqrt{HX[\idxTrans{1, a}] T})$~\citep{lattimore2018bandit}.
    Splitting $\gS_{u}\backslash \{ s_{+}^{u}, s_{-}^{u} \}$ and $\gX[\idxTrans{1, s}]$ to make $(S_u - 2) X[\idxTrans{1, s}]$ copies of this MAB-like FMDP and restarting at each copy uniformly at random, we obtain the lower bound on regret 
    \begin{align*}
        \Omega\left( (S_u - 2) X[\idxTrans{1, s}] \sqrt{ H X[\idxTrans{1, a}] \left\lfloor \frac{T}{(S_u - 2) X[\idxTrans{1, s}]}\right\rfloor  }\right) 
        = \Omega(\sqrt{HX[\idxTrans{1}]T}).
    \end{align*} 
    Note that making $S_u / 3$ copies is unnecessary since different copies can share $s_{+}^{u}$ and $s_{-}^{u}$.
\end{proof}

\newpage
\section{Lower bound for MDPs via the JAO MDP construction}
\label{sec:jao}

As pointed out in~\citep{azar2017minimax}, ``they (the $\bigotilde(\sqrt{HSAT})$ upper bounds) help to establish the information-theoretic lower bound of reinforcement learning at $\Omega(\sqrt{HSAT})$. \dots ~Moving from this big picture insight to an analytically rigorous bound is non-trivial.''
A rigorous $\Omega(\sqrt{HSAT})$ lower bound proof is possible~\citep{jin2020lecture} by constructing MAB-like MDPs~\citep{dann2015sample}.
In the meantime, the MDP literature suggests that the JAO MDP~\citep{jaksch2010near}, which establishes the minimax lower bound for nonepisodic MDPs, also establishes the minimax lower bound for episodic MDPs.
We look into this problem in detail, and find that a direct episodic extension of the JAO MDP actually establishes the lower bound at $\Omega(\sqrt{HSAT/\log T})$, missing a log factor.
It is not clear whether this result can be further improved with the same construction.
The rest of this section introduces our derivation, which we recommend reading in comparison with~\citep[Section 6]{jaksch2010near}.

Recall that the JAO MDP is an MDP with two states $s_0, s_1$ and $A' = \floor{(A - 1) / 2}$ actions. The transition probability from $s_1$ to $s_0$ is $\delta$ for all actions, and the transition probability from $s_0$ to $s_1$ is $\delta$ for all actions except that it is $\delta + \epsilon$ for one special action $a^{*}$. 
The reward is $1$ for each step at $s_1$ and $0$ otherwise.
By making $S' = \floor{S / 2}$ copies of the JAO MDP one extends the construction to $S$ or $S - 1$ states, which then reduces to a two-state JAO MDP with $S'A'$ actions.
In the episodic setting, we simply start the MDP at $s_0$ for each episode.
The symbols $\delta, \epsilon, A'$ have the same meanings as they do in~\citep{jaksch2010near}, and we use $S'$ to replace $k$ in~\citep{jaksch2010near}.
Let $\E_{a}, \E_{\unif}$ denote the expectation under $a \in [S'A']$ being the better action $a^{*}$ and there being no better action respectively, and $\sP_{a}, \sP_{\unif}$ denote the corresponding probability measures.
Let $\E_{*}$ denote the expectation under a uniformly random choice of the better action.
Hence, $\E_{*}[f] = \frac{1}{S'A'} \sum_{a=1}^{S'A'}\E_{a}[f]$.
Let $N_1, N_0, N_0^{*}$ denote the number of visits to state $s_1$, to state $s_0$ and to state-action pair $(s_0, a^{*})$ respectively.
We use another subscript $k$ to denote the corresponding quantity in the $k$th episode.

\paragraph{Step 1}
Let $H' = 1/\delta$.
\begin{align*}
    \E_{a}[N_{1, k}] 
    \le \E_{a}[N_{0, k}] + \epsilon H' \E_{a}[N_{0, k}^{*}] - \frac{1}{2\delta} + \frac{(1 - 2\delta)^{H}}{2\delta}.
\end{align*}

\begin{proof}
    Here we adopt a more refined analysis compared to that  in \citep{jaksch2010near}, where we consider the probability of the last state, because a constant deviation in each episode results in a relaxation on the order of $\bigo(K)$ in total.
    \begin{align*}
        \E_{a}[N_{1, k}] 
        &= \sum_{h=2}^{H} \sP_{a}(s_{k, h} = s_1 \vert s_{k, h - 1} = s_0) \cdot \sP_{a}(s_{k, h-1} = s_0) \\
        &\quad  + \sum_{h=2}^{H} \sP_{a}(s_{k, h} = s_1 \vert s_{k, h - 1} = s_1) \cdot \sP_{a}(s_{k, h-1} = s_1) \\
        &= \delta \sum_{h=2}^{H} \sP_{a}(s_{k, h-1} = s_0, a_{k, h - 1} \neq a) 
        + (\delta + \epsilon) \sum_{h=2}^{H} \sP_{a}(s_{k, h-1} = s_0, a_{k, h - 1} = a) \\
        &\quad + (1 - \delta) \sum_{h=2}^{H} \sP_{a}(s_{k, h - 1} = s_1) \\
        &= \delta \E_{a}[N_{0, k} - N_{0, k}^{*}] + (\delta + \epsilon) \E_{a}[N_{0, k}^{*}] + (1 - \delta) \E_{a}[N_{1, k}] - \delta \sP_{a}(s_{k, H} = s_0) - (1 - \delta) \sP_{a}(s_{k, H} = s_1).
    \end{align*}

    Note that 
    \begin{align*}
        \sP_{a}(s_{k, H} = s_1) &\ge \sP_{\unif}(s_{k, H} = s_1) = \frac{1}{2} - \frac{1}{2} (1 - 2\delta)^{H - 1}.
    \end{align*}
    Then for $\delta < \frac{1}{2}$, 
    \begin{align*}
        \delta \sP_{a}(s_{k, H} = s_0) + (1 - \delta) \sP_{a}(s_{k, H} = s_1)
        &= \delta (1 - \sP_{a}(s_{k, H} = s_1)) + (1 - \delta) \sP_{a}(s_{k, H} = s_1) \\
        &\ge \delta (\frac{1}{2} + \frac{1}{2} (1 - 2\delta)^{H - 1}) + (1 - \delta) \frac{1}{2} - \frac{1}{2} (1 - 2\delta)^{H - 1} \\
        &= \frac{1}{2} - \frac{(1 - 2\delta)^{H}}{2}.
    \end{align*}
    Hence, 
    \begin{align*}
        \E_{a}[N_{1, k}] 
        &\le \delta \E_{a}[N_{0, k} - N_{0, k}^{*}] + (\delta + \epsilon) \E_{a}[N_{0, k}^{*}] + (1 - \delta) \E_{a}[N_{1, k}] - \frac{1}{2} + \frac{(1 - 2\delta)^{H}}{2}.
    \end{align*}
    By rearranging the terms,
    \begin{align*}
        \E_{a}[N_{1, k}] 
        \le \E_{a}[N_{0, k}] + \epsilon H' \E_{a}[N_{0, k}^{*}] 
        - \frac{1}{2\delta} + \frac{(1 - 2\delta)^{H}}{2\delta},
    \end{align*}
    where we use $H' = \frac{1}{\delta}$.
\end{proof}

\paragraph{Step 2}
Let $R$ denote the cumulative reward by a given algorithm through $K$ episodes. 
Then assuming $\E_{a}[N_0]\le \E_{\unif}[N_0]$, by Step 1 we have 
\begin{align}
    \E_{a}[R] = \E_{a}[N_1] 
    = \sum_{k=1}^{K} \E_{a}[N_{1, k}] 
    \le KH - \sum_{k=1}^{K} \E_{\unif}[N_{1, k}] + \epsilon H' \sum_{k=1}^{K} \E_{a}[N_{0, k}^{*}] 
    - \frac{K}{2\delta} + K\frac{(1 - 2\delta)^{H}}{2\delta}.  \label{eqn:appx-jao-step-2}
\end{align}

\paragraph{Step 3}
Independent of the above two steps, 
\begin{align}  \label{eqn:appx-jao-step-3-1}
    \E_{\unif}[N_{1, k}] 
    = \sum_{t=1}^{H} \sP_{\unif}(s_t = s_1) 
    = \sum_{t=1}^{H} \frac{1}{2} + (1 - 2\delta)^{t - 1} (0 - \frac{1}{2})
    = \frac{H}{2} - \frac{1 - (1 - 2\delta)^{H}}{4\delta} 
    \ge \frac{H - H'}{2}.  
\end{align}
Therefore,
\begin{align}  \label{eqn:appx-jao-step-3-2}
    \E_{\unif}[N_{0, k}] \leq  \frac{H+H'}{2}.
\end{align}

\paragraph{Step 4}
Substituting~\eqref{eqn:appx-jao-step-3-1} in Step 3 into~\eqref{eqn:appx-jao-step-2} in Step 2 yields 
\begin{align*}
    \E_{a}[R] \le \frac{KH}{2} + \epsilon H' \sum_{k=1}^{K} \E_{a}[N_{0, k}^{*}] 
    - \frac{K}{4\delta} + K\frac{(1 - 2\delta)^{H}}{4\delta}.
\end{align*}
Therefore,
\begin{align*}
    \E_{*}[R] = \frac{1}{S'A'} \sum\limits_{a=1}^{S'A'} \E_{a}[R] \leq \frac{KH}{2} + \frac{\epsilon H' }{S'A'} \sum\limits_{a=1}^{S'A'} \sum_{k=1}^{K} \E_{a}[N_{0, k}^{*}] - \frac{K}{4\delta} + K\frac{(1 - 2\delta)^{H}}{4\delta}.
\end{align*}

\paragraph{Step 5}
\begin{align*}
    \E_{*}[R] \le \frac{KH}{2} + \frac{\epsilon KH'}{S'A'} \left( \frac{H + H'}{2} + \frac{\epsilon H \sqrt{H'}}{2} \sqrt{S'A'(H' + H)K} \right) 
    - \frac{K}{4\delta} + K\frac{(1 - 2\delta)^{H}}{4\delta}.
\end{align*}

\begin{proof}
    To be more explicit, let $N_{0, k, a}$ denote the number where action $a$ is chosen in state $s_0$ at the $k$th episode. Then by Lemma 13 in~\citep{jaksch2010near}, we have
    \begin{align*}
        \sum_{k=1}^{K} \sum_{a=1}^{S'A'} \E_{a}[N_{0, k}^{*}] 
        = \sum_{a=1}^{S'A'} \E_{a}\left[\sum_{k=1}^{K}  N_{0, k, a}\right]
        \le \sum_{a=1}^{S'A'} \E_{\unif}\left[\sum\limits_{k=1}^{K} N_{0,k,a}\right] +  \sum_{a=1}^{S'A'} \epsilon KH\sqrt{H'}\sqrt{ 2\E_{\unif}\left[\sum\limits_{k=1}^{K}N_{0,k,a}\right]}.
    \end{align*}
    Since $\sum_{a=1}^{S'A'} \E_{\unif}[N_{0, k,a}] \le \frac{H + H'}{2}$ by~\eqref{eqn:appx-jao-step-3-1} in Step 3, 
    \begin{align*}
        \sum_{k=1}^{K} \sum_{a=1}^{S'A'} \E_{a}[N_{0, k}^{*}] 
        \le K \frac{H + H'}{2} + \frac{\epsilon KH \sqrt{H'}}{2} \sqrt{S'A' K (H + H')}.
    \end{align*}
    Substituting the above into the bound on $\E_{*}[R]$ in Step 4 concludes the proof.
\end{proof}

\paragraph{Step 6}
Now we determine the optimal value $V_{1}^{*}(s_0)$ we aim to compare.
Let $V_{H+1}^{*} = [0, 0]^{\top}$.
Then the optimal value function is given by the iteration $V_{h}^{*} = B + A V_{h+1}^{*}$, where 
\begin{align*}
    A = \left[ \begin{array}{cc}
        1 - \delta - \epsilon & \delta + \epsilon \\
        \delta & 1 - \delta
    \end{array}\right], \quad 
    B = \left[ \begin{array}{c}
        0 \\
        1
    \end{array}\right].
\end{align*}
Iteratively, by matrix diagonalization, 
\begin{align*}
    V_{1}^{*} 
    = B + AV_{2}^{*} = B + A(B + AV_{3}^{*}) 
    = \cdots 
    = (\sum_{h=0}^{H-1} A^h) B 
    = U (\sum_{h=0}^{H-1} \Lambda^h) U B,
\end{align*}
where $U = [1, -\frac{\delta + \epsilon}{\delta}; 1, 1]$ and $\Lambda = \diag(1, 1 - 2\delta - \epsilon)$.
Hence,
\begin{align*}
    V_1^{*}(s_0) = \frac{\delta + \epsilon}{2\delta + \epsilon} H - \frac{\delta + \epsilon}{(2\delta + \epsilon)^2} \left( 1 - (1 - 2\delta - \epsilon)^{H} \right).
\end{align*}

Note that here we compute the exact optimal value because the episodic resetting causes a constant difference than the stationary optimal value in the infinite-horizon setting in each episode, which accumulates to the order of $\bigo(K)$ in total, similar to Step 1.

\paragraph{Step 7}
Since 
\begin{align*}
\frac{\delta + \epsilon}{ (2\delta + \epsilon)^2}  = \frac{\delta + \epsilon}{ 4\delta (\delta + \epsilon)  + \epsilon^2} \leq \frac{1}{4\delta}, 
\end{align*}
we have 
\begin{align*}
    \Regret(K)
    &= K V_{1}^{*}(s_0) - \E_{*}[R] \\
    &\ge \underbrace{\frac{\delta + \epsilon}{2\delta + \epsilon} KH - \frac{KH}{2} - \frac{\epsilon KH'}{S'A'} \left( \frac{H + H'}{2} + \frac{\epsilon H \sqrt{H'}}{2} \sqrt{S'A'K(H' + H)} \right)}_{\text{standard as in~\citep{jaksch2010near}}} \\
    &\quad \underbrace{- K \left( \frac{(\delta + \epsilon)\left( 1 - (1 - 2\delta - \epsilon)^{H} \right)}{(2\delta + \epsilon)^2} 
    -\frac{1}{4\delta} 
    + \frac{(1 - 2\delta)^{H}}{4\delta}
    \right)}_{\text{new challenge in the episodic setting}} \\
    &\ge \underbrace{\frac{\epsilon}{4\delta + 2\epsilon} T 
    - \frac{\epsilon TH'}{2S'A'} \left(1 + \frac{H'}{H}\right) 
    - \frac{\epsilon^2 TH'}{2S'A'} \sqrt{H'S'A'KH} \left(\sqrt{1 + \frac{H'}{H}}\right)}_{\Theta(\sqrt{H'S'A'T})} \\ 
    &\quad - \underbrace{\frac{K}{4\delta} \left((1 - 2\delta)^{H}- (1 - 2\delta - \epsilon)^{H} \right)}_{\Theta(KH'(1 - \frac{2}{H'})^{H})},
\end{align*}
where $\Theta(\sqrt{H'S'A'T})$ in the last line is obtained by taking $\epsilon = \Theta(\sqrt{\frac{S'A'}{H'T}}) = \Theta(\sqrt{\frac{\delta S'A'}{T}})$.
The logic of taking such an $\epsilon$ is to let the first term be on the order of the third term, i.e., 
\begin{align*}
    \frac{\epsilon}{4\delta + 2\epsilon} T = \Theta\left( \frac{\epsilon^2 TH'}{2S'A'} \sqrt{H'S'A'KH} \right).
\end{align*}
The new challenge in the episodic setting results from the episodic resetting of each episode and a different definition of regret.
By taking $H = H' \log KH = H'\log T$, we have 
\begin{align*}
    KH' (1 - \frac{2}{H'})^{H} \le KH' \frac{1}{e^{\log KH}} = \frac{H'}{H} = \frac{1}{\log T} = o(\sqrt{H'S'A'T}).
\end{align*}
Hence, we obtain the $\Omega(\sqrt{HSAT/\log T}) = \tilde{\Omega}(\sqrt{HSAT})$ lower bound.
Note that taking $H' = H / \log T$ is reasonable, because none of the parameters is exponential in others in our consideration.

\newpage
\section{Concentration inequalities}

In this section, we provide a summary of some important concentration inequalities that are frequently invoked in this work. In what follows, $\sP(\cdot)$ denotes an appropriate probability measure.

\begin{lemma}[Concentration on $\normlone$-norm of probability distributions] \label{lem:concen-l1}
    Let $P$ be a probability mass function on a finite set $\gY$ with cardinality $Y$. 
    Let $\rvy = [\ry_1, \cdots, \ry_n]$ be $n$ i.i.d. samples from $P$. 
    Let $\hat{P}_{\rvy}$ be the empirical distribution based on the observed samples.
    Then for all $\epsilon > 0$
    \begin{align*}
        \sP(\| P - \hat{P}_{\rvy}\|_1 \ge \epsilon) \le 2^{Y} \exp\left\{ -\frac{n\epsilon^2 }{2} \right\}.
    \end{align*}
    Alternatively, with probability at least $1 - \delta$,
    \begin{align*}
        \| P - \hat{P}_{\rvy}\|_1 \le \sqrt{\frac{2Y \log 2 + 2 \log \frac{1}{\delta}}{n}} \le \sqrt{\frac{2Y}{n} \log \frac{2}{\delta}}
    \end{align*}
\end{lemma}

\begin{proof}
    This lemma is a relaxation of Theorem 2.1 in~\citep{weissman2003inequalities}.
\end{proof}

\begin{lemma}[Hoeffding's inequality] \label{lem:concen-h}
    Given $n$ independent random variables such that $\rx_i \in [a, b]$ a.s., then for all $\epsilon \ge 0$, 
    \begin{align*}
        \sP\left( \left| \sum_{i=1}^{n} (\rx_i - \E[\rx_i]) \right| \ge \epsilon \right) \le 2\exp\left\{ -\frac{2\epsilon^2}{n (b - a)^2} \right\}.
    \end{align*}
    Alternatively, with probability at least $1 - \delta$,
    \begin{align*}
        \left| \frac{1}{n} \sum_{i=1}^{n} (\rx_i - \E[\rx_i]) \right| \le \sqrt{\frac{(b - a)^2}{2n} \log \frac{2}{\delta}}.
    \end{align*}
\end{lemma}

\begin{proof}
    See e.g., Proposition 2.5 in~\citep{wainwright2019high} for the one-sided Hoeffding's inequality, applying which to $-\rx_i$ yields the other side.
\end{proof}

\begin{lemma}[One-sided Bernstein's inequality] \label{lem:concen-b}
    Given $n$ independent random variables such that $\rx_i \le b$ a.s., then for all $\epsilon \ge 0$, 
    \begin{align*}
        \sP\left(\sum_{i=1}^{n} (\rx_i - \E[\rx_i])\ge n\epsilon\right) 
        \le \exp\left\{ -\frac{n\epsilon^2}{2(V + \frac{b\epsilon}{3})} \right\},
    \end{align*}
    where $V = \frac{1}{n} \sum_{i=1}^{n} \E[\rx_i^2]$. Alternatively, with probability at least $1 - \delta$,
    \begin{align}  \label{eqn:one-sided-b}
        \frac{1}{n} \sum_{i=1}^{n} (\rx_i - \E[\rx_i]) 
        \le \frac{b\log \frac{1}{\delta}}{3n} + \sqrt{(\frac{b\log \frac{1}{\delta}}{3n})^2 + \frac{2V\log \frac{1}{\delta}}{n}}
        \le \frac{2b\log \frac{1}{\delta}}{3n} + \sqrt{\frac{2V\log \frac{1}{\delta}}{n}}.
    \end{align}
\end{lemma}

\begin{proof}
    See e.g., Proposition 2.14 in~\citep{wainwright2019high}.
\end{proof}

For random variables $\rx_i \in [0, b]$, applying the one-sided Bernstein's inequality to $\rx_i - \E[\rx_i]$, we can replace $V$ in~\eqref{eqn:one-sided-b} by the average variance $\sigma^2 = \frac{1}{n} \sum_{i=1}^{n} \Var(\rx_i)$.
Moreover, applying the one-sided Bernstein's inequality to $- \rx_i + \E[\rx_i]$ yields the other side of the bound in terms of $\sigma^2$.
By the union bound, we have that with probability at least $1 - \delta$, 
\begin{align*}
    \left| \frac{1}{n} \sum_{i=1}^{n} (\rx_i - \E[\rx_i]) \right|
    \le \frac{2b\log \frac{1}{\delta}}{3n} + \sqrt{\frac{2\sigma^2 \log \frac{1}{\delta}}{n}},
\end{align*}
which is what we actually use in this work.

\begin{lemma}[Empirical Bernstein's inequality] \label{lem:concen-b-emp}
    Given $n$ independent random variables such that $\rx_i \le 1$ a.s.. Let $S = \frac{1}{n(n-1)} \sum_{i=1}^{n} \sum_{j=1}^{n} \frac{(\rx_i-\rx_j)^2}{2}$ be the sample variance. Then with probability at least $1 - \delta$,
    \begin{align*}
        \frac{1}{n} \sum_{i=1}^{n} (\E[\rx_i]-\rx_i) 
        \le \sqrt{\frac{2S \log \frac{2}{\delta}}{n}} + \frac{7 \log\frac{2}{\delta}}{3(n-1)}.
    \end{align*}
\end{lemma}

\begin{proof}
    See Theorem 11 in~\citep{maurer2009empirical}.
\end{proof}

\end{document}